\documentclass{article}

\usepackage{microtype}
\usepackage{graphicx}
\usepackage{caption}
\usepackage{subfigure}
\usepackage{booktabs} 
\usepackage{makecell}
\usepackage{hyperref}

\usepackage[accepted]{icml2024}

\usepackage{amsmath,amsfonts,bm}

\newcommand{\Od}{\mathrm{O}(d)}
\newcommand{\SOd}{\mathrm{SO}(d)}
\newcommand{\stab}[2]{\text{Stab}_{#1}\left(#2\right)}
\newcommand{\orb}[2]{\text{Orb}_{#1}\left(#2\right)}
\newcommand{\Sn}{\mathrm{S}_n}

\def\eqref#1{equation~\ref{#1}}

\def\1{\bm{1}}

\def\va{{\bm{a}}}
\def\vb{{\bm{b}}}

\def\ve{{\bm{e}}}

\def\vk{{\bm{k}}}

\def\vu{{\bm{u}}}
\def\vv{{\bm{v}}}
\def\vw{{\bm{w}}}
\def\vx{{\bm{x}}}
\def\vy{{\bm{y}}}
\def\vz{{\bm{z}}}

\def\mA{{\bm{A}}}
\def\mB{{\bm{B}}}

\def\mD{{\bm{D}}}

\def\mH{{\bm{H}}}
\def\mI{{\bm{I}}}
\def\mJ{{\bm{J}}}

\def\mM{{\bm{M}}}

\def\mO{{\bm{O}}}
\def\mP{{\bm{P}}}
\def\mQ{{\bm{Q}}}
\def\mR{{\bm{R}}}
\def\mS{{\bm{S}}}

\def\mU{{\bm{U}}}
\def\mV{{\bm{V}}}

\def\mX{{\bm{X}}}

\def\mLambda{{\bm{\Lambda}}}

\DeclareMathAlphabet{\mathsfit}{\encodingdefault}{\sfdefault}{m}{sl}
\SetMathAlphabet{\mathsfit}{bold}{\encodingdefault}{\sfdefault}{bx}{n}

\newcommand{\R}{\mathbb{R}}

\usepackage{amsmath}
\usepackage{amssymb}
\usepackage{mathtools}
\usepackage{amsthm}
\usepackage{tabularx}

\usepackage[capitalize,noabbrev]{cleveref}

\theoremstyle{plain}
\newtheorem{theorem}{Theorem}[section]

\newtheorem{lemma}{Lemma}[section]

\theoremstyle{definition}
\newtheorem{definition}[theorem]{Definition}

\theoremstyle{remark}

\usepackage[textsize=tiny,disable]{todonotes}
\usepackage{regexpatch}

\newcommand{\jacob}[1]{\textcolor{black}{#1}}

\icmltitlerunning{Equivariance via Minimal Frame Averaging for More Symmetries and Efficiency}

\usepackage{thmtools,thm-restate}

\newcommand{\blue}[1]{#1}
\newcommand{\orange}[1]{#1}
\newcommand{\colorblue}{} 
\newcommand{\colorblack}{} 
\newcommand{\colororange}{} 
\newcommand{\magenta}{} 

\newcommand{\cyan}{}

\renewcommand{\cyan}[1]{\textcolor{black}{#1}}

\renewcommand{\blue}[1]{\textcolor{blue}{#1}}
\renewcommand{\orange}[1]{\textcolor{orange}{#1}}
\renewcommand{\magenta}[1]{\textcolor{magenta}{#1}}
\renewcommand{\colorblue}{\color{blue}}
\renewcommand{\colorblack}{\color{black}}
\renewcommand{\colororange}{\color{orange}}

\renewcommand{\orange}[1]{\textcolor{blue}{#1}}
\renewcommand{\magenta}[1]{\textcolor{blue}{#1}}
\renewcommand{\colorblue}{\color{blue}}
\renewcommand{\colororange}{\color{blue}}

\renewcommand{\blue}[1]{\textcolor{black}{#1}}
\renewcommand{\orange}[1]{\textcolor{black}{#1}}
\renewcommand{\magenta}[1]{\textcolor{black}{#1}}
\renewcommand{\colorblue}{\color{black}}
\renewcommand{\colororange}{\color{black}}

\begin{document}

\twocolumn[
\icmltitle{Equivariance via Minimal Frame Averaging for More Symmetries and Efficiency}

\icmlsetsymbol{equal}{*}

\begin{icmlauthorlist}
\icmlauthor{Yuchao Lin}{tamu}
\icmlauthor{Jacob Helwig}{tamu}
\icmlauthor{Shurui Gui}{tamu}
\icmlauthor{Shuiwang Ji}{tamu}
\end{icmlauthorlist}

\icmlaffiliation{tamu}{Department of Computer Science and Engineering, Texas A\&M University, Texas, USA}

\icmlcorrespondingauthor{Shuiwang Ji}{sji@tamu.edu}

\icmlkeywords{Machine Learning, ICML}

\vskip 0.3in
]

\printAffiliationsAndNotice{}  

\begin{abstract}
We consider achieving equivariance in machine learning systems via frame averaging. Current frame averaging methods involve a costly sum over large frames or rely on sampling-based approaches that only yield approximate equivariance. Here, we propose Minimal Frame Averaging (MFA), a mathematical framework for constructing provably minimal frames that are exactly equivariant. The general foundations of MFA also allow us to extend
frame averaging to more groups than previously considered, including the Lorentz group for describing symmetries in space-time, and the unitary group for complex-valued domains.
Results demonstrate the efficiency and effectiveness of encoding symmetries via MFA across a diverse range of tasks, including $n$-body simulation, top tagging in collider physics, and relaxed energy prediction. Our code is available at \url{https://github.com/divelab/MFA}.

\end{abstract}

    

\section{Introduction}


Encoding symmetries in machine learning models has shown impressive results across diverse disciplines, including mathematical problem solving~\cite{puny2023equivariant,lawrence2023learning}, generalization ability~\cite{kondor2018generalization,gui2022good,li2023graph,gui2024joint}, vision~\cite{cohen2016steerable,esteves2020spin,worrall2019deep}, quantum mechanics~\cite{unke2021se,chen2022systematic,yu2023efficient}, chemistry~\cite{xu2021molecule3d,xu2023geometric,hoogeboom2022equivariant,batatia2022mace,stark2022equibind,wang2023learning,fu2023latent}, materials science~\cite{yan2022periodic,lin2023efficient,luo2023towards,yan2024complete}, and physics~\cite{wang2020incorporating,wang2022approximately}.
Nevertheless, the derivation of equivariant architectures is a non-trivial task and may come at the cost of expressiveness or increased computational effort~\cite{puny2021frame,du2022se,kim2023learning,duval2023faenet,pozdnyakov2024smooth}. Frame averaging~\cite{puny2021frame} has emerged as a model-agnostic alternative for instilling equivariance in non-equivariant models. Because the cost of frame averaging operator scales with the size of the frame, sampling-based approaches have been devised to enhance efficiency by sacrificing exact equivariance~\cite{kim2023learning,duval2023faenet}.

We introduce \textit{Minimal Frame Averaging}, a mathematical framework for efficient frame averaging that simultaneously maintains exact equivariance. Our general theory also enables us to derive minimal frames for a set of groups strictly larger than those proposed in previous works, including the Lorentz group, the proper Lorentz group, the unitary group, the special unitary group, the general linear group, and the special linear group. Empirically, we demonstrate the advantages of our method on a variety of tasks spanning diverse groups, including $n$-body simulation, isomorphic graph separation, classification of hadronically decaying top quarks (top tagging), relaxed energy prediction on OC20, and prediction of 5-dimenional convex hull volumes.      

\section{Background}

    $G$-equivariance of a function $f:\mathcal V\to\mathcal W$ \orange{is defined such}\todo{Reviewer 3, Q1} that actions from $G$ applied on the domain $\mathcal V$ are applied equally to the codomain $\mathcal W$. Formally, if $\rho_\mathcal W$ and $\rho_\mathcal V$ are representations of $G$ in the spaces $\mathcal V$ and $\mathcal W$, respectively, then for all $g\in G$ and $x\in \mathcal V$,
    $$
    f(\rho_\mathcal V(g)x)= \rho_\mathcal W(g)f(x).
    $$
    In the special case where $\rho_\mathcal W(g)=1$, known as the \textit{trivial representation}, $f$ is said to be $G$-invariant. Symmetry priors that instill equivariance in neural networks can improve generalization and sample complexity across diverse tasks with exact and approximate symmetries~\cite{bronstein2021geometric,han2022geometrically,zhang2023artificial}. In~\cref{sec:eqML}, we overview prevalent approaches in the sphere of equivariant machine learning before introducing frame averaging in~\cref{sec:FA}.

    \subsection{Equivariant Machine Learning}\label{sec:eqML}
        
        Equivariant methods are largely defined in terms of their choice of internal representation~\cite{duval2023hitchhiker,bronstein2021geometric,han2022geometrically}. Scalarization methods~\cite{schutt2018schnet,schutt2021equivariant,gasteiger2019directional,gasteiger2021gemnet,kohler2020equivariant,satorras2021n,tholke2021equivariant,jing2021learning,liu2022spherical,wang2022comenet,huang2022equivariant, dym2024low} achieve equivariance by transforming geometric quantities such as vectors into invariant scalars or by manipulating these quantities exclusively by scalar multiplication, \blue{and are known to be universal approximators of equivariant functions}~\cite{villar2021scalars,han2022geometrically}.
        Alternatively, group equivariant CNNs leverage the \textit{regular representation}, wherein feature maps and kernels are functions on the group $G$~\cite{cohen2016group,horie2021isometric,wang2022approximately,wang2023relaxed,helwig2023group}. As this approach requires a feature vector for each group element, it can become prohibitively expensive in cases where the group is large~\cite{bronstein2021geometric}. For $G$ with an infinite number of elements such as the orthogonal group $\Od$, network layers may instead map between \textit{irreducible representations}~\cite{thomas2018tensor,weiler20183d,anderson2019cormorant,brandstetter2021geometric,smidt2021finding,batzner20223}. Although the irreducible approach has been widely used for equivariance to groups such as $\Od$, derivation of the irreducible representations for new groups, such as the Lorentz group $\text{O}(1,d-1)$ \cite{gong2022efficient}, is a non-trivial task~\cite{bronstein2021geometric}. Thus, \citet{ruhe2023clifford} venture beyond irreducible representations by employing a steerable multivector basis which simultaneously maintains the geometric structure inherent to the data. 
        
    \subsection{Frame Averaging}\label{sec:FA}

        While the success of symmetry-preserving frameworks has been demonstrated in a variety of domains and tasks, the construction of equivariant architectures that are both expressive and computationally efficient is challenging~\cite{puny2021frame,du2022se,kim2023learning,duval2023faenet,kiani2024hardness}. Therefore, \citet{puny2021frame} proposed \textit{frame averaging}, an architecture-agnostic approach to equivariant learning derived from the \textit{group averaging} operator~\cite{yarotsky2022universal}. Group averaging maps an arbitrary function $\Phi: \mathcal{V} \to \mathcal{W}$ to a $G$-equivariant function $\hat{\Phi}$ defined in the case of a finite group $G$ by
        \begin{equation}
        \label{eqn:ga}
            \hat{\Phi}(x) = \frac{1}{|G|} \sum_{g\in G} \rho_{\mathcal{W}}(g) \Phi(\rho_{\mathcal{V}}(g^{-1}) x),
        \end{equation}
        where $|G|$ denotes the cardinality of $G$ and $\rho_{\mathcal{W}}(g)$ as the trivial representation gives a $G$-invariant $\hat\Phi$. 
        Computation of~\cref{eqn:ga} quickly becomes intractable as the cardinality of $G$ grows or in the case of an infinite group. Thus, the \textit{frame averaging operator} replaces $G$ in the summation with a \textit{frame} $\mathcal F(x)\subseteq G$. A special case of interest is a \textit{$G$-equivariant frame}:    \begin{definition}[$G$-Equivariant Frame~\citep{puny2021frame}]\label{def:fa}
    Given a power set of a finite group $G$ as $\mathcal{P}(G)$, a set-valued function \( \mathcal{F}: \mathcal{V} \to \mathcal{P}(G) \setminus \{\emptyset\} \) is termed a $G$-equivariant frame if and only if:
    \begin{enumerate}
      \item $G$-Equivariance: For all \( x \in \mathcal{V} \) and \( g \in G \), $\mathcal{F}(g\cdot x) = g\mathcal{F}(x)$, where $g\cdot x$ denotes the \blue{action} of $g$ on $x$;
      \item Boundedness: For any \( g \in \mathcal{F}(x) \), there exists a constant \( c > 0 \) such that $\Vert \rho_{\mathcal{W}}(g) \Vert_{op}\le c$, where $\Vert\cdot\Vert_{op}$ denotes the induced operator norm over $\mathcal{W}$.
    \end{enumerate}
    \end{definition}
        By replacing $G$ in~\cref{eqn:ga} with a $G$-equivariant frame $\mathcal F$ as
    \begin{equation}\label{eqn:fa}
    \langle\Phi\rangle_\mathcal F(x) \coloneq \frac{1}{|\mathcal{F}(x)|} \sum_{g\in\mathcal{F}(x)} \rho_\mathcal{W}(g) \Phi(\rho_\mathcal{V}(g^{-1})x),
    \end{equation}
        the frame averaging operator $\langle\cdot\rangle_\mathcal F$ achieves equivariance more efficiently than the group averaging operator~\citep{puny2021frame}. \citet{puny2021frame} furthermore leverage the boundedness property in~\cref{def:fa} to show that the resulting function $\langle\Phi\rangle_\mathcal F$ maintains the expressivity of the backbone architecture $\Phi$.    
    
\section{Methods}


    In spite of the aforementioned advantages of frame averaging, there are several vital considerations that determine its effectiveness in equivariant learning tasks. First, the efficiency of the frame averaging operator in terms of runtime scales with the size of the frame $\mathcal F(x)$. Thus, in~\cref{sec:min_frame_constr}, we formalize the concept of a \textit{minimal frame} and derive a defining property which will enable its computation. Specifically, this property is that the computation of the minimal frame requires a \textit{canonical form}, which we formalize in~\cref{sec:can_form}. To circumvent challenges that commonly arise when computing the canonical form, we discuss canonicalization on induced $G$-sets in~\cref{sec:inducedG}.  
    
    Second, the derivation of a $G$-equivariant frame depends directly on the group $G$. Thus, in~\cref{sec:linAlgG}, we derive a minimal frame for the \textit{linear algebraic group} defined by $\mO^T \eta \mO=\eta$, with $\eta$ a diagonal matrix with diagonal entries $\pm1$. This group subsumes a range of groups which appear frequently in common machine learning tasks, including the orthogonal group $\Od$ and the special orthogonal group $\mathrm{SO}(d)$, as well as more specialized groups such as the Lorentz group $\mathrm{O}(1,d-1)$ which describes symmetries of spacetime in particle physics and the unitary group $\mathrm{U}(d)$ describing complex orthogonal matrices. We go on to derive a minimal frame for the permutation group in~\cref{sec:permG}, which appears in various graph learning tasks. Beyond these groups, we leverage the mathematical framework developed in~\cref{sec:min_frame_constr,sec:can_form,sec:inducedG} to derive minimal frame averaging for a variety of groups in~\cref{sec:frame_group}, including $\mathbb{R}^d, \mathrm{E}(d),\mathrm{SE}(d), \mathrm{GL}(d, \mathbb{R}),$ and $ \mathrm{SL}(d, \mathbb{R})$. 
    

    \subsection{Construction of the Minimal Frame}\label{sec:min_frame_constr}


\todo{R3, Q6}
To more precisely define a frame with minimal size, we introduce the concept of a \textit{minimal frame} defined on a $G$-set\footnote{\blue{We provide a brief introduction to relevant topics in group theory, including $G$-sets, in~\cref{sec:group_intro}.}} $\mathcal S$.

\begin{definition}[Minimal Frame]\label{def:min_frame_def}
    A minimal frame \(\hat{\mathcal{F}}: \mathcal{S} \rightarrow \mathcal{P}(G) \setminus \{\emptyset\} \) is a frame such that for any \(x \in \mathcal{S}\), there does not exist a frame \(\mathcal{F}\) with \(\mathcal{F}(x) \subset \hat{\mathcal{F}}(x)\).
\end{definition}

The following result, proven in~\cref{app:pf_lem:minimal_frame}, relates an arbitrary frame $\mathcal F$ to \textit{the stabilizer} of an element of interest $x_0$ residing in the \textit{orbit of $x$}. The stabilizer of $x_0$ is defined as $\stab G{x_0}\coloneq\{g\in G\mid g\cdot x_0=x_0\}$, while the orbit of $x$ is defined as ${\orb G x\coloneq\{g\cdot x\mid g\in G\}}$. This result will prove vital in constructing minimal frames.




\begin{restatable}{lemma}{lemminimalframe}
\label{lem:minimal_frame}
    Given a frame $\mathcal F:\mathcal{S}\rightarrow \mathcal{P}(G) \setminus \{\emptyset\}$, for all \(x\in \mathcal S\), there exists \(x_0\in \text{Orb}_G(x)\) such that \(\text{Stab}_G(x_0)\subseteq \mathcal F (x_0)\).
\end{restatable}

Using this result, we next show that for each $x\in \mathcal S$, $\hat{\mathcal F}(x)$ may be derived by selecting a unique representative $x_0\in \orb Gx$ and computing its stabilizer, the proof of which can be found in~\cref{app:thm:minimal_frame}. We refer to $x_0 = c(x)$ as the \textit{canonical form} \orange{with the canonicalization function $c$}\todo{R3, Q2}, and its uniqueness is in the sense that all elements in $\orb G x$ share the same canonical form. We discuss the canonical form in greater detail in~\cref{sec:can_form} where we find that identifying a canonical form can be non-trivial depending on the group $G$ and base space $\mathcal S$, although for the purposes of a general derivation of the minimal frame in the following result, we simply assume that we have access to it. 
\begin{restatable}{theorem}{thmminimalframe}
\label{thm:minimal_frame}
    For all $x\in\mathcal S$, let $\orange{x_0 = c(x) = h^{-1}\cdot x}$ for some \(h\in G\). Define the frame \(\mathcal{F}:\mathcal{S}\rightarrow \mathcal{P}(G) \setminus \{\emptyset\} \) such that \(\mathcal{F}(x) = h\text{Stab}_G(x_0)\); then, \(\mathcal{F}\) is a minimal frame.  
\end{restatable}

    \todo{Reviewer 2, Q3, R3, Q7}
    \colororange
    This result shows how a minimal frame $\hat{\mathcal F}(x)$ \blue{is constructed} through the canonical form $x_0$ such that $\hat{\mathcal F}(x)=h\stab G{x_0}$, \blue{and is a generalization of Theorem 3 from~\citet{puny2021frame}, which shows that $\lvert \mathcal F(x)\rvert\geq\lvert\stab Gx\rvert$ given any equivariant frame $\mathcal F$ for invariant frame averaging. However, there are cases where $\lvert\stab Gx\rvert$ is infinite, \textit{e.g.,} \orange{a point cloud lying on a $(d-2)$-dimensional Euclidean subspace \blue{embedded} in a $d$-dimensional Euclidean space under the action of $\Od$}. While it then appears natural to define the minimal frame in terms of a measure on the group $\mu$\footnote{A brief introduction to measures on groups is in~\cref{sec:general_domain}.}, the existence of sets with measure 0 imply that~\cref{def:min_frame_def} is stronger. Specifically, we might instead define the minimal frame $\hat{\mathcal F}$ such that for all $x$ and equivariant frames $\mathcal F$, $\mu(\hat{\mathcal F}(x))\leq \mu(\mathcal F(x))$. It would then follow that for the set $A(x)\subset G$ such that $A(x)\cap \hat{\mathcal F}(x)\neq\emptyset$ and $\mu(A(x))=0$, the frame $\tilde{\mathcal F}(x)\coloneq\hat{\mathcal F}(x)\cup A(x)$ is also minimal, \orange{as $\mu(\tilde{\mathcal F}(x)) = \mu(\hat{\mathcal F}(x))$}, whereas $\tilde{\mathcal F}$ is not minimal by~\cref{def:min_frame_def} since $\hat{\mathcal F}(x)$ is a proper subset of $\tilde{\mathcal F}(x)$. However, as proven in~\cref{cor:smallest_measure}, minimality in terms of any measure is implied by~\cref{def:min_frame_def}, and thus, \cref{def:min_frame_def} is stronger.} 
    
    \colorblack

    \subsection{The Canonical Form}\label{sec:can_form}

    Equipped with a notion of the minimal frame and how it is constructed using a canonical form $c(x)\coloneq x_0$, we now formalize $c(x)$ by generalizing the definition of~\citet{mckay2014practical} beyond graphs, and discuss canonicalization in 2 concrete cases.
    \todo{R3, Q5}
    \begin{definition}[Canonical Form]\label{def:can_form}
    Given an equivalence relation \(\sim\) on \(\mathcal{S}\) such that for any \(x, y \in \mathcal{S}\), \(x \sim y\) if and only if $y\in\orb Gx$, a canonicalization with respect to \(\sim\) is a mapping \(c: \mathcal{S} \rightarrow \mathcal{S}\) satisfying the following conditions for all \(x, y \in \mathcal{S}\):
    \begin{enumerate}
        \item Representativeness: \(x \sim c(x)\);
        \item \(G\)-invariance/ Uniqueness: \orange{$c(x) = c(y)$ if $x\sim y$}.
    \end{enumerate}
    \end{definition}
    Property 1 in~\cref{def:can_form} requires that the canonical form is a single element residing in the orbit of $x$, while Property 2 requires that the canonical form be unique within each orbit. \todo{R3, Q5}\orange{An example of canonicalization is of a symmetric matrix \(\mP \in \mathbb{R}^{d\times d}\), where the orthogonal group \(\mathrm{O}(d)\) acts by conjugate multiplication.} Given the eigendecomposition \(\mP = \mO^T\mLambda\mO\) with eigenvalues \(\mLambda\) in an ascending/descending order, \(c(\mP)=\mLambda\) is a valid canonical form, as $\mLambda$ is uniquely defined and invariant to actions of $\Od$ on $\mP$, and furthermore, $\mLambda\in\mathrm{Orb}_{\mathrm{O}(d)}(\mP)$. 
    
    In many cases, canonicalization may be elusive. For example, consider the canonicalization of $\mP\in\R^{d\times n}$ where the orthogonal group instead acts by left multiplication. From~\cref{thm:minimal_frame}, we must decompose $\mP=\rho(g)\mR$ for $g\in\Od$, $\rho(g)\in\R^{d\times d}$, and $\mR\in\R^{d\times n}$ as the canonical form. This may be achieved by a QR decomposition of $\mP$, however, if $\mP$ does not have full column rank, $\mR$ will not be uniquely defined. This non-uniqueness raises issues in ensuring that the canonical form satisfies the $\Od$-invariance required by~\cref{def:can_form}. In the following section, we resolve such cases by demonstrating how canonicalization can instead be performed on an induced $G$-set to achieve a unique canonical form.
    
    \subsection{Frame Construction on Induced $G$-Sets}\label{sec:inducedG}
    

    As in the previous example of $\Od$ acting upon \orange{non-full column rank} matrices via left multiplication, there are cases in which practical attempts at deriving a canonical form instead give rise to a set of multiple candidates. In order to satisfy $G$-invariance, a candidate must be selected in a manner so as to preserve the uniqueness of the canonical form within $\orb Gx$. To do so, we introduce a $G$-equivariant function $\phi$ to remove the ambiguity, and instead compute the canonical form of $\phi(x)$ in a space $\mathcal S_\phi$ which we refer to as an \textit{induced $G$-set}. 
\begin{definition}[Induced \(G\)-set]\label{def:induced_Gset}
    An induced \(G\)-set \(\mathcal{S}_\phi\) is a \(G\)-set induced from the \(G\)-set \(\mathcal{S}\) through a \(G\)-equivariant function \(\phi: \mathcal{S} \rightarrow \mathcal{S}_\phi\) such that for all \(g \in G\) and \(x \in \mathcal{S}\), \(\phi(g \cdot x) = g \cdot \phi(x)\).
\end{definition}
A well-chosen $\phi$ can effectively remove ambiguity in the canonical form by computing the canonical form on $\mathcal S_\phi$ instead of on the original space $\mathcal S$, however, it is unclear how this will be useful for deriving a minimal frame on $\mathcal S$. As we state next and prove in~\cref{app:thm:induced_Gset_frame}, due to the $G$-equivariance of $\phi$, $\phi$ may be composed with a frame $\mathcal F_\phi$ on $\mathcal S_\phi$ to produce a frame on $\mathcal S$\blue{, although we note that this frame is not necessarily minimal}.    
\begin{restatable}{theorem}{thminducedGsetframe}\label{thm:induced_Gset_frame}
    Given a \(G\)-set \(\mathcal{S}\) and a \(G\)-equivariant function \(\phi: \mathcal{S} \rightarrow \mathcal{S}_\phi\), let \(\mathcal{F}_\phi: \mathcal{S}_\phi \rightarrow \mathcal{P}(G)\setminus \{\emptyset\}\) be a frame with the domain of \(\mathcal{S}_\phi\). Then, \(\mathcal{F}_\phi\circ\phi \) is a frame with the domain of \(\mathcal{S}\).
\end{restatable}
Interestingly, previous methods can be shown to be special cases of canonicalization on an induced $G$-set. For example, for $\phi(\mP) = \mP\mP^T$, the canonical form $c(\phi(\mP))$ is given by eigenvalues of $\phi(\mP)$, which is exactly the method of~\citet{puny2021frame}. Additionally, choosing the trivial map $\phi(\mP)=0$ gives the minimal frame $\hat{\mathcal F}_\phi(\phi(\mP))=G$ which corresponds to group averaging, as $c(\phi(\mP))=0$ and $\stab G0=G$. \todo{Reviewer 2, W2}
\blue{This illustrates that canonicalization can be more tractable on the induced $G$-set, though it is not necessarily the case that a minimal frame $\hat{\mathcal F}_\phi$ on the induced $G$-set produces a frame $\hat{\mathcal F}_\phi\circ\phi$ which is also minimal on the base space.} 

We next derive minimal frames for several groups acting on subsets of $\R^{d\times n}$. In~\cref{sec:linAlgG}, we employ an induced $G$-set to derive a minimal frame for the linear algebraic group $G_\eta(d)$, which is strictly more general than the previously mentioned example of $\Od$. In~\cref{sec:permG}, we go on to consider a case where the minimal frame can be directly computed on $\mathcal S$ for the permutation group $\text{S}_n$. Furthermore, in~\cref{tb:quotient}, we present induced $G$-sets which we leverage to compute canonical forms for a variety of groups and domains.

\section{Linear Algebraic Group}\label{sec:linAlgG}
    \todo{Reviewer 3, Q8}
    \colororange The linear algebraic group $G_\eta(d)$ includes several well-known groups: the orthogonal group $\Od$, the Lorentz group $\text{O}(1, d-1)$ and the unitary group $\mathrm{U}(d)$. By adding a constraint \blue{enforcing determinant equal to 1}, $G_\eta(d)$ can be extended to include groups like $\mathrm{SO}(d), \mathrm{SO}(1,d-1)$ and $\mathrm{SU}(d)$. \colorblack $G_\eta(d)$ acts on $\vx\in\R^d$ as \(g\cdot\vx\coloneq\mO\vx\), where $\rho(g)=\mO\in \R^{n\times n}$ satisfies \(\mO^T \eta \mO = \eta\), and \(\eta\) is a $d\times d$ diagonal matrix with diagonal elements \(\pm 1\). The pseudo-inner product on $\R^d$ is then defined by $\langle \vx, \vy\rangle \coloneqq \vx^T\eta \vy$. Setting $\eta=\mI_d$ results in $G_\eta(d)=\Od$ with the usual Euclidean inner product. However, as we aim to develop theory that goes beyond $\Od$, we encounter $\eta\neq\mI_d$, resulting in non-Euclidean, pseudo-inner products. For example, setting \(\eta = \text{diag}(1, -1,\cdots, -1)\in \R^{d\times d}\) gives \(G_\eta(d) = \mathrm{O}(1, d-1)\), the Lorentz Group. This generality introduces a multitude of challenges that lead us to define MFA frames on an induced $G_\eta(d)$-set in~\cref{sec:GLset} and to develop a generalized QR decomposition for canonicalization in~\cref{sec:genQR} \blue{before moving on to derive minimality and an efficient closed form of the frame averaging operator in~\cref{sec:mfa_Geta}.} Furthermore, to encompass the case of $G_\eta (d)$ acting on $\mathbb{C}^d$, MFA is extended to unitary groups in~\cref{sec:complex}.
        \subsection{Canonicalization on the Induced $G_\mathbf{\eta}(d)$-Set}\label{sec:GLset}

        We aim to construct a $G_\eta(d)$-equivariant map by applying the frame averaging operator over \todo{R3, Q9}\orange{a} frame $\hat{\mathcal F}(\mP)$, where $\mP\in\R^{d\times n}$ represents a collection of $n$ elements of $\R^d$. 
        From~\cref{thm:minimal_frame}, construction of $\hat{\mathcal F}$ requires a canonical form $\hat{\mR}=\hat{\mQ}^{-1}\mP$ for 
        $\rho^{-1}(\mQ)\in G_\eta(d)$\todo{R3, Q10}\orange{, where $\rho$ is the group representation}.
        While the QR decomposition can be applied as $\mP=\mQ\mR$, $\mR$ will only be unique in the case where $\mP$ has full column rank, which is necessary to satisfy the $G_\eta(d)$-invariance required by~\cref{def:can_form}. \orange{Furthermore, the QR decomposition involves a division by $\langle \vv_1, \vv_1\rangle$ for the first chosen columns $\vv_1$ of $\mP$. This may result in division by zero, as due to the non-Euclidean inner product, $\langle \vv_1, \vv_1\rangle$ may be $0$ even for $\vv_1\neq\mathbf{0}\in\R^d$, in which case we refer to $\vv_1$ as \textit{null}}. 
        
        For these reasons, we transform $\mP$ to an induced $G_\eta(d)$-set through the map $\phi$ which selects columns of $\mP$ such that $\phi(\mP)\in\R^{d\times d'}$ has full column rank and no null columns. Note that this selection is equivariant as required by~\cref{def:induced_Gset}, as $G_\eta(d)$ acts via left multiplication of $\mP$ and the selection is implemented as right multiplication with $\mM\in \R^{d\times d'}$ as \(\phi(\mP) = \mP\mM\). $\mM$ is constructed by removing the $j$-th column from $\mI_d$ if and only if the $j$-th column is null or is linearly dependent with respect to the other non-null vectors. \jacob{Note that when considering $\SOd$ and $\mM$ constructed so as to select three vectors within every local two-body system, the methods proposed by~\citet{du2022se} and~\citet{pozdnyakov2024smooth} can be seen as special cases of the approach we have taken here.}
        
        \subsection{Derivation of the Canonical Form via Generalized QR Decomposition}\label{sec:genQR}
        \todo{R3, Q12}
        
        Despite these considerations, there are still several obstacles preventing the computation of $\hat \mQ$ and $\hat \mR$ directly from the QR decomposition of $\phi(\mP)=\mQ\mR$, and we therefore develop a generalized QR decomposition which allows us to derive minimal frames for $G_\eta(d)$. First, there may not be $\hat g\in G_\eta(d)$ such that $\rho(\hat{g})=\mQ$, as $\mQ^T\eta\mQ$ \orange{may not result in $\eta$}\blue{, and instead give $\mS\eta \mS^T$, where $\mS$ is a permutation matrix}. This is due to \blue{the} negativity of the pseudo-inner product and the sequential nature with which the QR decomposition computes the columns of $\mQ$ via the Gram-Schmidt process. \blue{For example, consider $G_\eta(d) = \mathrm{O}(1,d-1)$ with $\eta = \text{diag}(1,-1,\cdots,-1)$. }\colororange 
        The first vector $\vv_1$ used for the Gram-Schmidt process could be time-like, \textit{i.e.}, $\langle \vv_1, \vv_1\rangle<0$, 
        \blue{such that}
        the first element of $\mQ^T \eta \mQ$ is $-1$ instead of $1$. 
         To counteract this, we introduce $\mS$ into our generalized QR decomposition as $\hat{\mQ}\coloneq\mQ{\mS}$ such that \colororange $\hat{\mQ}^T \eta \hat{\mQ} = \eta$. \colorblack Additionally, $\det(\mQ)=\pm 1$ because the QR decomposition results in an orthonormal $\mQ$, however, \orange{groups of $G_\eta(d)$ with \blue{additional determinant} constraints} such as $\SOd$ and $\mathrm{SO}(1,d-1)$ \blue{require} that $\det(\mQ)=1$. For these groups, we flip the sign of the elements in one of the columns of $\hat \mQ$ if $\det(\mQ)=-1$, which we detail further in~\cref{sec:frame_averaging_linear_algebraic_group}. These steps have ensured that there exists $\hat{g}\in G_\eta(d)$ such that $\rho(\hat{g})=\hat{\mQ}$.
        
        Second, \colororange the \blue{classical} QR decomposition \colorblack may not even be a valid decomposition of $\phi(\mP)$, that is, $\phi(\mP)\neq\mQ\mR$. \todo{Reviewer 3, Q11}This is due to division by \blue{the $L_2$-norm $\lVert\vu\rVert$} in the computation of the QR decomposition \blue{detailed in~\cref{sec:qr_decomposition}\todo{Reviewer 3, Q11}\orange{, where $\vu$ is an intermediate vector used in the Gram-Schmidt process}. While the standard definition of $\lVert\vu\rVert$ is $\sqrt{\langle \vu, \vu\rangle}$,} the pseudo-inner product may result in $\langle \vu, \vu\rangle<0$. 
        \blue{Redefining $\lVert\vu\rVert$ as $\sqrt{\lvert\langle \vu, \vu\rangle\rvert}$ circumvents a complex-valued norm, however, it also invalidates the decomposition. }
        We therefore define $\hat{\mR}\coloneqq \mS^T\mD_\eta\mR$, where both $\mS^T$ and $\mD_\eta$ serve to ensure $\hat{\mQ}\hat{\mR}$ is a valid decomposition of $\phi(\mP)$. $\mS^T$ cancels $\mS$ in $\hat \mQ$, while $\mD_\eta$ is a diagonal matrix with diagonal elements \((\mD_\eta)_{j,j}\coloneq \blue{\operatorname{sign}(\langle \vu_j,\vu_j\rangle)}\).
        The above steps have ensured that $\phi(\mP)=\hat{\mQ}\hat{\mR}$
        \todo{Reviewer 2, W1}\orange{with $\hat{\mQ}^T\eta \hat{\mQ} = \eta$.} \orange{In \cref{sec:qr_decomposition}, we further show the uniqueness of $\hat{\mR}$, \blue{ensuring that} $\hat{\mR}$ is a valid canonical form.}

\colorblue
\subsection{Minimal Frame Averaging on $G_\eta(d)$}\label{sec:mfa_Geta}
\colorblack

 \magenta{As discussed in~\cref{sec:inducedG}, since $\hat\mR$ was computed on an induced $G_\eta(d)$-set, it is not necessarily the case that frames derived from $\hat\mR$ will be minimal on the base space $\R^{d\times n}$. In the proof of the following result in~\cref{app:pf_thm:minimal_induced_by_QR}, we prove this base space minimality by leveraging~\cref{thm:minimal_frame,thm:induced_Gset_frame}.}
        \begin{restatable}{theorem}{thmminimalinducedbyQR}
        \label{thm:minimal_induced_by_QR}
            Let $\hat{\mathcal{F}}_\phi$ be a minimal frame on the \(G_\eta(d)\)-set induced by \(\phi(\mP) = \mP \mM\) computed via the generalized QR decomposition; if the columns of \(\mP\) are non-null, then $\hat{\mathcal{F}}_\phi\circ\phi$ is a minimal frame on the original domain $\mathbb{R}^{d\times n}$.
        \end{restatable}
        \magenta{Finally, we discuss practical computation of the frame averaging operator over $\hat{\mathcal  F}_\phi\circ\phi$. \cref{thm:minimal_frame} implies that we must next compute $\textrm{Stab}_{G_\eta(d)}(\hat{\mR})$, which is indeed the route we will take when $d'=d$. However, when $d'<d$, there are infinite $\hat{\mQ}$ satisfying $\phi(\mP)=\hat{\mQ}\hat{\mR}$, apparently making the operator intractable. We therefore employ an alternate form of the minimal frame given in~\cref{lemma:multiCanMaps} which we prove to be equivalent to the form established by~\cref{thm:minimal_frame}. This alternate form enables us to derive an efficient closed form of the operator shown in the following result, the proof of which is given in~\cref{app:pf_thm:linalG_op}.} 
        \begin{restatable}{theorem}{linalgGop}
        \label{thm:linalgG_op}
        \colororange
            Let $\hat{\mQ},\hat{\mR}$ be computed via the generalized QR decomposition such that $\phi(\mP)=\hat{\mQ}\hat{\mR}\in\R^{d\times d'}$\blue{, and assume that the columns of $\mP$ are non-null}. Then, there exists $\mQ_0\in \mathbb{R}^{d\times d}$ such that
                the frame averaging operator applied to an arbitrary function $\blue{\Phi}:\mathbb{R}^{d\times n}\rightarrow \mathbb{R}^{d\times n}$ \blue{over the minimal frame $\hat{\mathcal F}_\phi\circ\phi$} is given by $$\blue{\langle\Phi\rangle_{\hat{\mathcal F}_\phi\circ\phi}(\mP)} = \mQ_0 \Phi\left(\hat{\mQ}^{-1}\mP\right),$$
                \blue{where $d'$ columns of $\mQ_0$ are obtained directly from $\hat{\mQ}$, and the remaining $d-d'$ columns are all $\mathbf{0}\in\R^d$.}
        \end{restatable}
        \colorblack

        \subsection{Efficiency compared to~\citet{puny2021frame}.}
        \citet{puny2021frame} consider \(\mathrm{O}(d)\) and \(\mathrm{SO}(d)\). Their method is equivalent to employing the induced \(G\)-set \(\phi(\mP) = \mP\mP^T\) and performing canonicalization via the eigendecomposition, which cannot be used for groups such as \(\mathrm{O}(1,d-1)\). Furthermore, for a \todo{R1, W3}\orange{full column rank} non-degenerate (\textit{i.e.}, no eigenvalues are repeated) matrix \(\mP\), the size of their frame is \(2^d\), while the generalized QR decomposition is unique for both \(\hat\mQ\) and \(\hat\mR\) by~\cref{thm:qr_unique}, producing a frame with only \(1\) element. Additionally, the size of the frame obtained from the generalized QR decomposition remains 1 in the case of a \todo{R1, W3}\orange{full column rank} degenerate \(\mP\), however, relying on the eigendecomposition of \(\mP\mP^T\) in such a setting will produce a frame of infinite size. In~\cref{sec:ortho_ppt}, we provide an eigenvalue perturbation method to solve this degenerate case for the method of~\citet{puny2021frame}. However, 
        the size of the resulting frame is still \(2^d\), which is larger than the frame produced by the MFA method.


\subsection{Complex domain} 
\label{sec:complex}
The extension from the real vector space \(\mathbb{R}^{d\times n}\) to a complex vector space \(\mathbb{C}^{d\times n}\) is natural. By changing all the transpose operations \(\mQ^T\) to the conjugate transpose operations \(\mQ^*\) and inner products to conjugate inner products, the generalized QR decomposition is conducted as in the real space for the unitary group \(\mathrm{U}(d)\). For $\mathrm{SU}(d)$, we require that the magnitude of $\det(\hat{\mQ})\in \mathbb{C}$ be $1$. This constraint is satisfied through a scaling \(\det (\mQ)^{-1/d}\) and its inverse applied to $\hat{\mQ}$ and $\hat{\mR}$, respectively.

\section{Permutation Group}\label{sec:permG}

We now consider minimal frames for two different spaces under the action of the permutation group \(\mathrm{S}_n\) beginning with the space of undirected graphs in~\cref{sec:minF_udag}. Using the results from~\cref{sec:minF_udag}, we go on to derive minimal frames for point clouds acted upon by the group formed from the direct product of $\Sn$ and \blue{$G_\eta(d)$} in~\cref{sec:pcloud} under the assumption that we already possess an \blue{$G_\eta(d)$}-invariant/equivariant backbone model.

\subsection{Minimal Frame for Undirected Graphs}\label{sec:minF_udag}

$\Sn$ acts on the adjacency matrix for an undirected graph \(\mA\in\mathrm{Sym}(n,\mathbb{R})\) by conjugate multiplication as \(g\cdot\mA\coloneq\mS^T \mA \mS\), where $\rho(g)=\mS\in \R^{n\times n}$ and \(\mathrm{Sym}(n,\mathbb{R})\) is the set of \(n \times n\) real symmetric matrices. To compute the minimal frame, we leverage ties between our framework and results from classical graph theory. 

Canonicalization of an undirected graph, also known as canonical labeling, is commonly used to determine whether 2 graphs are isomorphic, \textit{i.e., } whether the graphs are equivalent up to a node relabeling. Here, $c(\mA_1)=c(\mA_2)$ if and only if the graph with adjacency matrix $\mA_1$ is isomorphic to the graph corresponding to $\mA_2$. Additionally, computing the stabilizer for a graph with adjacency matrix $\mA_0$ corresponds to the \textit{graph automorphism problem}, which is to compute the permutations $\mS$ which result in a self-isomorphism, \textit{i.e.,}
\begin{equation}\label{eq:grAut}
    \{\mS \in\Sn\mid \mS^T\mA_0 \mS =\mA_0\},
\end{equation}
where \cref{eq:grAut} can be recognized as $\stab{\Sn}{\mA_0}$. Thus, from~\cref{thm:minimal_frame}, the minimal frame $\hat{\mathcal F}(\mA)$ can be computed by identifying the canonical form $c(\mA)$ and calculating the stabilizer $\stab\Sn{c(\mA)}$ with a graph automorphism algorithm. Specifically, we adapt the canonical graph labeling algorithm of~\citet{mckay2014practical} detailed in~\cref{sec:mckay} which computes both the canonical form and the stabilizer. 
\todo{Reviewer 3, Q13, Q14}\colororange
The algorithm directly solves the \blue{graph} automorphism problem \blue{following an} individualization-refinement paradigm~\citep{mckay2007nauty}\blue{. Although this problem is known to be NP-intermediate, in practice,} the time complexity largely depends on the number of automorphisms. For graphs with a trivial stabilizer, the time complexity can be nearly linear as the search tree becomes a list, \blue{although} for highly-symmetric graphs, the search tree will expand, \blue{giving a factorial worst-case} time complexity. We further detail the conversion from weighted graphs to unweighted graphs in~\cref{sec:employ_canonical_labeling} \blue{for applying} this algorithm \blue{to} undirected weighted graphs.
\colorblack


\subsection{Efficiency compared to~\citet{puny2021frame}.}

\citet{puny2021frame} compute the frame for an adjacency matrix \(\mA\) by sorting the rows of the Laplacian matrix perturbed by the diagonal of the summation of eigenvector outer products, denoted by an equivariant map \(S(\mA)\). In other words, they compute the frame on an \(\Sn\)-set induced by \(\phi(\mA) = S(\mA)\). By~\cref{thm:induced_Gset_frame}, the resulting frame is a (possibly non-minimal) frame on the original domain. On the other hand, as MFA directly constructs frames via the stabilizer of the canonical form on the original domain \(\mathrm{Sym}(n, \mathbb{R})\), the resulting frame is minimal by~\cref{thm:minimal_frame}. 

Furthermore, \citet{puny2021frame} achieve invariant frame averaging by sampling elements from the frame, where achieving a greater degree of symmetry requires a larger sample and therefore incurs greater cost. In contrast, MFA reduces~\cref{eqn:fa} to a single forward pass of the backbone model $\Phi$ as
\begin{equation}
\label{eq:simplified_frame_averaging}
    \langle\Phi\rangle_{\hat{\mathcal{F}}}(\mA) = \frac{1}{|\hat{\mathcal{F}}(\mA)|} \left(\sum_{g\in\hat{\mathcal{F}}(\mA)} \rho(g)\right) \Phi(c(\mA)).
\end{equation}
This is because the MFA frames for adjacency matrices are constructed directly on the original space (as opposed to on an induced $\Sn$-set), and thus, for all $g\in \hat{\mathcal F}(\mA)$, $\rho(g^{-1})\mA=c(\mA)$, a property which follows from the form of the minimal frame in~\cref{thm:minimal_frame}. In addition to this substantial improvement in efficiency, our empirical results in~\cref{tb:invariance-error} demonstrate the superior invariance error of MFA relative to alternative frame averaging methods.

\subsection{Minimal Frame for Point Clouds}\label{sec:pcloud}

\blue{We now consider the group formed by the direct product of the permutation group $\Sn$ and the general linear algebraic group $G_\eta(d)$. For $\eta=\mI_d$, we have the group $\Sn\times \Od$, where further adding a determinant constraint $\det (\rho(g)) = 1$ for all $g\in G_\eta(d)$ gives the group $\Sn\times \SOd$. Both of these groups commonly arise in tasks involving Euclidean geometries such as molecular property prediction. Alternatively, particle collision simulations involve transformations in space-time from the group $\Sn\times \mathrm{O}(1,d-1)$ acting on relativistic point clouds, where $\eta=\text{diag}(1,-1,\cdots,-1)$.} 

\blue{Generally,} for a point cloud with $n$ points in $d$-dimensional space represented as $\mP\in\R^{d\times n}$, $\Sn$ acts by right multiplication, while \blue{$G_\eta(d)$} acts by left multiplication. Given a \blue{$G_\eta(d)$}-invariant/equivariant function $\Phi:\R^{d\times n}\to\R^{d\times n}$, we aim to construct a \blue{$\Sn\times G_\eta(d)$}-invariant/equivariant function. Formally, for \blue{$g=(s,o)\in\Sn\times G_\eta(d)$}, $g$ acts on $\mP$ as $g\cdot \mP=\mO\mP\mS$, where $\rho(g)=(\mS,\mO)$. Observe that the actions of $\Sn$ and \blue{$G_\eta(d)$} on $\mP$ commute, that is, $(so)\cdot \mP=(os)\cdot \mP$. Therefore, we leverage the following result from \citet{puny2021frame}.
\begin{theorem}[\citet{puny2021frame}]
\label{thm:prodG}
    For the groups \(G, H\) whose actions commute, assume the frame \( {\mathcal{F}: \mathcal{V} \to \mathcal{P}(G) \setminus \{\emptyset\}}\) is $H$-invariant and $G$-equivariant. If ${f:\mathcal V\to\mathcal W}$ is $H$-invariant/equivariant, then $\langle f\rangle_{\mathcal F}$ is $G\times H$-invariant/equivariant.
\end{theorem}
Thus, to construct a \blue{$\Sn\times G_\eta(d)$} invariant/equivariant function given $\Phi$, we must derive a frame which is \blue{$G_\eta(d)$}-invariant and $\Sn$-equivariant, which we will accomplish via the induced $G$-set approach described in~\cref{sec:inducedG}. Consider \blue{$\phi(\mP)=\mP^T\eta\mP$}, and observe that 
$$\phi(g\cdot\mP)=(\mO\mP\mS)^T\eta\mO\mP\mS=\mS^T\phi(\mP)\mS.$$
Thus, $\phi$ is \blue{$G_\eta(d)$}-invariant and $\Sn$-equivariant. 
Furthermore note that since $\phi(\mP)\in \mathrm{Sym}(n,\mathbb{R})$ and since $\Sn$ acts on the codomain of $\phi$ via conjugate multiplication, a minimal frame on the $\Sn$-set induced by $\phi$ can be derived following the methods in~\cref{sec:minF_udag}. Therefore, from~\cref{thm:induced_Gset_frame}, $\mathcal F\coloneq \mathcal{F}_\phi\circ \phi$ is a \blue{$G_\eta(d)$}-invariant, $\Sn$-equivariant frame on $\R^{d\times n}$. Thus, from~\cref{thm:prodG}, $\langle \Phi\rangle_{\mathcal F}$ is a \blue{$\Sn\times G_\eta(d)$}-invariant/equivariant function. \blue{We further prove minimality of this frame when $\text{rank}(\mP) = d$ in~\cref{thm:minProd}.}

In the case where $\Phi$ is $\Sn$-invariant/equivariant, one can alternatively apply a $\mathrm{S}_n$-invariant and \blue{$G_\eta(d)$}-equivariant frame using the $G$-set induced by $\phi(\mP) = \mP\mP^T$, which is the approach taken by~\citet{puny2021frame} \blue{for $G_\eta(d) = \Od$}. However, as stated previously \blue{for $\Od$}, degenerate eigenvalues result in a frame of infinite size, in which case the frame averaging operator is intractable. In contrast, our approach presented here instead leverages a $\mathrm S_n$-equivariant and \blue{$G_\eta(d)$}-invariant \blue{$\phi(\mP) = \mP^T\eta\mP$} and is therefore robust to degenerate eigenvalues \blue{and applicable to more groups}. It is worth noting that if there exists a non-trivial automorphism (or stabilizer) of \(\mP^T \mP\), our \blue{$\Sn\times G_\eta(d)$}-invariant \blue{frame averaging} is invariant to the action of point groups, which we detail further in~\cref{sec:point_group}.


\section{Related Work}

    Similar to~\citet{puny2021frame}, the $\mathrm{E}(3)$-equivariant frames $\mathcal F(\mP)$ from~\citet{duval2023faenet} are computed via an eigendecomposition with $\lvert\mathcal F(\mP)\rvert=8$ assuming there are no degenerate eigenvalues. \citet{duval2023faenet} sacrifice exact equivariance for efficiency by sampling a single $g\in \mathcal F(\mP)$ per forward pass such that only one model call is required instead of eight. The MFA approach detailed in~\cref{sec:linAlgG} also only requires one model call per forward pass, however, achieves exact equivariance, and is furthermore robust to degenerate eigenvalues. 
    
\cyan{
    \citet{lim2024expressive} employ a sign-equivariant (or $\mathrm{O}(1)^d$-equivariant) network with an eigenvector from the covariance matrix \jacob{of the data} to achieve $\Od$ equivariance, reducing the frame size to one. This method necessitates explicit sign-equivariance \jacob{of the backbone architecture} and distinct eigenvalues in the covariance matrix. Similarly, \citet{ma2024laplacian} select an eigenvector from the covariance matrix under three assumptions constraining the eigenvectors. Our approach does not require these assumptions, does not necessitate sign-equivariance, and is insensitive to repeated or zero eigenvalues. 
}
    
    \citet{kim2023learning} also take an approximately equivariant sampling-based approach to approximate the group averaging operator in~\cref{eqn:ga}. Because sampling uniformly from $g\in G$ has high variance, \citet{kim2023learning} instead sample from a learned $G$-equivariant distribution, thereby adding the requirement of a $G$-equivariant neural network to their framework. In generating elements from $\Od$, \citet{kim2023learning} employ the QR decomposition to orthogonalize a generated matrix, which by the analysis presented in~\cref{sec:genQR} may be a one-to-many mapping, potentially compromising the $G$-equivariance of the distribution.
    
    \citet{kaba2023equivariance} also employ a $G$-equivariant neural network to canonicalize the inputs of a non-equivariant backbone architecture. This can be viewed as a learned approach to canonicalization on induced $G$-sets described in~\cref{sec:inducedG} that instead utilizes a $G$-equivariant architecture for $\phi$ in place of the deterministic $\phi$ we use here. Because there may be more than one transformation that canonicalizes the input, \todo{Reviewer 3, W1}
    this approach could lead to an ill-posed objective for $\phi$ resulting from the one-to-many mapping. \orange{An example of such a case is a point cloud lying on a $(d-2)$-dimensional Euclidean subspace \blue{embedded} in a $d$-dimensional Euclidean space under the action of $\Od$, where the stabilizer has an infinite number of elements.} \blue{For this reason,} \citet{kaba2023equivariance} furthermore note that non-trivial stabilizers reduce their framework to respect a relaxed definition of equivariance. A further application based on \citet{kaba2023equivariance} is shown in \citet{mondal2024equivariant}.

    \blue{Because the methods of~\citet{kim2023learning} and~\citet{kaba2023equivariance} employ neural networks to produce group representations from the group $G$, both approaches rely on a \textit{contraction} step dependent on $G$ to ensure network outputs are valid representations. Since this contraction can be non-trivial to derive, \citet{nguyen2023learning} propose to add a loss term in place of the contraction as a soft constraint on network outputs to be valid representations. This allows for the method of~\citet{kim2023learning} to be extended to groups for which the derivation of a contraction is difficult, including $O(1,d-1)$. Nonetheless, in addition to approximate equivariance due to sampling, an additional source of equivariance error in this framework stems from invalid representations due to a non-zero loss.}
    \todo{Reviewer 2, W6}
    \colororange
    \blue{In contrast,} our method does not require optimization to achieve exact equivariance.
    \colorblack    

\section{Experiments}

\newcommand{\spc}{0.6cm}

\begin{table}[t]
\caption{MSE and inference time on the \(n\)-body experiment. Baselines are \textsc{\(\mathrm{SE}(3)\)-Tr.}~\cite{fuchs2020se}, TFN~\cite{thomas2018tensor}, EGNN~\cite{satorras2021n}, SEGNN~\cite{brandstetter2021geometric}, CN-GNN~\cite{kaba2023equivariance}, and CGENN~\cite{ruhe2023clifford}. FA-GNN and MFA-GNN are GNN backbones trained with FA~\cite{puny2021frame} and MFA, respectively.}
\label{tab:nbody}
\begin{center}
\begin{small}
\begin{sc}
\begin{tabular}{lcc}
\toprule
Method & MSE & Inference Time (s)  \\
\midrule
\(\mathrm{SE}(3)\)-Tr. & 0.0244 & 0.1346 \\
TFN & 0.0244 & 0.0343\\
EGNN & 0.0070 & 0.0062 \\
SEGNN & 0.0043 & 0.0030 \\
CN-GNN & 0.0043 & 0.0025\\
CGENN & 0.0039 & 0.0045 \\
\midrule
FA-GNN & 0.0054 & 0.0041 \\
MFA-GNN & \textbf{0.0036} & \textbf{0.0023} \\
\bottomrule
\end{tabular}
\end{sc}
\end{small}
\end{center}
\vspace{-\spc}
\end{table}

\begin{table}[t]
\caption{Accuracy and AUC for the top tagging dataset. The \textsc{MinkGNN} baseline is a non-invariant model built around the message-passing Minkowski dot-product attention module proposed by~\citet{gong2022efficient}. Extended results are presented in~\cref{tb:top-tagging}.}
\label{tb:top-tagging-ablation}
\begin{center}
\begin{small}
\begin{sc}
\begin{tabular}{lcc}
\toprule
Model & Accuracy & AUC \\
\midrule
MinkGNN & \cyan{94.2136} & 98.68  \\
MFA-MinkGNN & \cyan{\textbf{94.2178}}
&
 \textbf{98.69}
 \\
\bottomrule
\end{tabular}
\end{sc}
\end{small}
\end{center}
\vspace{-\spc}
\end{table}

\begin{table}[t]
\caption{Results for IS2RE Direct All on OC20 validation split. We compare in-distribution MAE (eV), average MAE (eV) across in-domain and out-of-domain tasks, invariance errors (eV), and throughput (samples per second) with \textsc{FAENet}~\cite{duval2023faenet}. Extended results are presented in~\cref{tab:OC}.}
\label{tb:ocp-faenet}
\begin{center}
\begin{small}
\begin{sc}
\begin{tabular}{lcc}
\toprule
Model & FAENet & MFAENet \\
\midrule
ID-MAE & 0.5446 & \textbf{0.5437} \\
Average-MAE &\textbf{0.5679}  & 0.5691\\
Invariance Error &  6.199$\times$10\textsuperscript{--2} &  \textbf{8.809$\mathbf{\times}$10\textsuperscript{--6}}\\
Throughput & 3863.3 & \textbf{3919.8}\\
\bottomrule
\end{tabular}
\end{sc}
\end{small}
\end{center}
\vspace{-\spc}
\end{table}

\begin{table}[t]
\caption{Results for WL datasets. \textsc{GRAPH8c} and EXP are counts of pairs of graphs that are not separated by randomly initialized models, while \textsc{EXP-classify} is the binary classification error for trained models. Baselines are GCN~\cite{kipf2016semi}, GIN~\cite{xu2018powerful}, and an MLP/GIN trained with FA~\cite{puny2021frame} and MFA, respectively.}
\label{tb:graph-separation}
\begin{center}
\begin{small}
\begin{sc}
\begin{tabular}{lccc}
\toprule
Model & GRAPH8c & EXP & EXP-classify (\%) \\
\midrule
GCN & 4755 & 600 & 50 \\
GIN & 386 &  600 & 50 \\
FA-MLP & \textbf{0} &  \textbf{0}  & \textbf{100}\\
FA-GIN & \textbf{0} &  \textbf{0}  & \textbf{100}\\
MFA-MLP & \textbf{0} &  \textbf{0}  & \textbf{100}\\
MFA-GIN & \textbf{0} &\textbf{0}  & \textbf{100} \\
\bottomrule
\end{tabular}
\end{sc}
\end{small}
\end{center}
\vspace{-\spc}
\end{table}

\begin{table}[t]
\caption{Invariance error and inference time on GRAPH8C. \textsc{FA-MLP}-$k$ denotes an MLP averaged over $k$ group elements randomly sampled from the $S_n$-invariant frame of~\citet{puny2021frame}.}
\label{tb:invariance-error}
\begin{center}
\begin{small}
\begin{sc}
\begin{tabular}{lcc}
\toprule
Model & \makecell{Invariance\\Error ($\times 10^{-2}$)} & \makecell{Inference\\Time (ms)} \\
\midrule
MLP & 8.60 & 0.27 \\
\midrule
FA-MLP-1 & 1.28 & 2.20 \\
FA-MLP-4 & 0.68 & 2.54 \\
FA-MLP-16 & 0.34 & 3.70 \\
FA-MLP-64 & 0.17 & 8.36 \\
FA-MLP-256 & 0.09 & 28.09 \\
MFA-MLP & \textbf{0.00} &  \textbf{1.10} \\
\bottomrule
\end{tabular}
\end{sc}
\end{small}
\end{center}
\vspace{-\spc}
\end{table}

\begin{table}[t]
\caption{MSE and invariance error for the convex hull experiment. Baselines are PointNet~\cite{qi2017pointnet} and CGENN~\cite{ruhe2023clifford}.}
\label{tb:convex-hull}
\begin{center}
\begin{small}
\begin{sc}
\begin{tabular}{lcc}
\toprule
Model & MSE & Invariance Error \\
\midrule
PointNet & 4.1475 & 8$\times$10\textsuperscript{--3} \\
CGENN & 3.5831 & 5$\times$10\textsuperscript{--2} \\
MFA-MLP & \textbf{1.5291} & \textbf{2$\mathbf\times$10\textsuperscript{--7}} \\
\bottomrule
\end{tabular}
\end{sc}
\end{small}
\end{center}
\vspace{-\spc}
\end{table}

Our study evaluates the effectiveness of MFA across a variety of tasks spanning diverse symmetry groups and multiple backbone architectures. The evaluation begins with the computation of the equivariance error for many groups in~\cref{sec:eq_test}. We go on to consider an $\mathrm{E}(3)$-equivariant \(n\)-body problem in~\cref{sec:nbody}, an $\mathrm{O}(1,3)$-invariant top tagging problem in~\cref{sec:tag}, a $\mathrm{SO}(3)$-invariant relaxed energy prediction problem in~\cref{sec:oc}, a $\Sn$-invariant Weisfeiler-Lehman test in~\cref{sec:WL}, and a $\Sn\times \mathrm{O}(5)$-invariant convex hull problem in~\cref{sec:hull}. Training details and model settings are presented in~\cref{sec:experimental_details}. 

\subsection{Equivariance Test on Synthetic Data}\label{sec:eq_test}
We present $G$-equivariance errors on synthetic point cloud data for a variety of groups in~\cref{sec:relative_error}. The error is computed using randomly initialized models $\Phi$ as
\begin{equation}
\label{eqn:relative_error}
    \mathbb E_{g,\vx}\lVert\Phi(g\cdot\vx )-g\cdot\Phi(\vx)\rVert_1,
\end{equation}
where $g$ is randomly sampled i.i.d.\ from $G$. The groups we consider are \(\mathrm{O}(d), \mathrm{SO}(d),\mathrm{U}(d),\mathrm{SU}(d), \mathrm{O}(1, d-1), \mathrm{SO}(1, d-1), \mathrm{E}(d),\mathrm{SE}(d), \mathrm{GL}(d, \mathbb{R})\) and $\mathrm{SL}(d,\mathbb{R})$. The analysis incorporates six different non-equivariant backbone architectures, including two variants of MLPs and two variants of GNNs. Where possible, we compare MFA to the method of \citet{puny2021frame} (FA), as well as to stochastic frame averaging (SFA)~\cite{duval2023faenet}. However, we note that for both of these methods, the groups for which frames have been derived are a strict subset of those we have derived here, and thus, a complete comparison across all groups, such as $\mathrm{O}(1,d-1)$, is not possible. We additionally demonstrate the robustness to degenerate eigenvalues of our \(\mathrm{O}(d)\)-equivariant frames. Lastly, we test the equivariance error for point groups with respect to \(\mathrm{S}_n\times \mathrm{O}(d)\) and \(\mathrm{S}_n\times \mathrm{O}(1,d-1)\). We include MLPs applied along the node dimension of the point cloud which, without frame averaging, are not permutation equivariant. In all settings, the equivariance error for MFA is 0 up to negligible numerical errors, whereas sampling-based FA approaches consistently introduce non-negligible equivariance errors.


\subsection{$\mathrm{E}(3)$: \(n\)-Body Problem}\label{sec:nbody}

In the \(\mathrm{E}(3)\)-equivariant $n$-body problem from~\citet{kipf2018neural,satorras2021n}, the prediction target is the positions of $n=5$ charged particles after a predetermined time interval. Each particle is defined by its initial position and velocity in \(\mathbb{R}^3\). Additionally, each pair of particles is associated with a scalar encoded as an edge feature representing their charge difference which determines whether the particles are attracted or repelled from one another. We adopt the \textsc{FA-GNN} backbone  from~\citet{puny2021frame} and train the model using FA and MFA with identical training configurations. As shown in~\cref{tab:nbody}, \textsc{MFA-GNN} achieves the best performance both in terms of MSE and inference time.

\subsection{$\mathrm{O}(1,3)$: Top Tagging}\label{sec:tag}
The task for the top tagging data~\cite{kasieczka_2019_2603256} is to classify hadronically decaying top quarks in particle collision simulations~\cite{kasieczka2019machine}. The \textit{top quark} is the heaviest-known elementary particle, however, due to its short lifespan, it is only feasible to study its decay as it hadronizes into a jet of smaller particles~\cite{ATL-PHYS-PUB-2022-039}. The resulting jets are difficult to distinguish from those stemming from light quarks and gluons, thereby giving rise to a binary classification task conditioned on a 4-vector of data for each of the $n$ constituent particles in the jet, where $n$ is as large as 200. As this task is invariant to transformations in space-time by elements of the Lorentz group $\mathrm{O}(1,d-1)$, it has led to the development of specialized $\mathrm{O}(1,d-1)$-invariant architectures
~\cite{gong2022efficient}. 
We adopt a non-$\mathrm{O}(1,d-1)$ invariant backbone referred to as \textsc{MinkGNN}. \textsc{MinkGNN} is built around the powerful message-passing Minkowski dot-product attention module proposed by~\citet{gong2022efficient} in designing their $\mathrm{O}(1,d-1)$-invariant architecture, however, \textsc{MinkGNN} additionally includes non-linearities in Minkowski space, thereby breaking the $\mathrm{O}(1,d-1)$-invariance. As shown in~\cref{tb:top-tagging-ablation}, returning invariance with \textsc{MFA-MinkGNN} improves accuracy, although due to the large training set of over 1 million jets, the benefits of our symmetry prior are limited, as \textsc{MinkGNN} is sufficiently expressive to extract symmetries during training.

\subsection{$\mathrm{SE(3)}$: Open Catalyst Project}\label{sec:oc}

To identify cost-effective electrocatalysts for energy storage, deep models have emerged as an alternative to costly quantum mechanical-based simulations. We consider the task of predicting the relaxed energy of an adsorbate interacting with catalyst conditioned on the initial atomic structure from the Open Catalyst (OC20) dataset~\cite{ocp_dataset}. As the energy is invariant to $\mathrm{SE}(3)$ transformations of the atomic structures, we evaluate the invariance error of randomly initialized models, which is defined analogously to the equivariance error from~\cref{sec:eq_test} as
\begin{equation}
\label{eqn:inv_error}
    \mathbb E_{g,\vx}\lVert\Phi(g\cdot\vx )-\Phi(\vx)\rVert_1,
\end{equation}
with $g$ again randomly sampled i.i.d. from $G$. As shown in~\cref{tb:ocp-faenet}, the invariance error for MFA is substantially lower than that of \textsc{FAENet}~\cite{duval2023faenet}. However, \textsc{FAENet} achieves a lower out-of-domain MAE. As the OC20 data has nearly 500K training samples, the data volume may be sufficient such that exact equivariance may not be vital. Nevertheless, MFA remains competitive for out-of-domain prediction and has a superior in-domain MAE. 

\subsection{$\mathrm{S}_n$: Graph Separation}\label{sec:WL}

The $\Sn$-invariant Weisfeiler-Lehman (WL) datasets considered by~\citet{puny2021frame} tasks models with separating and classifying graphs. The GRAPH8c data consists of non-isomorphic, connected 8-node graphs, while the graphs in EXP are distinguishable by the 3-WL test but not by the 2-WL test. As can be seen in~\cref{tb:graph-separation}, while the backbone models result in many failed tests, incorporating frame averaging leads to perfect performance across the board. We furthermore examine the invariance error and inference time on GRAPH8C in~\cref{tb:invariance-error}. We find that the sampling-based approach taken for the $\Sn$ group by~\citet{puny2021frame} requires a time-consuming large sample to achieve a low invariance error, whereas MFA achieves a perfect invariance error while maintaining efficiency.

\subsection{$\mathrm{S}_n \times \mathrm{O}(5)$: Convex Hull}\label{sec:hull}

Given a set of points in $d$-dimensional space, the \textit{convex hull} generated by this set is the convex set of minimum volume that contains all points. The convex hull dataset from~\citet{ruhe2023clifford} tasks models with computing the volume of the convex hull generated by sets of 5 dimensional points. This volume is invariant to both permutations and rotations of the points, and thus, we use the methods described in~\cref{sec:pcloud} to extend $S_n$-invariance to an $O(5)$-invariant MLP. As shown in~\cref{tb:convex-hull}, MFA achieves the best MSE as well as the best invariance error.


\colororange
\section{Limitations}
\todo{Reviewer 2, W4}
\blue{As shown by~\cref{thm:minimal_frame} and Theorem 3 of~\citet{puny2021frame}, the size of an equivariant frame is lower-bounded by the size of the stabilizer, which can be large for highly symmetric objects such as fully-connected graphs. A remaining challenge for frame averaging methods in general is therefore how to tractably compute the operator for large frames \orange{without compromising exact equivariance.}}
\todo{Reviewer 3, W2}
Furthermore, our proposed \blue{canonicalization} algorithm \blue{for $G_\eta(d)$-equivariant frames} may suffer from discontinuities, \blue{as determining the null and linearly dependent vectors in $\mP$}
is a discontinuous procedure. \blue{Recent work by~\citet{dym2024equivariant} highlights such discontinuities as a limiting factor for frame averaging methods. \citet{dym2024equivariant} prove that continuous canonicalizations for $\Sn,\Od$, and $\SOd$ do not exist and therefore propose the use of weighted frames with weak equivariance as a more robust alternative. }

\colorblack
    
\section{Conclusion}

In this work, we have introduced the MFA framework for achieving exact equivariance with frame averaging at a level of efficiency previously only achieved by approximately-equivariant approaches. The generality in our theoretical foundations have enabled us to extend MFA beyond the groups previously considered in the frame averaging literature. We have empirically demonstrated the utility of this approach on a diverse set of tasks and symmetries. While we have primarily focused on unstructured data, our general results provide a starting point for extending efficient frame averaging to architectures designed for regular grids. 

\section*{Acknowledgements}
We gratefully acknowledge the insightful discussions with Derek Lim and Hannah Lawrence. This work was supported in part by National Science Foundation grant IIS-2006861 and
National Institutes of Health grant U01AG070112.

\section*{Impact Statement}
This paper presents work whose goal is to advance the field of Machine Learning. There are many potential societal consequences of our work, none which we feel must be specifically highlighted here.

\bibliographystyle{icml2024}
\bibliography{frame}

\newpage
\appendix
\onecolumn
\section{Notations}
We list our notations in~\cref{tab:notations}.
\begin{table}[t]
\begin{center}
\caption{Notations.}
\label{tab:notations}
\begin{tabularx}{\textwidth}{lX}
\toprule
Notation & Meaning \\ 
\hline
\midrule
$G$ & A group. \\
$B$ & A Borel set.\\
$\text{Stab}_G$ & A stabilizer with respect to group $G$.\\
$\text{Orb}_G$ & An orbit with respect to group $G$.\\
$\mu_G$ & A $G$-invariant measure function. \\
$\mathcal{S}$ & A $G$-set.\\
$\mathcal{M}$ & A smooth manifold. \\
$\mathcal{L}^\infty(G)$ & The space of all bounded functions $f:G\rightarrow \mathbb{R}$.\\
$\mathcal{L}^\infty(G,\mathcal{W})$ & The space of all bounded functions $f:G\rightarrow \mathcal{W}$ where $\mathcal{W}$ is a real vector space.\\
$\mathcal{P}(G)$ & The powerset of $G$.\\
$\mathbb{R}^d$ & $d$-dimensional Euclidean space.\\
$\mathbb{C}^d$ & $d$-dimensional complex coordinate space.\\
$\mathbb{R}^{d\times n}_*$ & The set of all full-rank matrices in $\mathbb{R}^{d\times n}$.\\
$\mathbb{C}^{d\times n}_*$ & The set of all full-rank matrices in $\mathbb{C}^{d\times n}$.\\
$\mathrm{S}_n$ & The permutation group acting on a $n$-dimensional real vector space.\\
$\mathrm{O}(d)$ & The orthogonal group acting on a $d$-dimensional real vector space.\\
$\mathrm{SO}(d)$ & The special orthogonal group acting on a $d$-dimensional real vector space.\\
$\mathrm{O}(1,d-1)$ & The Lorentz group acting on a $d$-dimensional real vector space.\\
$\mathrm{SO}(1,d-1)$ & The proper Lorentz group acting on a $d$-dimensional real vector space.\\
$\mathrm{U}(d)$ & The unitary group acting on a $d$-dimensional complex vector space.\\
$\mathrm{SU}(d)$ & The special unitary group acting on a $d$-dimensional complex vector space.\\
$\mathrm{GL}(d, \mathbb{R})$ & The general linear group acting on a $d$-dimensional real vector space.\\
$\mathrm{GL}(d, \mathbb{C})$ & The general linear group acting on a $d$-dimensional complex vector space.\\
$\mathrm{SL}(d, \mathbb{R})$ & The special linear group acting on a $d$-dimensional real vector space.\\
$\mathrm{SL}(d, \mathbb{C})$ & The special linear group acting on a $d$-dimensional complex vector space.\\
$\mathrm{SPos}(d,\mathbb{R})$ & The set of all symmetric positive definite matrices in $\mathbb{R}^{d\times d}$.\\
$\mathrm{SPos}(d,\mathbb{C})$ & The set of all symmetric positive definite matrices in $\mathbb{C}^{d\times d}$.\\
$\mathrm{Sym}(d,\mathbb{R})$ & The set of all symmetric matrices in $\mathbb{R}^{d\times d}$.\\
$\mathrm{Sym}(d,\mathbb{C})$ & The set of all symmetric matrices in $\mathbb{C}^{d\times d}$.\\

\bottomrule
\end{tabularx}
\end{center}

\end{table}

\section{Mathematical Proofs}

    \subsection{\cref{thm:stab}}
        \begin{theorem}\label{thm:stab}
        For $g\in G$ and $x\in \mathcal S$, $\stab G{g\cdot x}=g\stab Gx g^{-1}$    
        \end{theorem}
        \begin{proof}    
        Let $h\in\stab Gx$, and observe that $ghg^{-1}\in\stab G{g\cdot x}$; therefore, $g\stab Gx g^{-1}\subseteq \stab G{g\cdot x}$.
        
        Next, let $h\in\stab G{g\cdot x}$, and observe that $g^{-1}hg\in\stab G{x}$; therefore, $g^{-1}\stab G{g\cdot x}g\subseteq \stab G{x}$, and thus, $\stab G{g\cdot x}\subseteq g\stab G{x}g^{-1}$ 
        \end{proof}

    \subsection{Proof of~\cref{lem:minimal_frame}}\label{app:pf_lem:minimal_frame}
        \lemminimalframe*
        \begin{proof}
        Let \(g \in \text{Stab}_G(x)\) and $\mathcal{F}$ be an arbitrary frame. From the definition of \(\text{Stab}_G(x)\) and the $G$-equivariance of $\mathcal F$, 
        \begin{equation*}
            \mathcal{F}(g \cdot x) = \mathcal{F}(x) = g\mathcal{F}(x)\coloneqq\{gh\mid h\in\mathcal F(x)\},    
        \end{equation*}
        which implies that \(gh \in \mathcal{F}(x)\) for all $h\in \mathcal F(x)$. Because this is true for all $g\in\text{Stab}_G(x)$, \(\text{Stab}_G(x)h \subseteq \mathcal{F}(x)\), and thus, by the $G$-equivariance of $\mathcal F$
        \begin{equation}\label{eq:stab_in_frame1}
            h^{-1}\text{Stab}_G(x)h \subseteq h^{-1}\mathcal{F}(x)=\mathcal{F}(x_0),  
        \end{equation}
        where $x_0\coloneqq h^{-1}x\in \text{Orb}_G(x)$. 
        By~\cref{thm:stab}, \(\text{Stab}_G(x_0)= h^{-1}\text{Stab}_G(x)h\), and thus, from~\cref{eq:stab_in_frame1}, 
        $\text{Stab}_G(x_0)\subseteq \mathcal F(x_0)$.    
        \end{proof}

    \subsection{Proof of~\cref{thm:minimal_frame}}\label{app:thm:minimal_frame}
        \thmminimalframe*

        \begin{proof}
            Suppose there exists a frame \(\mathcal F'\) and \(x \in \mathcal S\) such that \(\mathcal F'(x) \subset \mathcal F(x)\), \textit{i.e.}, $\mathcal F$ is not a minimal frame. As \(\mathcal F (x_0) = \mathcal F(h^{-1}\cdot x) = h^{-1}\mathcal F(x) = \text{Stab}_G (x_0)\), we obtain 
            \begin{equation}\label{eqn:frame_in_stab1}
            \mathcal F'(x_0) \subset  \text{Stab}_G (x_0).        
            \end{equation}
            Next, by~\cref{lem:minimal_frame}, there exists \(x_1 \in \text{Orb}_G (x)\)  such that \(\text{Stab}_G(x_1)\subseteq \mathcal F'(x_1)\). Since \(x_0\) and \(x_1\) are in the same orbit \(\text{Orb}_G (x)\), there exists \(h'\in G\) such that \(x_1 = h'\cdot x_0\). Observe that 
            \begin{equation}\label{eqn:hnotin}
                h'\notin  \text{Stab}_G(x_0),    
            \end{equation}
            as otherwise, \(x_1 = x_0\) and \( \text{Stab}_G (x_0) \subseteq \mathcal F'(x_0) \subset \text{Stab}_G (x_0)\), which is impossible. From~\cref{eqn:frame_in_stab1}, 
            \begin{equation}\label{eqn:stab1_in_stab0}
                \text{Stab}_G(x_1)\subseteq \mathcal F'(x_1) = h' \mathcal F'(x_0) \subset h'\text{Stab}_G(x_0).    
            \end{equation}
            Furthermore, from~\cref{thm:stab}, $\text{Stab}_G(x_1) = \text{Stab}_G(h'\cdot x_0) = h' \text{Stab}_G(x_0) h'^{-1}$,  
            and thus, \cref{eqn:stab1_in_stab0} implies that \(\text{Stab}_G(x_0) \subset \text{Stab}_G(x_0) h'\). This implies that $e\in \text{Stab}_G(x_0) h'$, since trivially $e\in \text{Stab}_G(x_0)$. However, this may only occur if there exists $j\in\text{Stab}_G(x_0)$ such that $jh'=e$, which gives that $h'^{-1}\in\text{Stab}_G(x_0)$. Finally, this implies that $h'\in \text{Stab}_G(x_0)$, which contradicts~\cref{eqn:hnotin}. As this contradiction was obtained through the assumption that $\mathcal F$ is not a minimal frame, \(\mathcal{F}\) is indeed a minimal frame.
        \end{proof}
    
    \subsection{Proof of~\cref{thm:induced_Gset_frame}}\label{app:thm:induced_Gset_frame}
        \thminducedGsetframe*
        \begin{proof}
            For any \(g \in G\) and \(x \in \mathcal{S}\), by the \(G\)-equivariance of \(\phi\), \(\mathcal{F}_\phi (\phi(g \cdot x)) = g \mathcal{F}_\phi(\phi(x))\). Thus, \(\mathcal{F}_\phi\circ\phi\) forms a frame on \(\mathcal{S}\).
        \end{proof}

    \subsection{\cref{lemma:multiCanMaps}}
    \begin{lemma}\label{lemma:multiCanMaps}
        For some canonical form $c(x)$, the minimal frame $\hat{\mathcal F}$ is given by $\hat{\mathcal F}(x)=\{g\in G\mid g\cdot c(x)= x\}$. 
    \end{lemma}
    \begin{proof}
        Let $\mathcal H_x\coloneq \{g\in G\mid g\cdot c(x)=x\}$, and observe that since $e\in\stab G{c(x)}$, 
        \begin{equation}\label{eq:hx_in_prd}
        \mathcal H_x\subseteq \mathcal H_x\stab G{c(x)},            
        \end{equation}
        where $\mathcal H_x\stab G{c(x)}\coloneq\{hg\mid (h,g)\in\mathcal H_x\times\stab G{c(x)}\}$. Next, for all $(h,g)\in \mathcal H_x\times\stab G{c(x)}$,
        \begin{equation}
            hg\cdot c(x)=h\cdot c(x)=x,    
        \end{equation}
        which implies that $hg\in\mathcal H_x$. Since this is true for an arbitrary $(h,g)\in \mathcal H_x\times\stab G{c(x)}$,
        \begin{equation}\label{eq:prd_in_hx}
            \mathcal H_x\stab G{c(x)}\subseteq \mathcal H_x. 
        \end{equation}
        From~\cref{eq:hx_in_prd,eq:prd_in_hx}, $\mathcal H_x=\mathcal H_x\stab G{c(x)}$. Finally, from~\cref{thm:minimal_frame}, for all $h\in \mathcal H_x$, the minimal frame is given by $\hat{\mathcal F}(x)=h\stab G{c(x)}$.
        Therefore, $\hat{\mathcal F}(x)=\mathcal H_x\stab G{c(x)}=\mathcal H_x$.
    \end{proof}

    \subsection{Proof of~\cref{thm:minimal_induced_by_QR}}\label{app:pf_thm:minimal_induced_by_QR}
        \thmminimalinducedbyQR*


        \begin{proof}
        Let \(d' = \text{rank}(\phi(\mP))\). We discuss two cases.
        
        \textbf{Case I. $d = d'$.} By the uniqueness of the generalized QR decomposition \(\phi(\mP) = \hat{\mQ}\hat{\mR} =  (\mQ\mS)(\mS^T\mD_\eta\mR)\) in~\cref{thm:qr_unique}, \(\text{Stab}_{G_\eta}(\mS^T\mD_\eta\mR) = \{\mI_d\}\). Thus, the minimal frame \(\hat{\mathcal{F}}_{\phi}(\phi(\mP)) = \{\hat{\mQ}\} = \{\mQ\mS\} \) has only a single element. By~\cref{thm:induced_Gset_frame}, \(\hat{\mathcal{F}}_{\phi}\circ \phi\) is a frame on the original domain. Since the only proper subset of \(\hat{\mathcal{F}}_{\phi}(\phi(\mP))\) is the empty set, by~\cref{def:min_frame_def}, \(\hat{\mathcal{F}}_{\phi}\circ\phi\) is a minimal frame on the original domain.
        
        \textbf{Case II. \(d > d'\).} Please refer to~\cref{sec:frame_averaging_linear_algebraic_group} for details on practical computation of the minimal frame. This case corresponds to Case II in~\cref{sec:qr_decomposition}, which says that $\hat \mQ$ is not uniquely determined in the generalized QR decomposition. That is, there exists a set $\mathcal Q_\mP$ such that for all $\hat{\mQ}\in\mathcal Q_\mP$, $\phi(\mP)=\hat{\mQ}\hat{\mR}$. Let $\mathcal H_\mP\coloneq\{g\in G_\eta(d)\mid \rho(g)\in\mathcal Q_\mP\}$. From~\cref{lemma:multiCanMaps}, the minimal frame on the induced $G_\eta(d)$-set is then given by
        \begin{equation}\label{eq:form_of_glFr}
            \hat{\mathcal F_\phi}(\phi(\mP))=\mathcal H_\mP.
        \end{equation} 
        Furthermore, from the definition of $G_\eta(d)$, $\hat{\mQ}^T \eta \hat{\mQ} = \eta$, and thus, we obtain 
        $$
            \hat{\mQ}^{-1} = \eta \hat{\mQ}^T \eta.
        $$
        Additionally, by definition of the QR decomposition, $d'$ columns of $\hat{\mQ}$ are uniquely determined and span the column space of $\phi(\mP)$. From the definition of $\phi(\mP)=\mP\mM$, the $d'$ columns also span the column space of $\mP$, as we have assumed that all columns of $\mP$ are non-null, and thus, $\phi(\mP)$ serves only to remove linearly dependent columns. The remaining $d-d'$ columns are arbitrary $d$-vectors which are orthonormal to the column space of $\mP$. Therefore, 
        $$\mP'\coloneq\hat{\mQ}^{-1}\mP = \eta \hat{\mQ}^T \eta \mP$$ 
        is a fixed value for all $\hat{\mQ}\in\mathcal Q_\mP$, since \(\hat{\mQ}^T \eta \mP\) calculates the inner product between the columns of \(\hat{\mQ}\) and \(\mP\), and the inner products involving the non-unique columns of $\hat \mQ$ are all 0. Thus, for all $g\in\mathcal H_\mP$, $\mP'=g^{-1}\cdot \mP$. Since \(\mP'\) and \(\mP\) are in the same orbit, there exists an element \(g' \in G_\eta (d)\) such that $g'\cdot \mP' = \mP$, which implies that 
        $$
        \mP'=g'^{-1}g\cdot \mP'.
        $$
        Thus, $g'^{-1}g\in\stab{G_\eta(d)}{\mP'}$. Since this is true for an arbitrary $g\in\mathcal H_\mP$, from~\cref{eq:form_of_glFr},
        \begin{equation}\label{eq:h_in_stab}
            g'^{-1}\mathcal H_\mP=g'^{-1}\hat{\mathcal F_\phi}(\phi(\mP))=\hat{\mathcal F_\phi}(\phi(g'^{-1}\cdot\mP))=\hat{\mathcal F_\phi}(\phi(\mP'))\subseteq\stab{G_\eta(d)}{\mP'}.            
        \end{equation}
        Next, by~\cref{lem:minimal_frame}, there exists \(\mP''\in \text{Orb}_{G_\eta(d)} (\mP)\) such that 
        \begin{equation}\label{eq:stab_in_indF}
        \text{Stab}_{G_\eta(d)}(\mP'') \subseteq \hat{\mathcal{F}}_{\phi}(\phi(\mP'')).    
        \end{equation}
        Since $\mP'$ and $\mP''$ are in the same orbit, there exists $g''$ such that \(\mP' = g''\cdot \mP''\). Thus, by~\cref{eq:h_in_stab}, 
        \begin{equation}\label{eq:indF_in_stab}
            \hat{\mathcal{F}}_{\phi}(\phi(\mP'))  = \hat{\mathcal{F}}_{\phi}(\phi(g''\cdot \mP'')) = g'' \hat{\mathcal{F}}_{\phi}(\phi(\mP'')) \subseteq \text{Stab}_{G_\eta(d)}(\mP').    
        \end{equation}
        Putting together~\cref{eq:stab_in_indF,eq:indF_in_stab} gives that
        \begin{equation}
            g''\text{Stab}_{G_\eta(d)}(\mP'')\subseteq g''\hat{\mathcal{F}}_{\phi}(\phi(\mP'')) \subseteq \text{Stab}_{G_\eta(d)}(\mP') = \text{Stab}_{G_\eta(d)}(g'' \cdot \mP'') = g''\text{Stab}_{G_\eta(d)}(\mP'')g''^{-1},    
        \end{equation}
        where the last equality follows from~\cref{thm:stab}. Therefore, \(\text{Stab}_{G_\eta(d)}(\mP'') \subseteq \text{Stab}_{G_\eta(d)}(\mP'')g''^{-1}\), so \(g''=e\), implying that \(\mP' = \mP''\). Together with~\cref{eq:stab_in_indF}, this implies that $\text{Stab}_{G_\eta(d)}(\mP') \subseteq \hat{\mathcal{F}}_{\phi}(\phi(\mP'))$. Furthermore recall from~\cref{eq:h_in_stab} that $\hat{\mathcal{F}}_{\phi}(\phi(\mP'))\subseteq \text{Stab}_{G_\eta(d)}(\mP')$, implying that 
        $$\hat{\mathcal{F}}_{\phi}(\phi(\mP')) = \hat{\mathcal{F}}_{\phi}(\phi(g^{-1}\cdot \mP))= g^{-1}\hat{\mathcal{F}}_{\phi}(\phi(\mP))= \text{Stab}_{G_\eta(d)}(\mP').$$
        And thus:
        \begin{equation}\label{eq:indF_eq_stab}
            \hat{\mathcal{F}}_{\phi}(\phi(\mP))= g\text{Stab}_{G_\eta(d)}(\mP').
        \end{equation}
        On the other hand, $\mP'\in\orb{G_\eta(d)}\mP$; therefore, \cref{eq:indF_eq_stab} together with~\cref{thm:minimal_frame} imply that \(\hat{\mathcal{F}}_{\phi}\circ\phi\) is a minimal frame on the original domain.
        \end{proof}

        \colorblue
    \subsection{Proof of~\cref{thm:linalgG_op}}\label{app:pf_thm:linalG_op}
        \linalgGop*
        \begin{proof}
        Let \(d' = \text{rank}(\phi(\mP))\). We discuss two cases.
        
        \textbf{Case I. $d = d'$.} By the uniqueness of $\hat{\mQ}$ and $\hat{\mR}$ in the generalized QR decomposition \(\phi(\mP) = \hat{\mQ}\hat{\mR}\) given by~\cref{thm:qr_unique}, \(\text{Stab}_{G_\eta}(\hat{\mR}) = \{\mI_d\}\). Thus, by~\cref{thm:minimal_frame}, the minimal frame \(\hat{\mathcal{F}}_{\phi}(\phi(\mP)) = \{\orange{\rho^{-1}}(\hat{\mQ})\}\) has only a single element. Therefore, \orange{for $\mQ_0\coloneq\hat{\mQ},$} the frame averaging operator is given by
        $$
        \langle\Phi\rangle_{\hat{\mathcal{F}}_{\phi}\orange{\circ \phi}}(\mP) = \frac{1}{|\hat{\mathcal{F}}_{\phi}(\phi(\mP))|} \sum_{g\in\hat{\mathcal{F}}_{\phi}(\phi(\mP))} \orange{\rho(g)} \Phi(\orange{\rho(g^{-1})}\mP)=\orange{\mQ_0} \Phi(\orange{\hat{\mQ}^{-1}}\mP).
        $$
        
        \textbf{Case II. \(d > d'\).} This case corresponds to Case II in~\cref{sec:qr_decomposition}, which says that $\hat \mQ$ is not uniquely determined in the generalized QR decomposition. That is, there exists a set $\mathcal Q_\mP\subset\R^{d\times d}$ such that for all $\hat{\mQ}\in\mathcal Q_\mP$, $\phi(\mP)=\hat{\mQ}\hat{\mR}$. Let $\mathcal H_\mP\coloneq\{g\in G_\eta(d)\mid \rho(g)\in\mathcal Q_\mP\}$. From~\cref{lemma:multiCanMaps}, the minimal frame is then given by
        \begin{equation}\label{eq:form_of_glFr}
            \hat{\mathcal F_\phi}(\phi(\mP))=\mathcal H_\mP.
        \end{equation} 
        Since $\mathcal H_\mP$ may be infinite, we express the frame averaging operator introduced in discrete form in~\cref{eqn:fa} in integral form as
        \begin{equation}\label{eq:casetwo_int}            
        \langle\Phi\rangle_{\hat{\mathcal{F}}_{\phi}\orange{\circ \phi}}(\mP) =  \int_{\mathcal Q_\mP} \mQ\Phi(\mQ^{-1}\mP)d\mu_{\mathcal Q_\mP}(\rho^{-1}(\mQ)),
        \end{equation}
        where $\mu_{\mathcal Q_\mP}$ is a uniform probability measure over $\mathcal Q_\mP$.
        We next prove the tractability of this integral. From the definition of $G_\eta(d)$, $\hat{\mQ}^T \eta \hat{\mQ} = \eta$, and thus, we obtain 
        $$
            \hat{\mQ}^{-1} = \eta \hat{\mQ}^T \eta.
        $$
        \orange{        Additionally, by definition of the QR decomposition, $d'$ columns of $\hat{\mQ}$ are uniquely determined and span the column space of $\phi(\mP)$. From the definition of $\phi(\mP)=\mP\mM$, the $d'$ columns also span the column space of $\mP$, as we have assumed that all columns of $\mP$ are non-null, and thus, $\phi(\mP)$ serves only to remove linearly dependent columns. The remaining $d-d'$ columns are arbitrary $d$-vectors which are orthonormal to the column space of $\mP$. Therefore, 
        $$\hat{\mQ}^{-1}\mP = \eta \hat{\mQ}^T \eta \mP$$ 
        is a fixed value for all $\hat{\mQ}\in\mathcal Q_\mP$, since \(\hat{\mQ}^T \eta \mP\) calculates the inner product between the columns of \(\hat{\mQ}\) and \(\mP\), and the inner products involving the non-unique columns of $\hat \mQ$ are all 0 due to orthonormality.} This implies that~\cref{eq:casetwo_int} can be re-written as 
        \begin{equation}\label{eq:casetwo_int_two}            
        \langle\Phi\rangle_{\hat{\mathcal{F}}_{\phi}\orange{\circ \phi}}(\mP) =  \int_{\mathcal Q_\mP} \mQ\Phi(\orange{\hat{\mQ}^{-1}\mP})d\mu_{\mathcal Q_\mP}(\orange{\rho^{-1}(\mQ)})=  \left(\int_{\mathcal Q_\mP} \mQ d\mu_{\mathcal Q_\mP}(\orange{\rho^{-1}(\mQ)})\right)\orange{\Phi(\hat{\mQ}^{-1}\mP)},
        \end{equation}
        \orange{where $\hat\mQ$ can be any element from $\mathcal Q_\mP$ which we choose as the $\hat\mQ$ produced by the generalized QR decomposition for convenience.} Thus, it only remains to be shown that $\int_{\mathcal Q_\mP} \mQ d\mu_{\mathcal Q_\mP}(\orange{\rho^{-1}(\mQ)})$ is a tractable integral. Without loss of generality, we partition ${\mQ}\in\mathcal Q_\mP$ into $${\mQ}=[{\mQ}^0, {\mQ}^1],$$
        where ${\mQ}^0\in \R^{d\times d'}$ are the $d'$ uniquely determined columns shared by all elements of $\mathcal Q_\mP$ and ${\mQ}^1\in \R^{d\times (d-d')}$ are the $(d-d')$ arbitrary non-unique columns which may be freely chosen. Thus, for all $\mQ=[{\mQ}^0, {\mQ}^1]\in\mathcal Q_\mP$, there exists $\mQ'\in\mathcal Q_\mP$ such that $\mQ'=[{\mQ}^0, -{\mQ}^1]$, with $$\orange{\frac12}(\mQ+\mQ')=\orange{{\mQ}_0\coloneq}[{\mQ}^0, \mathbf{0}^{d\times (d-d')}].$$ We may therefore partition $\mathcal Q_\mP$ as $$
        \mathcal Q_\mP=\mathcal Q_\mP^0\cup \mathcal Q_\mP^1,
        $$
         where $\mathcal Q_\mP^0$ and $\mathcal Q_\mP^1$ are constructed such that for all ${\mQ}\in\mathcal Q_\mP^0$, there exists ${\mQ}'\in\mathcal Q_\mP^1$ such that $\orange{\frac12}(\mQ+\mQ')=\orange{{\mQ}_0}$ and for all ${\mQ}'\in\mathcal Q_\mP^1$, there exists ${\mQ}\in\mathcal Q_\mP^0$ such that $\orange{\frac12}(\mQ+\mQ')=\orange{\mQ_0}$. Thus, by the construction of $\mathcal Q_\mP^0$ and $\mathcal Q_\mP^1$,
         \todo[inline]{We might add some more details here}
         \begin{equation}
            \begin{split}
                \int_{\mathcal Q_\mP} \mQ d\mu_{\mathcal Q_\mP}(\orange{\rho^{-1}(\mQ)}) &= \int_{\mathcal Q_\mP^0} \mQ d\mu_{\mathcal Q_\mP}(\orange{\rho^{-1}(\mQ)})+\int_{\mathcal Q_\mP^1} \mQ d\mu_{\mathcal Q_\mP}(\orange{\rho^{-1}(\mQ)}) \\
                &= \int_{\mathcal Q_\mP^0} \mQ d\mu_{\mathcal Q_\mP}(\orange{\rho^{-1}(\mQ)})+\int_{\mathcal Q_\mP^0} (2\mQ_0 - \mQ) d\mu_{\mathcal Q_\mP}(\orange{\rho^{-1}(\mQ)}) \\
                &= 2\mQ_0\int_{\mathcal Q_\mP^0} d\mu_{\mathcal Q_\mP}(\orange{\rho^{-1}(\mQ)}) \\
                &= {\mQ}_0,
            \end{split}             
         \end{equation}
where the final equality follows because 
$$
\int_{\mathcal Q_\mP^0} d\mu_{\mathcal Q_\mP}(\orange{\rho^{-1}(\mQ)})= \mu_{\mathcal Q_\mP}({\mathcal{Q}^0_\mP})=\frac12,
$$
since 
$\mu_{\mathcal Q_\mP}({\mathcal{Q}^0_\mP})=\mu_{\mathcal Q_\mP}({\mathcal{Q}^1_\mP})$ and $\mu_{\mathcal Q_\mP}({\mathcal{Q}^0_\mP})+\mu_{\mathcal Q_\mP}({\mathcal{Q}^1_\mP})=\mu_{\mathcal Q_\mP}({\mathcal{Q}_\mP})=1$. \cref{eq:casetwo_int_two} can now be re-written as
        \begin{equation*}            
        \langle\Phi\rangle_{\orange{\hat{\mathcal{F}}_{\phi}\circ\phi}}(\mP) =  {\mQ}_0\Phi(\orange{\hat{\mQ}^{-1}\mP}).
        \end{equation*}
    \end{proof}
    \colorblack
        
\section{Group Basics}
\label{sec:group_intro}
\paragraph{Group definition.} Let \(G\) be a set and \(*\) be a binary operation on \(G\). Then \((G,*)\) is a group if it satisfies the following properties:
\begin{enumerate}
    \item \textbf{Closure}: \(\forall a,b\in G\), the result of the operation \(a* b \in G\).
    \item \textbf{Associativity}: \(\forall a,b,c\in G\), the equation \((a* b)\cdot c = a* (b* c)\) holds.
    \item \textbf{Identity}: \(\exists e\in G\) such that \(\forall a\in G\), the equation \(a * e = e * a = a\) holds.
    \item \textbf{Inverse}: \(\forall a\in G\), \(\exists b\in G\) such that the equation \(a * b = b * a = e\) holds.
\end{enumerate}
For the purposes of this discussion, we will simplify the notation of the group and group operation. Let \(G\) be a group, and for any elements $a, b \in G$, we will denote the group operation, typically expressed as a binary operation $a * b$, simply as $ab$.

\paragraph{Group action.} Let \(G\) be a group and \(S\) be a set. The left group action of \(G\) on \(S\) is a mapping \(\cdot: G\times S \rightarrow S\) (often denoted simply as \(g\cdot x\) for \( g\in G, x\in S\)), satisfying that for \(e\in G\) and \(x\in S\), the equation \(e\cdot x = x\) holds, and \(\forall g,h\in G\) and \(x\in S\), the equation \(g\cdot (h\cdot x) = (gh)\cdot x\) holds. The right group action can be defined similarly. If for a group \(G\) and a set \(\mathcal{S}\), there exists such (left) group action $\cdot$, then \(\mathcal{S}\) is called a (left) \textbf{\(G\)-set}. Furthermore, the action of \(G\) on \(\mathcal{S}\) is transitive if \(\forall x,y\in \mathcal{S},\exists g\in G, y = g\cdot x\). If a group \(G\) acts transitively on a \(G\)-set \(\mathcal{S}\), then \(\mathcal{S}\) is called a \textbf{homogeneous space} of \(G\). For \(x\in \mathcal{S}\), the set of the group elements fixing \(x\) form a subgroup of \(G\) called the \textbf{stabilizer} of \(x\) denoted by
\begin{equation}
    \text{Stab}_G(x) = \{g\mid g\cdot x = x\} \subseteq G,
\end{equation}
and the set of all group elements acting on \(x\) is called the \textbf{orbit} of \(x\) denoted by
\begin{equation}
    \text{Orb}_G(x) = \{g\cdot x\mid g\in G\}\subseteq \mathcal{S}.
\end{equation}
Let \(\hat{\mathcal{S}}\) be another \(G\)-set. A mapping \(f:\mathcal{S}\rightarrow\hat{\mathcal{S}}\) is \textbf{equivariant} if \(\forall x\in \mathcal{S}, \forall g\in G\),
\begin{equation}
    f(g\cdot x) = g\cdot f(x),
\end{equation}
and \(f\) is \textbf{invariant} if \(\forall x\in \mathcal{S}, \forall g\in G\),
\begin{equation}
    f(g\cdot x) = f(x).
\end{equation}

\paragraph{Lie group.} A Lie group is a group and also a smooth manifold, such that both group binary operation and the inversion map are smooth. The general linear group \(\mathrm{GL}(n,\mathbb{R})\) is a Lie group consisting of all invertible \(n\times n\) matrices. A linear or matrix Lie group refers to a Lie subgroup of \(\mathrm{GL}(n,\mathbb{R})\).




\section{Frame Averaging on General Domain}
\label{sec:general_domain}

To describe a general domain, we employ a \( G \)-set \( \mathcal{S} \).
 Consider a \( \sigma \)-algebra \( \Sigma \) over \( G \) and define a measure \( \mu_G: \Sigma \rightarrow [0, +\infty] \) that satisfies the following properties:

\begin{enumerate}
    \item \textbf{Non-negative}: For any set \( X \in \Sigma \), the measure \( \mu_G (X) \) is non-negative, i.e., \( \mu_G (X) \ge 0 \).
    \item \textbf{Null Set}: The measure of the null set is zero, i.e., \( \mu_G (\emptyset) = 0 \).
    \item \textbf{\( \sigma \)-additivity}: For any countable collection \( \{X_i\} \) of pairwise disjoint sets in \( \Sigma \), the measure is \( \sigma \)-additive, i.e., \( \mu_G (\bigcup_{i=1}^\infty X_i) = \sum_{i=1}^\infty \mu_G(X_i) \).
    \item \textbf{\( G \)-invariance}: The measure is \( G \)-invariant, meaning that for any \( g \in G \) and \( X \in \Sigma \), it holds that \( \mu_G (gX)=\mu_G (Xg) = \mu_G (X) \), where \( gX = \{g\cdot x \mid x \in X\} \) and similarly \( Xg = \{x\cdot g \mid x \in X\} \).
\end{enumerate}

The combination of \( G \), \( \Sigma \), and \( \mu_G \) forms a measure space \(\left(G,\Sigma, \mu_G \right)\). For simplicity, we only consider the group \(G\) with left actions on \(\mathcal{S}\), while the right actions can be defined similarly. In this context, consider the process of frame averaging within this measure space.

\begin{theorem}[Frame Averaging on $G$-set]
\label{thm:fa}
    Consider a group \(G\) with the above measure space \(\left(G, \Sigma, \mu_G\right)\) and left actions on \(\mathcal S\) and a \( K \)-dimensional vector space \( \mathcal{W} \). Given a continuous bounded function \( \Phi \in \mathcal{L}^\infty(\mathcal{S}, \mathcal{W}) \) and a \( G \)-equivariant frame \( \mathcal{F}:\mathcal{S} \rightarrow \mathcal{P}(G) \setminus \{\emptyset\} \), the frame averaging
    \begin{equation}
        \langle \Phi \rangle_{\mathcal{F}}(x) = \int_{\mathcal{F}(x)} g \cdot \Phi(g^{-1} \cdot x) \, d\mu_G(g)
    \end{equation}
    is \( G \)-equivariant, where \( \mathcal{F}(x) \) is a compact Borel subset of \( G \) and \( \mu_G\left(\mathcal{F}(x)\right) < +\infty\) for all \( x \in \mathcal{S} \).
\end{theorem}

\begin{proof}
Given any $h\in G$,
\begin{equation}
\begin{split}
    \left<\Phi\right>_{\mathcal{F}}(h\cdot x) &= \int_{\mathcal{F}(h\cdot x)} g\cdot \Phi(g^{-1}h\cdot x)d\mu_G(g) \\
    &= \int_{h\mathcal{F}(x)} g\cdot \Phi(g^{-1}h\cdot x)d\mu_G(g)\\
    &= \int_{\mathcal{F}(x)} hg\cdot \Phi(g^{-1}h^{-1}h\cdot x)d\mu_G(hg) \\
    &= h\cdot \int_{\mathcal{F}(x)} g\cdot \Phi(g^{-1}\cdot x)d\mu_G(g) \\
    &= h\cdot \left<\Phi\right>_{\mathcal{F}}(x).
\end{split}
\end{equation}
\end{proof}
\colorblue
In a similar manner, the group convolutional form of frame averaging can be proposed. Let $\Phi$ be a continuous bounded function in $\mathcal{L}^\infty(\mathcal{S}, \mathbb{R})$ and $K$ a kernel in $\mathcal{L}^\infty(\mathcal{S}, \mathbb{R})$. We obtain the $G$-equivariant group convolution
\begin{equation}
        \langle \Phi * K \rangle_{\mathcal{F}}(g) = \int_{\mathcal{F}(g)} \Phi(g') K(g^{-1}g') \, d\mu_G(g').
\end{equation}
The implementation of the measure $\mu_G$ depends on specific properties of $\mathcal{S}$, $G$, and $\mathcal{F}$. Based on~\cref{thm:minimal_frame}, the minimal frame $\mathcal{F}$ constitutes a coset of the stabilizer. Therefore, if a measure on the stabilizer is defined, one may directly apply this measure to $\mathcal{F}$ rather than to the entire group. This approach effectively renders the size of the whole group irrelevant during averaging for practical purposes. Consequently, the process of minimal frame averaging adopts a form
\begin{equation}
    \langle \Phi \rangle_{\mathcal{F}}(x) = \int_{\mathcal{F}(x)} g \cdot \Phi(g^{-1} \cdot x) \, d\mu_{\mathcal{F}(x)}(g).
\end{equation}
The $G$-equivariance of the above form can be proved similar to~\cref{thm:fa}
such that given any $h\in G$,
\begin{equation}
\begin{split}
    \left<\Phi\right>_{\mathcal{F}}(h\cdot x) &= \int_{\mathcal{F}(h\cdot x)} g\cdot \Phi(g^{-1}h\cdot x)d\mu_{\mathcal F(h\cdot x)}(g) \\
    &= \int_{h\mathcal{F}(x)} g\cdot \Phi(g^{-1}h\cdot x)d\mu_{\mathcal{F}(x)}(g)\\
    &= \int_{\mathcal{F}(x)} hg\cdot \Phi(g^{-1}h^{-1}h\cdot x)d\mu_{\mathcal{F}(x)}(hg) \\
    &= h\cdot \int_{\mathcal{F}(x)} g\cdot \Phi(g^{-1}\cdot x)d\mu_{\mathcal F(x)}(g) \\
    &= h\cdot \left<\Phi\right>_{\mathcal{F}}(x).
\end{split}
\end{equation}

\colorblack
We next consider frame averaging over a Lie group, as this ensures that the frame averaging operator maps to a well-defined function. By definition of the manifold, all Lie groups are locally Euclidean and thus locally compact. Furthermore, for any locally compact group \( G \), there exists a unique (up to a multiplicative constant) left-invariant Radon measure \( \mu_G \), known as the Haar measure. This measure is defined such that \( \mu_G(gB) = \mu_G(B) \) for all \( g \in G \) and any Borel set \( B \subseteq G \). A canonical example is \( G = (\mathbb{R}^n, +) \), where \( \mu_G \) is the Lebesgue measure giving the volume of a \(n\)-dimensional set \(X\subset \mathbb{R}^n\), a translation-invariant quantity. Importantly, the Haar measure enables \( G \)-invariant group integration 
\[
\int_G f(g) d\mu_G(g).
\]

Commonly, the group integration is defined on a function \(f\in\mathcal{L}^\infty(G, \mathbb{R})\) with scalar output. However, more generally, the integration of a vector-valued function \( f \) in \( \mathcal{L}^\infty(\mathcal S, \mathcal{W}) \), where \( \mathcal{W} \) is a \( K \)-dimensional vector space, can be expressed via decomposition into \( K \) components, each corresponding to a dimension of \( \mathcal{W} \):
\[
(\int_G f^{1} \, d\mu_G, \, \int_G f^{2} \, d\mu_G, \, \cdots, \, \int_G f^{K} \, d\mu_G).
\]

If \(G\) is a Lie group, the domain of the $G$-set is usually a manifold where the Lie group $G$ acts smoothly. This consideration is essential for ensuring the averaging process respects the group action and to be well-defined and finite.

\begin{theorem}
    Consider a Lie group \(G\) with a measure space \(\left(G, \Sigma, \mu_G\right)\) where \(G\) acts upon \(\mathcal{M}\) smoothly with \(\mathcal{M}\) a smooth manifold. Given the definition of frame averaging in~\cref{thm:fa}, \(\langle\Phi\rangle_{\mathcal{F}}(x)\) is finite.
\end{theorem}

\begin{proof}
    Given that \(\Phi \in \mathcal{L}^\infty(\mathcal{M}, \mathcal{W})\), \(\sup_{x \in \mathcal{M}} \|\Phi(x)\|\) is bounded. Since \(\Phi\) is also continuous, for a compact Borel subset \(\mathcal{F}(x)\) of \(G\), the image set \(\{\Phi(g^{-1} \cdot x) \mid g \in \mathcal{F}(x)\}\) is compact, as the continuous image of a compact set is compact. Define \(F_x(g) = g \cdot \Phi(g^{-1} \cdot x)\). By Tychonoff's theorem, the set \(\{F_x(g) \mid g \in \mathcal{F}(x)\}\) is compact in \(\mathcal{W}\), since a product of any subsets of compact topological spaces is compact. Decomposing \(F_x(g)\) into its components \(F_x(g) = (F_x^1(g), F_x^2(g), \cdots, F_x^K(g))\), observe that each component \(F_x^i(g)\) is bounded. Hence,
    \begin{equation}
        \begin{split}
            \left| \int_{\mathcal{F}(x)} F_x^i(g) \, d\mu_G(g) \right| \leq \int_{\mathcal{F}(x)} |F_x^i(g)| \, d\mu_G(g) \leq \sup_{g \in \mathcal{F}(x)} |F_x^i(g)| \cdot \mu_G(\mathcal{F}(x)).
        \end{split}
    \end{equation}
    Since \(\mu_G\) is a finite \(G\)-invariant measure and \(\mathcal{F}(x)\) is compact, \(\mu_G(\mathcal{F}(x))\) is finite. Therefore, each integral \(\int_{\mathcal{F}(x)} F_x^i(g) \, d\mu_G(g)\) is finite, ensuring that the frame averaging \(\langle\Phi\rangle_{\mathcal{F}}(x)\) in \(\mathcal{W}\) is finite.
\end{proof}

The finiteness of \(\langle\Phi\rangle_{\mathcal{F}}(x)\) enables its computation. If \( \mathcal{F}(x) \) is a finite set, a counting measure is appropriate, and the frame averaging can still be computed as:
\begin{equation}
    \langle\Phi\rangle_{\mathcal{F}}(x) = \frac{1}{|\mathcal{F}(x)|} \sum_{g \in \mathcal{F}(x)} g \cdot \Phi(g^{-1} \cdot x).
\end{equation}
Conversely, if \(\mathcal{F}(x) \subseteq G\) is a compact subset of \(G\) with a manifold structure, the Haar measure is more applicable. In this case, the computation of frame averaging can be approximated by Monte Carlo integration:
\begin{equation}
    \int_{\mathcal{F}(x)} g \cdot \Phi(g^{-1} \cdot x) \, d\mu_{\mathcal{F}(x)}(g) \approx \frac{1}{N} \sum_{i=1}^N \tilde{g}_i \cdot \Phi(\tilde{g}_i^{-1} \cdot x),\quad \tilde{g}_i \sim \mathcal{F}(x),
\end{equation}
where \(\tilde{g}_i \sim \mathcal{F}(x)\) means that \(\tilde{g}_i\) is sampled uniformly from \(\mathcal{F}(x)\). Overall, the computation of frame averaging depends on the corresponding \(G\)-invariant measure over \(\mathcal{F}(x)\), and the choice of measure depends on the measurability of \( \mathcal{F}(x) \).

In studying \(G\)-equivariant frame averaging, arbitrary continuous functions \(f\), typically neural networks, are considered. These functions usually need computation multiple times relative to the cardinality or dimension of \(\mathcal{F}(x)\). For efficient and accurate computation of frame averaging, it is advantageous to minimize the cardinality or dimension of \(\mathcal{F}(x)\). Importantly, a \(G\)-set \(\mathcal{S}\) can be decomposed into a union of disjoint orbits under the action of \(G\), forming the quotient space \(\mathcal{S}/G\). To describe the size of \(\mathcal{F}(x)\), we investigate an important relationship among the group, the orbit, and the stabilizer.

\begin{theorem}[Orbit-Stabilizer Theorem]
\label{thm:ost}
Let \(G\) be a finite group acting on a \(G\)-set \(\mathcal{S}\), and let \(x \in \mathcal{S}\). The orbit of \(x\) under the action of \(G\) is denoted as \(\text{Orb}_G(x) = \{g \cdot x \mid g \in G\}\). The stabilizer of \(x\) in \(G\), denoted \(\text{Stab}_G(x)\), satisfies the following relation connecting the cardinalities of \(G\), the stabilizer, and the orbit:
\begin{equation}
    |G| = |\text{Stab}_G(x)| \cdot |\text{Orb}_G(x)|.
\end{equation}
\end{theorem}
Additionally, if the group is not finite, but instead a Lie group, the following characterization can be made:
\begin{theorem}[Orbit-Stabilizer Theorem for Lie Groups~\citep{tauvel2005homogeneous}]
    For a Lie group \(G\) acting on a smooth manifold \(\mathcal{M}\), let \(x \in \mathcal{M}\). Then the dimensions of \(G\), the stabilizer of \(x\) in \(G\), and the orbit of \(x\) satisfy:
    \begin{equation}
        \text{dim}(G) = \text{dim}(\text{Stab}_G(x)) + \text{dim}(\text{Orb}_G(x)).
    \end{equation}
\end{theorem}

By these theorems, a group $G$ can be decomposed into a stabilizer and its equal measure left cosets. In addition, according to~\cref{lem:minimal_frame} that there exists an $x_0 \in \text{Orb}_G(x)$ such that $\text{Stab}_G(x_0)\subseteq\mathcal{F}(x_0)$, for any \(x \in \text{Orb}_G(x)\), it holds that either \(|\mathcal{F}(x)| \ge |\text{Stab}_G(x)|\) in the case where $G$ is finite or \(\text{dim}(\mathcal{F}(x)) \ge \text{dim}(\text{Stab}_G(x))\) in the case where $G$ is a Lie group. Thus for efficiency, we can define a minimal frame and establish that there exists an \(x_0\) for which the quotient \(G/\text{Stab}_{G}(x_0)\) can be used to construct this minimal frame, which is exactly the result of \cref{thm:minimal_frame}. \colororange Furthermore, we desire that the minimality of a frame does not change regardless of the groups and the finite measure we have chosen. We define the minimal frame in \cref{def:min_frame_def} and the below proof shows its universality with respect to any finite measure. 

\todo{R3, Q3}
\begin{theorem}\label{cor:smallest_measure}
    Let $\mu$ be a $G$-invariant, $\sigma$-finite measure on $\Sigma$, the Borel $\sigma$-algebra on $G$. For the frames ${\mathcal F}$ and $\hat{\mathcal F}$ on $\mathcal S$, where ${\mathcal F}$ is arbitrary and $\hat{\mathcal F}$ is minimal, let $x\in\mathcal S$ such that $\mathcal F(x)$ and $\hat{\mathcal F}(x)$ are both $\mu$-measurable. Then \(\mu (\hat{\mathcal{F}}(x)) \le \mu (\mathcal{F}(x))\).
\end{theorem}

        \begin{proof}
        By \cref{lem:minimal_frame,thm:minimal_frame}, there exists $h_0,h_1\in G$ such that $\text{Stab}_{G}(h_0\cdot x) \subseteq \mathcal{F}(h_0\cdot x)$ and $\text{Stab}_G(h_1\cdot x) = \hat{\mathcal{F}}(h_1\cdot x)$. Define $\hat{h}\coloneq h_1h_0^{-1}$, and observe that for all $g\in \stab{G}{h_1\cdot x}$, $\hat{h}^{-1}g\hat{h}\in \stab G{h_0\cdot x}$. Therefore, $\hat{h}^{-1}\stab{G}{h_1\cdot x}\hat{h}\subseteq \stab G{h_0\cdot x}$, and thus, $\hat{h}^{-1}\stab{G}{h_1\cdot x}\hat{h}\subseteq \mathcal{F}(h_0\cdot x)$. From the $G$-invariance of $\mu$ and $G$-equivariance of $\hat{\mathcal F}$ and $\mathcal F$,
        \begin{align*}
            \mu(\hat{\mathcal F}(x)) & =\mu(\hat{\mathcal F}(h_1\cdot x)) =\mu(\hat{h}^{-1}\stab G{h_1\cdot x}\hat{h})
             \\
             & \le\mu(\mathcal F(h_0\cdot x))=\mu(\mathcal F(x)).
        \end{align*}
        \end{proof}

\colorblack

\section{Matrix Decomposition}
\label{sec:matrix_decomposition}

This section presents four matrix decompositions: QR decomposition, eigendecomposition, polar decomposition, and Jordan decomposition. These decompositions are employed to obtain canonical forms of a matrix relative to distinct groups. Each decomposition is analyzed for its unique properties and applications in transforming a given matrix into its canonical form.

\subsection{Generalized QR Decomposition}
\label{sec:qr_decomposition}
\begin{algorithm}
\caption{Generalized Gram-Schmidt Orthogonalization}
\label{alg:qr}
\begin{algorithmic}[1]
    \STATE {\bfseries Input:} matrix $\mA$, metric $\eta$
    \STATE $d, m \gets \mA\text{.shape}$
    \STATE $\mQ \gets \mathbf{0}$
    \STATE $\mD \gets\mathbf{0}$
    \STATE $\mR \gets\mathbf{0}$

    \FOR{$j = 1$ to $n$}
        \STATE $\vv \gets \mA_{:,j}$

        \FOR{$i = 1$ to $j-1$}
            \STATE $c \gets (\mA_{:,j}^T \eta \mQ_{:,i}) / (\mQ_{:,i}^T\eta \mQ_{:,i})$
            \STATE $\vv \gets \vv - c \mQ_{:,i}$
        \ENDFOR

        \STATE $\text{norm} \gets \vv^T \eta \vv$
        \IF{$\text{norm} = 0$}
            \STATE \textbf{raise} Exception(``Cannot normalize a null vector.'')
        \ENDIF
        \STATE $\mQ_{:,j} \gets \vv / \sqrt{\lvert \text{norm} \rvert}$
        \STATE $\mD_{j, j} \gets \text{sign}(\mQ_{:,j}^T\eta\mQ_{:,j})$
        \FOR{$i = 1$ to $j$}
            \STATE $\mR_{ij} \gets \mQ_{:,i}^T \eta \mA_{:,j}$
            \IF{$i = j$ and $\mR_{ij} < 0$}
                \STATE $\mQ_{:,i} \gets -\mQ_{:,i}$
                \STATE $\mR_{ij} \gets -\mR_{ij}$
            \ENDIF
        \ENDFOR
    \ENDFOR
    \STATE \textbf{return} $\mQ, \mD, \mR$
\end{algorithmic}
\end{algorithm}

QR decomposition is a decomposition of a matrix \(\mP\in\mathbb{R}^{d\times n}\) into a product \(\mP = \mQ\mR \) of an orthonormal matrix \(\mQ \in \mathrm{O}(d)\) and an upper triangular matrix \(\mR\). And the uniqueness and \(\mathrm{O}(d)\) of \(\mR\) is guaranteed by full column rank of \(\mP\). In this paper, we propose a generalized QR decomposition of \(\mP = \hat{\mQ}\hat{\mR}\) with full column rank, which is achieved by our generalized Gram-Schmidt orthogonalization~\cref{alg:qr}. Here, \(\hat{\mQ} \in G_\eta (d)\) is a \(d\) by \(d\) orthonormal matrix belonging to the linear algebraic group with equation \(\mO^T\eta\mO = \eta\), where \(\eta\) is a diagonal matrix with entires \(\pm 1\), and \(\hat{\mR}\) is a unique matrix serving as canonical form. As in our frame averaging application, \(\mP\) is derived from the induced \(G_\eta(d)\)-set and has full column rank, we have \(d\ge n\). We further divide it into \(d = n\) and \(d > n\) case. To introduce our method properly, we first recap the traditional QR decomposition in \(d = n\) case.

\paragraph{Case I. \(d = n\).} The traditional Gram-Schmidt process for full-rank matrix \(\mP\in \mathbb{R}^{d\times d}\) has a form of 
\begin{equation}
\label{eqn:gs}
\begin{aligned}
\vu_{1} &= \vv_{1}, & \ve_{1} &= \frac{\vu_{1}}{\|\vu_{1}\|} \\
\vu_{2} &= \vv_{2} - \text{proj}_{\vu_{1}}(\vv_{2}), & \ve_{2} &= \frac{\vu_{2}}{\|\vu_{2}\|} \\
\vu_{3} &= \vv_{3} - \text{proj}_{\vu_{1}}(\vv_{3}) - \text{proj}_{\vu_{2}}(\vv_{3}), & \ve_{3} &= \frac{\vu_{3}}{\|\vu_{3}\|} \\
& \vdots && \vdots \\
\vu_{d} &= \vv_{d} - \sum_{j=1}^{d-1}\text{proj}_{\vu_{j}}(\vv_{d}), & \ve_{d} &= \frac{\vu_{d}}{\|\vu_{d}\|}.
\end{aligned}
\end{equation}

where \(\text{proj}_{\vu}(\vv) = \frac{\langle \vu, \vv \rangle}{\langle \vu, \vu\rangle} \vu \) denoting the projection from \(\vv\) to \(\vu\) and we obtain \(\mQ = \left[\ve_1,\cdots,\ve_d\right]\). Conversely, each vector \(v_i\) of \(\mP\) can be written  as

\begin{equation}
    \vv_{i} =\sum _{j=1}^{i}\left\langle \ve_{j},\vv_{i}\right\rangle \ve_{j}
\end{equation}

Therefore, the element of \(\mR\) in QR decomposition can be represented as \(\mR_{ij} = \langle \ve_{i}, \vv_{j}\rangle\). One of the ambiguities of QR decomposition comes from the non-unique signs of \(\vv_i\), as by flipping the signs of \(\vv_i\), \(\mQ\) is still an orthogonal matrix. This can be solved by enforcing the sign of each diagonal element \(\mR_{ii}\) to be positive.

In our generalized QR decomposition, the inner product \(\langle \vu_i,\vu_i \rangle\) is not always positive, and this leads to an undefined result of \(\|\vu_{i}\|\). To remedy this situation, we define \(\|\vu_{i}\| = \sqrt{|\langle\vu_i,\vu_i\rangle|}\) and when constructing \(\mR\), we introduce \(\mD_\eta = \text{diag}\left(\langle\ve_1,\ve_1\rangle, \langle\ve_2,\ve_2\rangle,\cdots, \langle\ve_d,\ve_d\rangle \right)\) so that \(\mQ^T \eta \mQ = \mD_\eta\), and each vector \(\vv_i\) of \(\mP\) can be constructed by \(\mQ\mD_\eta\mR\) as
\begin{equation}
\begin{aligned}
    \vv_{i} &= \sum_{j=1}^{i} \langle\ve_j,\ve_j\rangle\left\langle \ve_{j},\vv_{i}\right\rangle \ve_{j} \\
    &= \sum_{j=1}^{i} \langle\ve_j,\ve_j\rangle \left\langle \frac{\vu_{j}}{\|\vu_{j}\|},\vv_{i}\right\rangle \frac{\vu_{j}}{\|\vu_{j}\|} \\
    &= \sum_{j=1}^{i} \frac{\left\langle \ve_{j},\ve_{j}\right\rangle}{|\left\langle \vu_{j},\vu_{j}\right\rangle|}\left\langle \vu_{j},\vv_{i}\right\rangle \vu_{j}\\
    &= \sum_{j=1}^{i} \frac{\left\langle \vu_{j},\vv_{i}\right\rangle}{\left\langle \vu_{j},\vu_{j}\right\rangle} \vu_{j},
\end{aligned}
\end{equation}
which is equivalent to~\cref{eqn:gs}, suggesting the validity of decomposition \(\mP = \mQ\mD_\eta\mR\). By enforcing the sign of each diagonal element \(\mR_{ii}\) to be positive, we can obtain a unique \(\mR\). Although \(\mQ\) has \(d\) orthonormal vectors with respect to metric \(\eta\), \(\mQ^T\eta\mQ\) is not necessarily equal to \(\eta\). For instance, for Gram-Schmidt orthogonalization with respect to the Minkowski metric \(\eta = \text{diag}(-1, 1, 1, 1)\), the first vector used for orthogonalization might be space-like (i.e., vectors with positive self inner products), which leads to the first diagonal element of \(\mQ^T\eta\mQ\) be \(1\) instead of \(-1\). On the other hand, since \(\mQ\) has \(d\) orthonormal vectors forming a complete basis for \(d\)-dimensional vector space with respect to metric \(\eta\), there must be the equal counts of \(\pm 1\) between \(\mD_\eta\) and \(\eta\). Consider a permutation \(\mS\) uniquely determined by~\cref{alg:permutation} that permutes columns of \(\mQ\) such that self inner products of \(\mQ\mS \) match signatures in \(\eta\), and subsequently \(\mQ\mS\in G_\eta(d)\), i.e., \(\mS^T\mQ^T\eta\mQ\mS = \eta\), and the QR decomposition can be written as \(\mP =\mQ\mS\mS^T \mD_\eta \mR\). Let \(\hat{\mQ} = \mQ\mS\) and \(\hat{\mR} = \mS^T \mD_\eta\mR\), we obtain the generalized QR decomposition \(\mP = \hat{\mQ}\hat{\mR}\).

\paragraph{Case II. \(d > n\).} As \(d>n\), there exist indeterminant \(d - n\) vectors in Gram-Schmidt orthogonalization. We define the generalized QR decomposition as 
\begin{equation}
    \mP = \hat{\mQ}\hat{\mR} =  \mQ\mS\mS^T\mD_\eta\mR = \mQ\mD_\eta\mR  = \begin{bmatrix}\widetilde{\mQ} & \bar{\mQ}\end{bmatrix} \begin{bmatrix}\widetilde{\mD}_\eta & \\ & \bar{\mD}_\eta\end{bmatrix}  \begin{bmatrix}\widetilde{\mR} \\ \mathbf{0} \end{bmatrix} = \widetilde{\mQ}\widetilde{\mD}_\eta \widetilde{\mR},
\end{equation}
where \(\widetilde{\mQ}\) is a \(d\times n\) matrix and \(\bar{\mQ}\) is a \(d\times (d-n)\) matrix both with all column orthonormal with respect to \(\eta\), \(\widetilde{\mD}_\eta\) and \(\bar{\mD}_\eta\) are diagonal matrices with entries of the self inner product of \(\widetilde{\mQ}\) and \(\bar{\mQ}\), respectively, and \(\widetilde{\mR}\) is a \(n\times n\) full-rank upper triangular matrix. The creation of \(\widetilde{\mQ}\) and \(\widetilde{\mR}\) is the same as \(d = n\) case, and \(\bar{\mQ}\) is produced by randomly choosing \(d - n\) vectors that are orthonormal to each other and to \(\widetilde{\mQ}\), and \(\mS\) is a permutation matrix that permutes columns of \(\mQ\) such that \(\mQ\mS\in G_\eta(d)\). Conceptually, \(\mS\) can be represented by two permutation matrices \(\widetilde{S}\) and \(\mS'\) such that \(\mS = \widetilde{\mS}\mS'\) where \(\widetilde{S}\) is produced by~\cref{alg:permutation} permuting the determinate column vectors in \(\mQ\) and \(\mS'\) is a non-unique permutation matrix subsequently permuting those non-unique column vectors of \(\mQ\).

To show that generalized QR decomposition can produce a canonical form of the linear algebraic group \(G_\eta(d)\), we first show the uniqueness of \(\mQ\) and \(\mR\) when \(d = n\) and uniqueness of \(\widetilde{\mQ}\) and \(\widetilde{\mR}\) when \(d > n\) by the below theorem. As \(\mS\mS^T = \mI_n\) and the below theorem is not relevant to \(\mS\), we omit the \(\mS\) here and discuss it in the next theorem.

\begin{theorem}[Uniqueness of QR Decomposition]
\label{thm:qr_unique}
Let \(\mP\in \mathbb{R}^{d\times n}, d\ge n\) equipped with a diagonal matrix \(\eta\) with entries \(\pm 1\) and all \(\mP\)'s columns non-null, and let the generalized QR decomposition \(\mP = \mQ\mD_\eta\mR\) where \(\mQ\) contains \(d\) orthonormal vectors with respect to inner product metric \(\eta\), \(\mD_\eta\) is a diagonal matrix containing self inner products of column vectors in \(\mQ\), and \(\mR\) is an upper triangular matrix and all diagonal values of \(\mR\) are non-negative. The \(\mR\) is unique if \(\mP\) is full-rank. Additionally, if \(d = n\) and \(\mP\) is a full-rank, then both \(\mQ\) and \(\mR\) are unique.
\end{theorem}
\begin{proof}
We first establish the uniqueness of the generalized QR decomposition for a full-rank square matrix \(\mP\) that does not contain any null vectors. And we aim to prove that both \(\mQ\) and \(\mR\) in the decomposition \(\mP = \mQ\mD_\eta\mR\) are unique. Given the inner product property in \(\eta\), if \(\langle\vu, \vv\rangle = 0\), then \(\langle\vu, -\vv\rangle = 0\). This indicates that the only freedom in the Gram-Schmidt process is the sign of each produced orthonormal vector. Suppose there exist alternate matrices \(\mQ',\mD_\eta',\mR'\) such that 
\begin{equation}
\label{eqn:d}
    \mQ\mD_\eta\mR = \mQ'\mD_\eta'\mR'.
\end{equation}
Given \(\mQ^T\eta\mQ = \mD_\eta\) and \(\mQ'^T\eta\mQ' = \mD_\eta'\), it follows that \(\mQ^{-1} = \mD_\eta\mQ^T\eta\) and \(\mQ'^{-1} = \mD_\eta'\mQ'^T\eta\). Rearranging, we obtain
\begin{equation}
    \mQ'^T \eta \mQ = \mR'\mR^{-1}\mD_\eta
\end{equation}
and
\begin{equation}
    \mQ^T \eta \mQ' = \mR\mR'^{-1}\mD_\eta',
\end{equation}
where both right-hand sides are upper triangular matrices. Taking transposes, we find
\begin{equation}
    \mQ^T \eta \mQ' = (\mR'\mR^{-1}\mD_\eta)^T
\end{equation}
and
\begin{equation}
    \mQ'^T \eta \mQ = (\mR\mR'^{-1}\mD_\eta')^T,
\end{equation}
where both right-hand sides are lower triangular matrices. Thus, \(\mQ^T \eta \mQ'\) and \(\mQ'^T \eta \mQ\) are diagonal matrices. Letting \(\mD = \mR'\mR^{-1}\), we derive
\begin{equation}
    \mQ\mD_\eta = \mQ'\mD_\eta'\mD
\end{equation}
and subsequently,
\begin{equation}
    \mQ = \mQ'\mD_\eta'\mD\mD_\eta.
\end{equation}
Since diagonal matrices commute under multiplication, it follows that
\begin{align}
    \mQ^T\eta\mQ &=\mD_\eta \mD \mD_\eta' \mQ'^T \eta \mQ'\mD_\eta'\mD\mD_\eta \\
    &= \mD_\eta \mD \mD_\eta' \mD_\eta'\mD_\eta'\mD\mD_\eta \\
    &= \mD^2 \mD_\eta' \\
    &= \mD_\eta,
\end{align}
implying \(\mD^2 = \mD_\eta\mD_\eta'\). Since \(\mD^2\) has positive diagonal values and \(\mD_\eta\), \(\mD_\eta'\) have entries of \(\pm 1\), we get
\begin{equation}
    \mD^2 = \mD_\eta\mD_\eta' = \mI_d.
\end{equation}
Thus, \(\mD_\eta = \mD_\eta'\) and, as \(\mD = \mR'\mR^{-1}\), it follows that \(\mD\mR = \mR'\). Requiring positive diagonals for \(\mR\) and \(\mR'\) leads to \(\mD = \mI_d\) and \(\mR = \mR'\).~\cref{eqn:d} then implies \(\mQ' = \mQ\), establishing the uniqueness of \(\mQ\) and \(\mR\) in the decomposition for a non-null full-rank square matrix.

For the case where \(\mP\) is not square but has non-null, linearly independent columns, let \(d > n\) and \(\text{rank}(\mP) = n\). The generalized QR decomposition is
\begin{equation}
    \mP =  \mQ\mD_\eta\mR = \begin{bmatrix}\widetilde{\mQ} & \bar{\mQ}\end{bmatrix} \begin{bmatrix}\widetilde{\mD}_\eta & \\ & \bar{\mD}_\eta\end{bmatrix}  \begin{bmatrix}\widetilde{\mR} \\ \mathbf{0} \end{bmatrix} = \widetilde{\mQ}\widetilde{\mD}_\eta \widetilde{\mR},
\end{equation}
where \(\widetilde{\mQ}\) is a \(d\times n\) orthonormal matrix with respect to \(\eta\), \(\widetilde{\mD}_\eta\) is a \(n\times n\) matrix representing self inner products in \(\widetilde{\mQ}\), and \(\widetilde{\mR}\) is a \(n\times n\) full-rank upper triangular matrix. Assume there exists \(\widetilde{\mQ}', \widetilde{\mD}_\eta', \widetilde{\mR}'\) such that
\begin{equation}
\label{eqn:d_rank_deficient}
    \mP = \widetilde{\mQ}\widetilde{\mD}_\eta \widetilde{\mR} = \widetilde{\mQ}'\widetilde{\mD}_\eta' \widetilde{\mR}'.
\end{equation}
Given \(\widetilde{\mQ}^T\eta\widetilde{\mQ} = \widetilde{\mD}_\eta\) and \(\widetilde{\mQ}'^T\eta\widetilde{\mQ}' = \widetilde{\mD}_\eta'\), analogous manipulations yield 
\begin{equation}
    \widetilde{\mQ}'^T\eta \widetilde{\mQ} = \widetilde{\mR}'\widetilde{\mR}^{-1}\widetilde{\mD}_\eta
\end{equation}
and
\begin{equation}
    \widetilde{\mQ}^T\eta \widetilde{\mQ}' = \widetilde{\mR}\widetilde{\mR}'^{-1}\widetilde{\mD}_\eta'.
\end{equation}
Setting \(\widetilde{\mD} = \widetilde{\mR}'\widetilde{\mR}^{-1}\) and following the steps for \(d = n\), we deduce that \(\widetilde{\mD} = \mI_n\), \(\widetilde{\mR}' = \widetilde{\mR}\), and \(\widetilde{\mD}_\eta = \widetilde{\mD}_\eta'\).~\cref{eqn:d_rank_deficient} then implies \(\widetilde{\mQ} = \widetilde{\mQ}'\), confirming the uniqueness of \(\mR = \begin{bmatrix}\widetilde{\mR} \\ \mathbf{0} \end{bmatrix}\) and the \(\widetilde{\mQ}\) part of \(\mQ = \begin{bmatrix}\widetilde{\mQ} & \bar{\mQ} \end{bmatrix}\).

In conclusion, the uniqueness of the generalized QR decomposition is established for both \(d = n\) and \(d > n\) cases, completing the proof.

\end{proof}

\begin{algorithm}
\caption{Generate Permutation \(\mS\)}
\label{alg:permutation}
\begin{algorithmic}
\STATE \(\mS \gets \mI_d\)
\IF{$d = n$}
\STATE \(\va \gets \) diagonal of \(\mD_\eta\)
\ELSE
\STATE \(\va \gets \) diagonal of \(\widetilde{\mD}_\eta\)
\ENDIF
\STATE \(\vb \gets\) diagonal of \(\eta\)
\STATE Right padding \(\va\) with \(0\) to the same length as \(\vb\)
\FOR{\(i = 1\) \textbf{to} \( d \)}
    \IF{\(\va_i \neq 0\) \AND \(\va_i \ne \vb_i\)}
        \FOR{\(j = i + 1\) \textbf{to} \( d \)}
            \IF{\(\va_j \ne \vb_j\) \AND \(\va_j = \vb_i\)}
                \STATE \(\mS' \gets \mI_d\)
                \STATE \( \mS_{i,i}' \gets 0 \)
                \STATE \( \mS_{j,j}' \gets 0 \)
                \STATE \( \mS_{i,j}' \gets 1 \)
                \STATE \( \mS_{j,i}' \gets 1 \)
                \STATE \( \va \gets \va \mS' \)
                \STATE \( \mS \gets \mS \mS' \)
                \STATE \textbf{break}
            \ENDIF
        \ENDFOR
    \ENDIF
\ENDFOR
\STATE \textbf{return} $\mS^T$
\end{algorithmic}
\end{algorithm}

Then we focus on the \(G_\eta\)-invariance of the QR decomposition incorporating the permutation \(\mS\). Given the generalized QR decomposition \(\mP =\mQ\mD_\eta\mR = \mQ\mS\mS^T \mD_\eta \mR\), we decompose the permutation \(\mS = \widetilde{\mS}\mS'\) where the permutation \(\widetilde{\mS}\) permutes those determinate orthonormal vectors in \(\mQ\) in a deterministic way by~\cref{alg:permutation} and \(\mS'\) permutes those non-unique orthonormal vectors in \(\mQ\) (and if there does not exist non-unique orthonormal vectors then \(\mS' = \mI_n\).)

\begin{theorem}[\(G_\eta\)-Invariance of QR Decomposition]
\label{thm:qr_inv}
Let \(\mP\in \mathbb{R}^{d\times n}\) equipped with a diagonal matrix \(\eta\) with entries \(\pm 1\) and all \(\mP\)'s columns non-null, and let generalized QR decomposition \(\mP = \hat{\mQ}\hat{\mR} = \mQ\mS\mS^T\mD_\eta\mR\) defined above, then \(\hat{\mR}=\mS^T\mD_\eta\mR\) is \(G_\eta\)-invariant if all columns in \(\mP\) are linearly independent.
\end{theorem}
\begin{proof}

We consider two cases based on the dimensions of \(\mP \in \mathbb{R}^{d\times n}, d\ge n\).

\textbf{Case 1: \(d = n\).} In the generalized QR decomposition \(\mP = \mQ\mS\mS^T\mD_\eta\mR = \mQ\mD_\eta\mR\), for any \(\mQ_\eta \in G_\eta\), which preserves the inner product with respect to \(\eta\), the matrix \(\mQ_\eta \mQ\) comprises \(d\) orthonormal vectors. The decomposition \((\mQ_\eta\mQ)\mD_\eta \mR\) serves as a generalized QR decomposition of \(\mQ_\eta\mP\). Due to the uniqueness of this decomposition, \(\mQ_\eta\mP\) and \(\mP\) share the same \(\mR\) matrix, indicating that \(\mR\) is invariant under \(G_\eta\). Moreover, as \(\mP = \mQ\mS\mS^T\mD_\eta\mR\), it follows that \(\mQ_\eta \mP = \mQ_\eta\mQ\mS\mS^T\mD_\eta\mR\). Given that \(\mS\) is determined by the self inner product of \(\mQ\) and \(\mQ_\eta\) preserves this inner product, \(\mQ_\eta\mQ\) yields the same \(\mD_\eta\) and, consequently, the same permutation matrix \(\mS\). Thus, both \(\mD_\eta\) and \(\mS\) are invariant under \(G_\eta(d)\), leading to the invariance of \(\mS^T\mD_\eta\mR\) to \(G_\eta(d)\).

\textbf{Case 2: \(d > n\).} For the generalized QR decomposition \(\mP = \widetilde{\mQ}\widetilde{\mD}_\eta\widetilde{\mR}\), a similar argument applies. With \(\mQ_\eta \in G_\eta(d)\), the decomposition \((\mQ_\eta\widetilde{\mQ})\widetilde{\mD}_\eta\widetilde{\mR}\) is the generalized QR decomposition of \(\mQ_\eta \mP\), affirming the invariance of \(\widetilde{\mQ}_\eta\) and \(\widetilde{\mR}\) under \(G_\eta(d)\). As mentioned, the permutation matrix \(\mS\) can be decomposed into \(\widetilde{\mS}\mS'\). \(\widetilde{\mS}\) by~\cref{alg:permutation} depends on the self inner product of the determinate orthonormal vectors in \(\mQ\), rendering \(\widetilde{\mS}\) invariant under \(G_\eta(d)\). The matrix \(\mS'\) adjusts the non-unique orthonormal column vectors, with \(\mS^T\mD_\eta \mR = \mS'^{T}\widetilde{\mS}\mD_\eta \mR\) showing that permutations by \(\mS'\) on the non-unique diagonal elements of \(\mD_\eta\), which is eventually applied on the zero rows of \(\mR\) do not alter the outcome of \(\mS^T \mD_\eta \mR\), confirming \(\mS^T \mD_\eta \mR\) is \(G_\eta(d)\)-invariance.

These considerations for both \(d = n\) and \(d > n\) complete the proof, establishing the \(G_\eta\)-invariance of \(\hat{\mR} = \mS^T \mD_\eta \mR\).

\end{proof}

In addition, \(\hat{\mQ} = \mQ\mS\) is an element of \(G_\eta (d)\) and so \(\hat{R} = \mS^T \mD_\eta \mR\) is inside the orbit of \(\mP\), thereby serving as the canonical form. In conclusion, \(\hat{R} = \mS^T\mD_\eta\mR\) is a canonical form of \(\mP\) with respect to \(G_\eta(d)\) by the generalized QR decomposition.

\subsection{Eigendecomposition}

Eigendecomposition decomposes a matrix \( \mP \in \mathrm{Sym}(d, \mathbb{C})\)—the set of \(d \times d\) complex symmetric matrices—into a product of its eigenvalues and eigenvectors, denoted as \(\mP = \mO^* \mLambda \mO\). Here, \(\mO \in \mathrm{U}(d)\) contains eigenvectors, and \(\mLambda\) is a diagonal matrix of eigenvalues. The uniqueness of \(\mLambda\) is up to permutation, rooted in the property \(\mP \vv = \lambda\vv\), where \(\vv\) and \(\lambda\) are corresponding eigenvectors and eigenvalues. This leads to the characteristic polynomial \(p(\lambda)=\det(\mP -\lambda \mI_d) = 0\), yielding \(d\) roots. For any \(\mO'\) in \(\mathrm{U}(d)\) or \(\mathrm{SU}(d)\), the eigendecomposition of \(\mO'^*\mP\mO'\) is \(\mO'^*\mO^*\mLambda\mO\mO'\), with \(\mLambda\) remaining unique and invariant under \(\mathrm{U}(d)\) and \(\mathrm{SU}(d)\) transformations. Similarly, for \(\mP\in \mathrm{Sym}(d, \mathbb{R})\), \(\mLambda\) is invariant under \(\mathrm{O}(d)\) and \(\mathrm{SO}(d)\). Thus, \(\mLambda\) serves as a canonical form of \(\mP\) with respect to these groups through eigendecomposition.

\subsection{Polar Decomposition}

For a full-rank square matrix \(\mP\in \mathbb{C}^{d\times d}\), the polar decomposition is \(\mP = \mU\mH\), where \(\mU\) is a unitary (or orthogonal for real matrices) matrix and \(\mH\) is a positive semi-definite Hermitian (or symmetric for real matrices) matrix. This decomposition splits \(\mP\) into a rotation/reflection component, \(\mU\), and a scaling component, \(\mH\). The polar decomposition always exists and is unique when \(\mP\) is invertible. In such cases, \(\mP\) can also be expressed as \(\mP = \mU e^{\mX}\), where \(\mX\) is the unique self-adjoint logarithm of \(\mP\). Given \(\mU' \in \mathrm{U}(d)\) or \(\mU'\in\mathrm{O}(d)\), the polar decomposition of \(\mU'\mP\) is \(\mU'\mU\mH\), with \(\mH\) remaining unique and invariant. Hence, \(\mH\) represents the canonical form of \(\mP\) with respect to \(\mathrm{U}(d)\) or \(\mathrm{O}(d)\) under polar decomposition.

\subsection{Jordan Decomposition}

The Jordan canonical form is the canonical form for a square matrix \(\mP \in \mathbb{C}^{n \times n}\) under \(\mathrm{GL}(n, \mathbb{C})\). It represents \(\mP\) as a block diagonal matrix \(\mJ = \begin{bmatrix} \mJ_1 & & \\ & \ddots & \\ & & \mJ_p \end{bmatrix}\), where each \(\mJ_i\) is a Jordan block associated with a general eigenvalue \(\lambda_i\). An invertible matrix \(\mU \in \mathrm{GL}(n, \mathbb{C})\) exists such that \(\mP = \mU \mJ \mU^{-1}\), with the uniqueness of \(\mJ\) up to the order and size of Jordan blocks~\citep{horn2012matrix}. In diagonalizable cases where \(\mP = \mU\mJ\mU^{-1}\) with \(\mU\) unitary, \(\mJ\) aligns with the complex eigendecomposition, maintaining its canonical status under the unitary group \(\mathrm{U}(n)\). Thus, the Jordan decomposition offers a generalized approach for canonical forms of diagonalizable square matrices. As the Jordan canonical form can serve as the canonical form of groups \(\mathrm{GL}(n,\mathbb{C})\), it can also serve as the canonical form of its subgroup

\section{Canonical Labeling and McKay's Algorithm}
\label{sec:mckay}

An \textit{isomorphism} between two graphs is defined as a bijection between their vertex sets that preserves the adjacency relation. Specifically, two vertices \(u\) and \(v\) are adjacent in one graph if and only if their images under the bijection are adjacent in the other graph. An \textit{automorphism} of a graph is an isomorphism from the graph to itself, capturing the concept of symmetry within the graph structure. The collection of all automorphisms of a graph constitutes a group under the operation of composition, known as the \textit{automorphism group} of the graph. In applications where vertices of a graph are distinguished by certain attributes, such as color or weight, the definition of an automorphism accommodates these distinctions. Specifically, an automorphism must map vertices to vertices of the same attribute value. This consideration ensures that the automorphism respects the additional structure imposed on the graph by these attributes.

Formally, consider two real symmetric matrices \(\mA, \mB \in \mathrm{Sym}(n, \mathbb{R})\) representing the adjacency matrices of two graphs. The matrix \(\mA\) is said to be \textit{isomorphic} to \(\mB\) if and only if there exists a permutation matrix \(\mS \in \mathrm{S}_n\) satisfying
\begin{equation}
    \mS^T \mA \mS = \mB.
\end{equation}
Furthermore, the \textit{automorphism group} (or stabilizer) of a graph represented by \(\mA\) consists of all permutation matrix \(\mS\in\mathrm{S}_n\) for which
\begin{equation}
    \mS^T\mA \mS = \mA.
\end{equation}

\subsection{Individualization-Refinement Paradigm}

\citet{mckay2014practical} gives a general overview of graph isomorphism as well as the refinement-individualization paradigm to solve canonical labeling and graph automorphism. Canonical labeling is a canonicalization process in which a graph is relabeled in such a way that isomorphic graphs are identical after relabeling. By denoting this process as \(c:\mathrm{Sym}(n,\mathbb{R}) \rightarrow \mathrm{Sym}(n,\mathbb{R})\), we obtain the canonical form \(c(\mA)\) via the canonical labeling and \(\mA\) and \(\mB\) is isomorphic if and only if \(c(\mA) = c(\mB)\). 

To be concrete,~\citet{mckay2014practical} provides the canonical labeling algorithm serving to categorize a collection of objects into isomorphism classes through a search tree. An important concept within this method is the \textit{equitable partition}, wherein vertices sharing the same color have an identical count of adjacent vertices for each distinct color. The process of partition refinement iteratively subdivides its cells to achieve greater granularity. Given any initial partition, there exists a uniquely determined equitable partition that refines it with the minimal possible color count.

On this basis, the construction of a search tree is based on \textit{partition refinement}, and the nodes of the tree denote the partitions themselves. Nodes that correspond to discrete partitions, wherein each vertex is uniquely colored, constitute the leaves of the tree. For nodes that are not leaves, a color that is replicated across vertices is selected; one vertex is then individualized by applying it with a new color, and subsequent refinement to equitable partition yields a child node. Identical labeling at different leaves shows the presence of an automorphism. By defining an ordering scheme for labeled graphs (for example, lexicographical order), the greatest graph at a leaf is treated as the canonical form. However, the vast potential size of the search tree requires the implementation of pruning strategies through the node invariants. Node invariants are attributes invariant to graph labeling, for example, node degree, which can be used to further subdivide and refine the partition. 

\subsection{Employing Canonical Labeling}
\label{sec:employ_canonical_labeling}

For (undirected) graph data, the canonical labeling can be directly applied to compute the canonical graph as well as the automorphism group. For point cloud data, by our induced $\Sn$-set, we first compute the inner product matrix of the point cloud as an undirected weighted adjacency matrix. As the original nauty algorithm only works for node-colored graphs, we convert the edge-weighted graph to the unweighted counterpart by replacing each weighted edge with a colored node, where each color corresponds a certain magnitude of the inner product between two points. Then we employ the nauty algorithm we present above to obtain the canonical form and the automorphism group of the unweighted graph. Subsequently, we convert the newly generated nodes back to the weighted edges and remove the permutation of those newly generated nodes in each element of the automorphism group, then we obtain the canonical form and automorphism group of the original weighted adjacency matrix. As a result, we are able to create the frame of \(\mathrm{S}_n\) for point clouds.

In our implementation, we replicate one of the most famous graph automorphism algorithms nauty (no automoprhism, yes) by pure Python. The nauty algorithm is an implementation of the above individualization-refinement paradigm by depth-first search. To speed up the algorithm, we use just-in-time (JIT) compiler Numba to compile and optimize the code.

\section{Point Group and Minimal Frame of $\Sn \times G_\eta(d)$}
\label{sec:point_group}

A \textit{point group} is a set of symmetry operations \(g\in \mathrm{O}(d)\) mapping a set of a point cloud onto itself. \blue{In this section, we extend the definition to $G_\eta(d)$ defined in~\cref{sec:linAlgG}.} Mathematically, for a given \(d\)-rank point cloud \(\mP\in \mathbb{R}^{d \times n}\) without null and repetitive columns, the point group \(\mathcal{G}\) can be defined as:

\[
\mathcal{G} = \{\mO \in G_\eta(d) \,|\, \mO\mP = \mP \mS, \, \exists \mS \in \mathrm{S}_n\}.
\]

\colororange
Consider the following theorem:
\begin{theorem}
\label{thm:point_group_invariant}
    Given $\mP\in \mathbb{R}^{d\times n}$ with $\text{rank}(\mP) = d$ and without null and repetitive columns, and given the stabilizer \(\text{Stab}_{\mathrm{S}_n}(\mP^T\eta\mP) = \{\mS \mid \mS^T\mP^T\eta\mP\mS = \mP^T\eta\mP, \mS \in \mathrm{S}_n\}\), there exists \(\mS \in \text{Stab}_{\mathrm{S}_n}(\mP^T\eta\mP)\) such that \(\mS \neq \mI_n\) if and only if there exists an orthogonal transformation \(\mO \in G_\eta(d)\) with \(\mO \neq \mI_d\) such that \(\mO\mP = \mP\mS\).
\end{theorem}

\begin{proof}
    Let \(\mP = [\vv_1,\vv_2,\cdots, \vv_n]\) and \(\mP\mS = [\vw_1, \vw_2, \cdots, \vw_n]\) and recall that for any index $i=1,2,\cdots,n$, $\vv_i$ and $\vw_i$ are both non-null, and collectively, both $\{\vv_i\}_{i=1}^n$ and $\{\vw_i\}_{i=1}^n$ span the entire $d$-dimensional indefinite inner product space. Since \(\mS^T\mP^T\eta\mP\mS = (\mP\mS)^T\eta(\mP\mS) = \mP^T\eta\mP\), it follows that for all indices \(k\) and \(l\), \(\langle \vv_k, \vv_l \rangle = \langle \vw_k, \vw_l \rangle\), \textit{i.e.}, the same Gram matrix. Considering these as two different indefinite inner product spaces preserving the same inner product, there must exist a linear isometry \(\mO \in G_\eta(d)\) that transforms between these two spaces such that \(\mO\vv_k = \vw_k\). Since $\mS \ne \mI_n$, $\mO \ne \mI_d$. Therefore, we have \(\mO\mP = \mP\mS\) with $\mO\ne \mI_d$.
    
    Conversely, consider \(\mO\mP = \mP\mS\) and \(\mO \ne \mI_d\). Since \(\mP\) is \(d\)-rank and \(\mO\ne \mI_d\), it is impossible \(\mO\mP = \mP\). Therefore, \(\mS\ne \mI_n\) and we have \(\mS^T \mP^T \eta \mP\mS = \mP^T\mO^T \eta \mO \mP = \mP^T\eta\mP\), showing that \(\mS \in \text{Stab}_{\mathrm{S}_n}(\mP^T\eta\mP)\).
\end{proof}

Consequently, a non-trivial stabilizer \(\text{Stab}_{\mathrm{S}_n}(c(\mP^T\eta\mP))\) reveals a non-trivial point group with respect to the point cloud structure \(\mP\). Since our frame averaging \(\langle\Phi\rangle_{\mathcal F_\phi}\) in~\cref{sec:pcloud} is \(\mathrm{S}_n\times G_\eta(d)\)-invariant, for any \(\mO \in \mathcal{G}\), \(\langle\Phi\rangle_{\mathcal F_\phi}(\mO\mP) = \langle\Phi\rangle_{\mathcal F_\phi} (\mP\mS) = \langle\Phi\rangle_{\mathcal F_\phi}(\mP)\). Therefore, \(\langle\Phi\rangle_{\mathcal F_\phi}\) is \(\mathcal{G}\)-invariant.

Now we derive the $\Sn\times G_\eta(d)$ frame averaging. By~\cref{thm:prodG}, given a $G_\eta (d)$-invariant and $\Sn$-equivariant frame $\mathcal{F}$ and a $G_\eta(d)$-equivariant function $\Phi:\mathbb{R}^{d\times n}\rightarrow\mathbb{R}^{d\times n}$, we can obtain the $\Sn\times G_\eta(d)$-equivariant frame averaging
\begin{equation}
    \langle\Phi \rangle_{\mathcal{F}}(\mP) =\frac{1}{|\mathcal{F}(\mP)|} \sum_{g \in \mathcal{F}(\mP)} g\cdot\Phi(g^{-1}\cdot\mP) =\frac{1}{|\mathcal{F}(\mP)|} \sum_{\rho^{-1}(\mS) \in \mathcal{F}(\mP)} \Phi(\mP\mS^T) \mS.
\end{equation}
In fact, $\Phi$ can be constructed by the frame averaging presented in~\cref{sec:linAlgG} and the induced $G_\eta(d)$-set is built on each $\mP\mS^T$. By~\cref{sec:pcloud}, let the canonical form of the $G_\eta (d)$-invariant and $\Sn$-equivariant induced set be $c(\mP^T\eta\mP) = \mS_0^T\mP^T \eta \mP\mS_0$, where $\mS_0\in\Sn$ transforms $\mP^T\eta\mP$ into its canonical form. Then the $\Sn\times G_\eta(d)$ frame can be defined as $\mathcal{F} (\mP) = \{(\mS, \mO) \mid \mS\in \rho^{-1}(\mS_0^T)\text{Stab}_{\Sn}(\mS_0^T\mP^T \eta \mP\mS_0), \mO\in \hat{\mathcal{F}}_{\phi}(\phi(\mP\mS^T)) \}$, where $\hat{\mathcal{F}}_{\phi}$ is an induced $G_\eta(d)$-set in~\cref{sec:mfa_Geta}. Continuing from~\cref{thm:point_group_invariant}, we can prove this \(\mathrm{S}_n\times G_\eta(d)\) frame is a minimal frame.

\begin{theorem}\label{thm:minProd}
    Given $\mP\in \mathbb{R}^{d\times n}$ with $\text{rank}(\mP) = d$ and without null and repetitive columns. Let the \(\mathrm{S}_n\times G_\eta(d)\) frame defined by $\mathcal{F} (\mP) = \{(\mS, \mO) \mid \mS\in \rho^{-1}(\mS_0^T)\text{Stab}_{\Sn}(\mS_0^T\mP^T \eta \mP\mS_0), \mO\in \hat{\mathcal{F}}_{\phi}(\phi(\mP\mS^T)) \}$, where $\hat{\mathcal{F}}_{\phi}$ is a $G_\eta$ minimal frame on an induced $G_\eta(d)$-set induced by $\phi$ defined in~\cref{sec:mfa_Geta}. Then $\mathcal{F}$ is a \(\mathrm{S}_n\times G_\eta\) minimal frame on $\mathbb{R}^{d\times n}$.
\end{theorem}

\begin{proof}
    Let $\mP_0 = \mP\mS_0$ and we obtain $\mathcal{F} (\mP_0) = \{(\mS, \mO) \mid \mS\in \text{Stab}_{\Sn}(\mP^T_0 \eta \mP_0), \mO\in \hat{\mathcal{F}}_{\phi}(\phi(\mP_0\mS^T)) \}$. Given $\text{rank}(\mP_0) = d$ and any $(\mS_i, \mO_i),(\mS_j,\mO_j) \in  \mathcal{F}(\mP_0)$, by~\cref{thm:point_group_invariant}, there exists $\hat{\mO}_i,\hat{\mO}_j\in G_\eta(d)$ such that $\mP_0\mS_i^T = \hat{\mO}_i \mP_0$ and $\mP_0\mS_j^T = \hat{\mO}_j \mP_0$. Subsequently, both $\mP_0\mS_i^T \in \text{Orb}_{G_\eta(d)}(\mP_0)$ and $\mP_0\mS_j^T \in \text{Orb}_{G_\eta(d)}(\mP_0)$. In addition, since $\text{rank}(\mP_0) = d$, both $\phi(\mP_0\mS_i^T)$ and $\phi(\mP_0\mS_j^T)$ are square matrices. By the uniqueness of the QR decomposition shown in~\cref{sec:qr_decomposition}, $\phi(\mP_0\mS_i^T) = \hat{\mQ}_i\hat{\mR}_i$ and $\phi(\mP_0\mS_j^T) = \hat{\mQ}_j \hat{\mR}_j$ have unique $\hat{\mQ}_i$ and $\hat{\mQ}_j$ and the same canonical form $\hat{\mR}_i = \hat{\mR}_j$. By~\cref{lemma:multiCanMaps}, we obtain $\mO_i = \hat{\mQ}_i$ and $\mO_j = \hat{\mQ}_j$. Since $\mP_0\mS_i^T$ and $\mP_0\mS_j^T$ are in the same orbit, there exists $\mO'\in G_\eta (d)$ such that $\mO'\mO_i^{-1}\mP_0\mS_i^T = \mO_j^{-1}\mP_0\mS_j^T$. Applying $\phi$ on both sides of the equation we obtain $\phi(\mO'\mO_i^{-1}\mP_0\mS_i^T) = \mO' \mO_i^{-1}\phi(\mP_0\mS_i^T) = \mO'\hat{\mR}_i = \phi(\mO_j^{-1}\mP_0\mS_j^T) = \mO_j^{-1}\phi(\mP_0\mS_j^T) = \hat{\mR}_j $. Since $\hat{\mR}_i = \hat{\mR}_j$ and $\hat{\mR}_i,\hat{\mR}_j$ are full-rank square matrices, $\mO' = \mI_d$. Therefore, we obtain $\mO_i^{-1}\mP_0\mS_i^T = \mO_j^{-1}\mP_0\mS_j^T$. Let $\mP_0'= \mO_i^{-1}\mP_0\mS_i^T = \mO_j^{-1}\mP_0\mS_j^T$. Then $\mathcal{F} (\mP_0') = \mathcal{F} (\mO_j^{-1}\mP_0\mS_j^T) = (\mS_j^T,\mO_j^{-1}) \mathcal{F} (\mP_0)$. Now, for any $(\mS_i, \mO_i) \in  \mathcal{F}(\mP_0)$, every element of $\mathcal{F} (\mP_0')$ has a form of $(\mS_i\mS_j^T, \mO_j^{-1}\mO_i)$ such that
    \begin{equation}
        \mO_j^{-1}\mO_i\mP_0'\mS_i\mS_j^T = \mO_j^{-1}\mP_0\mS_i^T\mS_i\mS_j^T = \mO_j^{-1}\mP_0\mS_j^T = \mP_0'.
    \end{equation}
    Therefore, every element in $\mathcal{F} (\mP_0')$ stabilizes $\mP_0'$ and $\mathcal{F} (\mP_0') \subseteq \text{Stab}_{\Sn\times G_\eta(d)}(\mP_0')$. By~\cref{lem:minimal_frame}, there exists $\mP_0'' \in \text{Orb}_{\Sn \times G_\eta(d)}(\mP_0')$ such that $\text{Stab}_{\Sn\times G_\eta(d)}(\mP_0'')\subseteq \mathcal{F} (\mP_0'')$. Let $g \in \Sn\times G_\eta(d)$ such that $g\cdot \mP_0'' = \mP_0'$. Then we obtain
    \begin{equation}
        g\text{Stab}_{\Sn\times G_\eta(d)}(\mP_0'') \subseteq g'' \mathcal{F}(\mP_0'') = \mathcal{F}(\mP_0') \subseteq \text{Stab}_{\Sn\times G_\eta(d)}(\mP_0') = g \text{Stab}_{\Sn\times G_\eta(d)}(\mP_0'') g^{-1}
    \end{equation}
    where the last equality follows from~\cref{thm:stab}. Therefore, $\text{Stab}_{\Sn\times G_\eta(d)}(\mP_0'')\subseteq \text{Stab}_{\Sn\times G_\eta(d)}(\mP_0'') g^{-1}$, so $g = e$ and $\mP_0' = \mP_0''$. And we can obtain $\mathcal{F} (\mP_0') = \text{Stab}_{\Sn\times G_\eta(d)}(\mP_0')$. Since $\mP_0' \in \text{Orb}_{\Sn \times G_\eta(d)}(\mP)$, by~\cref{thm:minimal_frame}, this completes the proof.
    
\end{proof}
\colorblack
\section{Frame Averaging for Common Groups}
\label{sec:frame_group}

In this section, we detail the concrete equivariant frame averaging process for common groups, including translation group \((\mathbb{R}^{d}, +)\), linear algebraic group with equation \(\mO^T\eta \mO = \eta\), \(\mathrm{O}(d)/\mathrm{SO}(d)\) by~\citet{puny2021frame} and our improvement over it, \(\mathrm{E}(d)/\mathrm{SE}(d)\) as well as \(\mathrm{GL}(d,\mathbb{R})/\mathrm{SL}(d, \mathbb{R})\). We mainly focus on the domain of \(\mathbb{R}^{d\times n}\) and the equivariant frame averaging. For graph data and frame averaging on permutation group \(\mathrm{S}_n\), see~\cref{sec:mckay} and~\cref{sec:point_group}. And the invariant frame-averaging counterpart can be easily derived as it is a special case of equivariant frame averaging.

\subsection{Translation Group $(\mathbb{R}^{d}, +)$}
\label{sec:translation_group}

Consider the translation group $(\mathbb{R}^{d}, +)$ acting on a continuous function \(\Phi:\mathbb{R}^{d\times n}\rightarrow \mathbb{R}^{d\times n}\) by left addition via broadcasting. Let \(\mP\in \mathbb{R}^{d\times n}\), consider the induced \(\mathbb{R}^{d}\)-set \(\phi(\mP) = \frac{1}{n}\mP\mathbf{1}\) where \(\mathbf{1}\) is a all one column vector. As \(\mathbf{0}\) is in the orbit of \(\phi(\mP)\) and is unique, we can select \(\mathbf{0}\) as the canonical form and then the stabilizer is trivial. Thus, we obtain the frame \(\mathcal F_\phi (\phi(\mP)) = \frac{1}{n}\mP\mathbf{1}\) and the frame averaging \(\langle\Phi\rangle_{\mathcal{F}_\phi \circ \phi }(\mP) = f(\mP - \frac{1}{n}\mP\mathbf{1}) + \frac{1}{n}\mP\mathbf{1}\).

\subsection{Linear Algebraic Group by Induced $G$-set $\phi(\mP) = \mP\mM$}
\label{sec:frame_averaging_linear_algebraic_group}

Consider the linear algebraic group $G_\eta (d)$ by the equation \(\mO^T\eta\mO = \eta, \mO\in G_\eta(d)\) where \(\eta\) is a diagonal matrix with elements \(\pm 1\), and consider \(G_\eta(d)\) acting on a continuous function \(\Phi:\mathbb{R}^{d\times n}\rightarrow \mathbb{R}^{d\times n}\) by left multiplication. Consider the induced \(G_\eta(d)\)-set \(\phi(\mP) = \mP\mM\) where \(\mM = \mM\textsubscript{null}\mM\textsubscript{rank}\) is composed of two parts: \(\mM\textsubscript{null}\) filters all the null vectors of \(\mP\) and \(\mM\textsubscript{rank}\) selects \(\text{rank}(\mP\mM\textsubscript{null})\) linearly independent vectors. Let \(d' = \text{rank}(\phi(\mP))\).

\paragraph{Case I. \(d = d'\).} By the generalized QR decomposition presented in~\cref{sec:qr_decomposition}, \(\phi(\mP)\) can be decomposed into \(\hat{\mQ}\hat{\mR}\) and both \(\hat{\mQ} \in G_\eta (d)\) and \(\hat{\mR}\) is unique. Therefore, \(\hat{\mR}\) serves as the canonical form in the induced \(G_\eta(d)\)-set and the minimal frame \(\mathcal F_\phi(\phi(\mP)) = \{\hat{\mQ}\}\) by~\cref{thm:minimal_induced_by_QR}. Therefore, we obtain the minimal frame averaging
\begin{equation}
    \langle\Phi\rangle_{\mathcal{F}_\phi \circ \phi } (\mP) = \hat{\mQ}\Phi(\hat{\mQ}^{-1} \mP).
\end{equation}

\paragraph{Case II. \(d > d'\).} In this case, we only consider all column vectors in \(\mP\) to be non-null. By proof of~\cref{thm:linalgG_op}, let \(\mQ_0\) be \(\hat{\mQ}\) with all non-unique column vectors zero and \(\mP_0  = \hat{\mQ}^{-1} \mP\) be a fixed matrix, then by employing the counting measure \(\mu_{G_\eta(d)}\) and \(\rho(g)\coloneq \hat{\mQ}\) we obtain the frame averaging
\begin{equation}
\begin{split}
    \langle\Phi\rangle_{\mathcal{F}_\phi \circ \phi } (\mP) &= \int_{\mathcal{F}_\phi(\phi(\mP))} g\cdot \Phi(g^{-1}\cdot \mP)d\mu_{\mathcal{F}_\phi(\phi(\mP))}(g)\\
    &= \int_{\mathcal{F}_\phi(\phi(\mP))} g\cdot \Phi(\mP_0)d\mu_{\mathcal{F}_\phi(\phi(\mP))}(g)\\
    &= \mQ_0 \Phi(\mP_0).
\end{split}
\end{equation}

\paragraph{Determinant constraint.} Consider the special orthogonal group and proper Lorentz group where the group element has a determinant $1$. The simplest way to enforce determinant constraint to QR decomposition \(\phi(\mP) = \hat{\mQ}\hat{\mR}\) is to change the sign of one column of \(\hat{\mQ}\). We discuss this for different cases of \(d'\). If \(d = d'\) and \(\det (\hat{\mQ}) = -1\) then by flipping the sign of the last column of \(\hat{\mQ}\) and the last row of \(\hat{\mR}\), we can enforce \(\hat{\mQ}\) after flipping to belong to the group. If \(d = d' + 1\), the sign of the one indeterminant vector can be determined to ensure \(\det (\hat{\mQ}) = 1\). It is determined by adding one random linearly independent vector into $\phi(\mP)$ and follows the generalized Gram-Schmidt orthogonalization procedure. It is then treated as $d = d'$ case after orthogonalization. If \(d > d' +1\), there exist at least two non-unique orthonormal vectors. By flipping signs of two non-unique vectors of \(\hat{\mQ}\), we can obtain another orthonormal matrix belonging to the group, and by adding it to \(\hat{\mQ}\) we can cancel out those non-unique vectors without determining their signs. Therefore, if \(d > d' + 1\), the frame averaging is the same as the above \(d > d\) case without determinant constraint.

\subsection{Orthogonal Group $\mathrm{O}(d)$ by Induced $G$-set $\phi(\mP) = \mP\mP^T$}
\label{sec:ortho_ppt}

Consider the orthogonal group $\mathrm{O}(d)$ acting on a continuous function \(\Phi:\mathbb{R}^{d\times n}\rightarrow \mathbb{R}^{d\times n}\) by left multiplication. A induced \(G\)-set can be defined as
\[
\phi(\mP) = \mP\mP^T.
\]
with \(\mathrm{O}(d)\) acting upon \(\phi(\mP)\) by conjugate multiplication. Employing eigendecomposition, we have \(\mP\mP^T = \mO \mLambda \mO^T\), where \(\mLambda = \text{diag}(\lambda_1, \cdots, \lambda_d)\) is a diagonal matrix with its diagonal values in descending order. The canonical form is given by \(c(\mP\mP^T) = \mLambda\). With algebraic multiplicities of \(\mLambda\) as \(k_1, k_2, \cdots, k_m\) (where \(m \le d\)), the minimal frame at \(\mLambda\) can be defined as \(\hat{\mathcal{F}}_\phi (\mLambda) = \text{Stab}_{\mathrm{O}(d)}(\mLambda)\), where
\begin{equation}
    \text{Stab}_{\mathrm{O}(d)}(\mLambda) = \mathrm{O}(k_1) \times \mathrm{O}(k_2) \times \cdots \times \mathrm{O}(k_m).
\end{equation}

By the transitive property of group actions, there exists \(g_0 \in \mathrm{O}(d)\) such that \(\mP\mP^T = g_0 \cdot \mLambda = \rho(g_0) \mLambda \rho(g_0^{-1})\), where $\rho$ is the group representation of $\mathrm{O}(d)$. All the eigenvector matrices \(\mO\) of \(\mP\mP^T\) can be represented by such \(g_0\). Consequently, \(\mathcal{F}_\phi(g_0 \cdot \mLambda) = g_0 \cdot \mathcal{F}_\phi(\mLambda) = g_0 \text{Stab}_{\mathrm{O}(d)}(\mLambda)\) so that all such \(\mO\) constitute this left coset of \(\text{Stab}_{\mathrm{O}(d)}(\mLambda)\), which can be also shown by~\cref{lemma:multiCanMaps}. Hence, all minimal frame values \(\hat{\mathcal{F}}_\phi(\phi(\mP))\) can be deduced from the set of eigenvector matrices \(\mO\), aligning with the results in~\citet{puny2021frame}.

A potential issue arises when the eigenvalues of \(\mP\mP^T\) repeat or if \(\mP\mP^T\) is rank deficient, as this could lead to an infinite cardinality of \(\text{Stab}_{\mathrm{O}(d)}(\mP\mP^T)\). Merely selecting a matrix and forming a subset of \(\mathrm{O}(d)\) by altering the signs of the eigenvectors is insufficient, since frame averaging is effectively a group integral over (a coset of) the stabilizer. The following sections delve into two specific scenarios that emerge from this problem.

\paragraph{Case I. Rank-deficient \(\mP\mP^T\) without repeated nonzero eigenvalues.}

As previously mentioned, the rank deficiency of \(\mP\mP^T\) implies that the cardinality of \(\text{Stab}_{O(d)}(\mP\mP^T)\) is infinite, due to the infinite choices for eigenvectors corresponding to zero eigenvalues. Specifically, these eigenvectors, associated with zero eigenvalues, are orthogonal to each column vector in \(\mP\). Let \(\text{rank}(\mP) = d'\), and consider partitioning \(\mathcal{F}_\phi(\phi(\mP))\) into \(\{\mathcal{F}^i_\phi(\phi(\mP))\}_{i=1}^{2^{d'}}\), where each \(\mathcal{F}^i_\phi(\phi(\mP))\) consists of orthogonal matrices \(\rho(g)\coloneq\mO = [\mO_i^{d'}, \bar{\mO}]\), having a fixed \(\mO_i^{d'}\) with \(d'\) vectors and varying \(\mO^{*}\) with \(d-d'\) non-unique orthonormal vectors. The differences among \(\mO_i^{d'}\) in different \(\mathcal{F}^i_\phi(\mP)\) are the signs of each columns of \(\mO_i^{d'}\). For all \(\mO\in \mathcal{F}^i_\phi(\phi(\mP))\), \(\mP_i = \mO^{-1} \mP\) are the same as column vectors of \(\bar{\mO}\) are orthogonal to \(\mP\). 

Integrating over \(\bar{\mO}\) effectively sums over orthogonal transformations in the residual subspace, but these terms cancel out due to symmetry, i.e., by flipping the sign of the column in \(\bar{\mO}\), \(\mO\) is still an orthogonal matrix and summing up those \(\bar{\mO}\) will lead to zero. Define \(\mO_i = [\mO_i^{d'}, \mathbf{0}]\) and the counting measure $\mu_{\mathrm{O}(d)}$, and the frame averaging can be represented as
\begin{equation}
\begin{split}
   \langle\Phi\rangle_{\mathcal{F}_\phi \circ \phi }(\phi(\mP)) &= \int_{\mathcal{F}_\phi(\phi(\mP))}  g\cdot \Phi(g^{-1}\cdot \mP) \, d\mu_{\mathcal{F}_\phi(\phi(\mP))}(g) \\
    &=\sum_{i=1}^{2^{d'}} \int_{\mathcal{F}^i_\phi(\phi(\mP))} g\cdot \Phi(\mP_i) \, d\mu_{\mathcal{F}_\phi(\phi(\mP))}(g) \\
    &= \frac{1}{2^{d'}}\sum_{i=1}^{2^{d'}} \mO_i \Phi(\mP_i).
\end{split}
\end{equation}

\paragraph{Case II. \(\mP\mP^T\) with repeated nonzero eigenvalues.}
One method to estimate frame averaging in the presence of repeated nonzero eigenvalues is to use Monte Carlo integration, where each orthogonal matrix is sampled by drawing from each \(\mathrm{O}(k_i)\). However, this approximation becomes inefficient when a certain algebraic multiplicity \(k_i\) is large, especially in high-dimensional spaces. To address this, we suggest a first-order perturbation-based method to reduce the algebraic multiplicities of the eigenvalues. Consider a non-negative vector \(\vw \in \mathbb{R}^n\) and apply weighted PCA~\citep{atzmon2022frame}:
\begin{equation}
    \phi(\mP) = \mP \, \text{diag}(\vw) \, \mP^T,
\end{equation}
where \(\vw\) is \(\mathrm{O}(d)\)-invariant to \(\mP\), making \(\phi\) an \(\mathrm{O}(d)\)-equivariant function, thus inducing a induced \(G\)-set. The design of \(\text{diag}(\vw)\) can be such that it maps each \(\mP\) to a matrix with distinct nonzero eigenvalues, constructing the frame from the eigenvectors of \(\phi(\mP)\). Note that if \(\text{diag}(\vw) = \mI_n\), then \(\phi(\mP)\) reduces to the original case without perturbation.

Intuitively, \(\vw\) serves as scaling factors for each vector in \(\mP\). Adjusting the length of each vector in \(\mP\) helps to break the symmetry and reduce the corresponding cardinality of the stabilizer. Now, consider varying \(\vw\) starting from the all-one vector \(\mathbf{1}\). This alteration in \(\vw\) leads to perturbation in the eigenstructure of \(\phi(\mP)\).

\begin{theorem}
    There exist a non-negative vector $\vz\in \mathbb{R}^n$ such that all nonzero eigenvalues of $\phi(\mP)=\mP\text{diag}(\mathbf{1} + \vz)\mP^T$ have algebraic multiplicity exactly one.
\end{theorem}

\begin{proof} 
Given a matrix $\phi(\mP)=\mP\mP^T\in \mathbb{R}^{d\times d}$, consider the following perturbation procedure to lift degeneracy in its eigenvalues. The perturbation is performed on $\mP\mP^T$ using a series of vectors $\{\vz_l\}_{l=1}^{l_{\text{max}}}$ for each degenerate eigenvalue. The aim is to adjust the eigenvalues to resolve degeneracies without introducing new ones. The perturbation process begins by acknowledging that 
\begin{equation}
    \mP\text{diag}(\mathbf{1} + \vz)\mP^T = \mP\mP^T + \mP \text{diag}(\vz) \mP^T.
\end{equation}
Thus, $\phi(\mP)$ is effectively a perturbation by $\mP \vz \mP^T$. Under perturbation theory, the eigenvalue $\lambda'$ and eigenvector $\vv'$ of the perturbed matrix can be expressed as 
\begin{equation}
    \lambda' = \lambda^{(0)} + \sigma\lambda^{(1)} + \sigma^2\lambda^{(2)} + \cdots, \quad \vv' = \vv^{(0)} + \sigma\vv^{(1)} + \sigma^2\vv^{(2)} + \cdots
\end{equation}
where $\lambda^{(i)}$ and $\vv^{(i)}$ are the $i$-th order terms in the expansion, with $\lambda^{(0)}$ and $\vv^{(0)}$ being the eigenvalue and eigenvector of $\mP\mP^T$, respectively. Given a degenerate eigenvalue $\lambda$ and its degenerate eigenvectors $\vv_1, \vv_2,\cdots, \vv_m$, the first-order perturbation matrix $\mM \in \mathbb{R}^{m\times m}$ is defined with entries
\begin{equation}
    \mM_{ij} = \vv_i\mP \text{diag}(\vz) \mP^T \vv_j.
\end{equation}
Consider the iterative rank-one perturbation. The perturbation begins by iteratively modifying $\mP\mP^T$ with vectors $\{\vz_l\}$, where $\vz_l$ is non-zero only at the $k_l$-th entry during the $l$-th iteration for each degenerate eigenvalue. For example, the first iteration involves a rank-one perturbation using $\vz_1$, leading to
\begin{equation}
    \mM_{ij} = (\vz_1)_{k_1} (\mP^T \vv_i)_{k_1} (\mP^T \vv_j)_{k_1},
\end{equation}
which simplifies to $\mM = (\vz_1)_{k_1} \vu_1 \vu_1^T$ with $\vu_1 = \left[(\mP^T \vv_1)_{k_1}, \cdots, (\mP^T \vv_m)_{k_1}\right]$. The choice of $(\vz_1)_{k_1}$ is critical. It is selected through an iterative approach, incrementing from a small value and ensuring that the resulting perturbed eigenvalue does not introduce new degeneracies or match any other existing eigenvalue of $\mP\mP^T$. This process continues, with the perturbed eigenvalue being
\begin{equation}
    \lambda + (\vz_1)_{k_1} \sum_{i=1}^m (\vu_1)_i^2,
\end{equation}
and the algebraic multiplicity of $\lambda$ is reduced by one after each perturbation. Subsequent iterations involve perturbing $\mP_1\mP_1^T = \mP\text{diag}(\mathbf{1} + \vz_1)\mP^T$ with $\vz_2$, and so on. The process requires $l_{\text{max}} = m - 1$ iterations to fully resolve the degeneracy of $\lambda$. For a set of unique eigenvalues $\{\lambda_1, \lambda_2, \cdots, \lambda_{d'}\}$ of $\mP\mP^T$, where $\eta(\lambda)$ denotes the multiplicity of $\lambda$, the total time complexity of these perturbations is $O(\sum_{i=1}^{d'}\eta(\lambda_i) - d')$. The final vector $\vz$ is given by 
\begin{equation}
    \vz = \prod_{l=1}^{l_{\text{max}}} (\mathbf{1} + \vz_l) - \mathbf{1}.
\end{equation}
This procedure ensures that the eigenvalues of $\mP\mP^T$ are perturbed in a controlled manner to lift degeneracies without introducing new ones.
\end{proof}

The algorithm is shown in~\cref{alg:perturbation}. For efficiency, the algorithm does not recompute $\mP^T\vv$ but uses the one at the start of the iteration. Also note that for each eigenspace, the eigenvectors are produced from the Gram-Schmidt process by the projection of vectors of $\mP$, in each iteration, $\vz_l$ is $\mathrm{O}(d)$ invariant, and thus the perturbation $\vz$ is $\mathrm{O}(d)$ invariant. As a result, $\phi(\mP)=\mP\text{diag}(\mathbf{1} + \vz)\mP^T$ is an \(\mathrm{O}(d)\)-equivariant frame without degenerate eigenvalues.

\begin{algorithm}
\caption{Eigenvalue Perturbation for Degeneracy Reduction}
\label{alg:perturbation}
\begin{algorithmic}[1]
\STATE {\bfseries Input:} matrix $\mP$, perturbation increment $\epsilon$, perturbation start $\epsilon\textsubscript{start}$, maximum times of sub-iteration $l\textsubscript{submax}$
\STATE $\vz \gets \mathbf{1}$
\STATE $\vk\textsubscript{set} \gets \text{set}()$
\STATE Compute eigenvalues and eigenvectors of $\mP\mP^T$, denoted as $\lambda$ and $\vv$
\FOR{each degenerate eigenvalue $\lambda$}
    \STATE $m\gets \text{algebraic multiplicity of } \lambda$
    \STATE Project $\mP$ into the eigenspace of $\lambda$ as $\mP\textsubscript{proj}$
    \STATE Select $m$ linearly independent vectors from $\mP\textsubscript{proj}$
    \STATE Perform Gram-Schmidt orthogonalization to $m$ linearly independent vectors to obtain eigenvectors $\vv_1, \vv_2, \cdots, \vv_m$
    \STATE $\mV= \left[\mP^T \vv_1, \mP^T \vv_2, \cdots, \mP^T \vv_m\right]$
    
    \FOR{$l = 1$ to $m - 1$}
        \FOR{$k = 1$ to $d$}
            \STATE $\vz_l \gets \mathbf{0}$ 
            \IF{$k$ in $\vk\textsubscript{set}$}
                \STATE $\textbf{continue}$
            \ENDIF
            \STATE $\vu \gets \left[\mV_{1k}, \mV_{2k}, \cdots, \mV_{mk}\right]$
            \IF{$\sum_{i=1}^m (\vu)_i^2$ = 0}
                \STATE $\textbf{continue}$
            \ELSE
            \STATE Add $k$ to $\vk\textsubscript{set}$ 
            \ENDIF
            \STATE $\epsilon\textsubscript{perturb} = \epsilon\textsubscript{start}$
            \STATE $l\textsubscript{sub} = 0$
            \WHILE{true}
                \STATE $l\textsubscript{sub} = l\textsubscript{sub} + 1$
                \STATE $(\vz_l)_k = \epsilon\textsubscript{perturb} * \lambda / \sum_{i=1}^m (\vu)_i^2 $
                \STATE Compute perturbed eigenvalue of $\mP\text{diag}(\mathbf{1} + \vz_l)\mP^T$
                \IF{no new degeneracy is introduced}
                    \STATE $\textbf{break}$
                \ENDIF
                \IF{$l\textsubscript{sub} \ge l\textsubscript{submax}$}
                    \STATE $\textbf{break}$
                \ENDIF
            \ENDWHILE
            \STATE $\mP \gets \mP \sqrt{\text{diag}(\mathbf{1} + \vz_l)}$
            \STATE $\vz \gets \vz (\mathbf{1} + \vz_l)$
            \STATE $\textbf{break}$
        \ENDFOR
    \ENDFOR
\ENDFOR
\STATE $\vz \gets \vz - \mathbf{1}$
\STATE \textbf{return} $\vz$
\end{algorithmic}
\end{algorithm}

\subsection{Euclidean Group $\mathrm{E}(d)$ and Special Euclidean Group $\mathrm{SE}(d)$}

Consider the Euclidean group $\mathrm{E}(d) = \mathbb{R}^d \rtimes \mathrm{O}(d)$ and special Euclidean group $\mathrm{SE}(d) = \mathbb{R}^d \rtimes \mathrm{SO}(d)$ acting on a continuous function $\Phi:\mathbb{R}^{d\times n}\rightarrow \mathbb{R}^{d\times n}$. By~\citet{puny2021frame}, this can be achieved such that first create a translation-equivariant function by~\cref{sec:translation_group}, then apply a \(\mathrm{O}(d)/\mathrm{SO}(d)\)-equivariant frame to this translation-equivariant function, which can be derived from our previous~\cref{sec:frame_averaging_linear_algebraic_group}.

\subsection{General Linear Group \(\mathrm{GL}(d,\mathbb{R})\) and Special Linear Group \(\mathrm{SL}(d, \mathbb{R})\)}

Consider the general linear group $\mathrm{GL}(d, \mathbb{R})$ and special linear group $\mathrm{SL}(d,\mathbb{R})$ acting on the left on a continuous function $\Phi:\mathbb{R}^{d\times n}_*\rightarrow \mathbb{R}^{d\times n}_*$. Let $\mP \in \mathbb{R}^{d\times n}_*$, consider the induced $\mathrm{GL}(d,\mathbb{R})$-set orr $\mathrm{SL}(d,\mathbb{R})$-set \(\phi(\mP) = \mP\mM\), where the mask \(\mM \in \mathbb{R}^{n \times d}\) is defined similar to~\cref{sec:GLset} selecting the non-zero and linearly independent vectors from $\mP$. 

For $\mathrm{GL}(n,\mathbb{R})$, we have $\phi(\mP) \in \mathrm{GL}(n,\mathbb{R})$ and thus $\phi(\mP)^{-1} \in \mathrm{GL}(n,\mathbb{R})$. Let the canonical form $c(\phi(\mP)) = \phi(\mP)^{-1}\phi(\mP) = \mI_d$ which is in the orbit of \(\phi(\mP)\) and is unique. Then $\text{Stab}_{\mathrm{GL}(d, \mathbb{R})}(\mI_d) = \{\mI_d\}$ and the minimal frame $\hat{\mathcal{F}_\phi}(\phi(\mP)) = \{\phi(\mP)\}$, and we obtain minimal frame averaging $\langle\Phi\rangle_{\mathcal{F}_\phi \circ \phi }(\mP) = \phi(\mP) \Phi(\phi(\mP)^{-1}\mP)$. 

On the other hand, let $\mD = \frac{\text{sign}(\text{det}\phi(\mP))}{\sqrt[d]{|\text{det}\phi(\mP)|}} \mI_d$ and we have $\phi(\mP)\mD \in \mathrm{SL}(d,\mathbb{R})$ and thus $\mD^{-1}\phi(\mP)^{-1} \in \mathrm{SL}(d,\mathbb{R})$. Let the canonical form $c(\phi(\mP)) = \mD^{-1}\phi(\mP)^{-1} \phi(\mP) = \mD^{-1}$ which is in the orbit of \(\phi(\mP)\) and $\mathrm{SL}(d, \mathbb{R})$-invariant. Then, $\text{Stab}_{\mathrm{SL}(d, \mathbb{R})}(\mD^{-1}) = \{\mI_d\}$ and the minimal frame $\hat{\mathcal{F}}_\phi(\phi(\mP)) = \{\phi(\mP)\mD\}$, and we obtain minimal frame averaging $\langle\Phi\rangle_{\mathcal{F}_\phi \circ \phi }(\mP) = \phi(\mP)\mD \Phi(\mD^{-1}\phi(\mP)^{-1}\mP)$.

Note that the induced set $\phi(\mP)$ requires a determinant bound to ensure the stability of matrix inverse computation.

\begin{table}[t]
\begin{center}
\caption{Canonical forms of common groups on the domain of \(\mathcal S = \mathbb{R}^{d\times n}, \mathbb{C}^{d\times n}, \mathbb{R}^{d\times n}_*\) or \(\mathbb{C}^{d\times n}_*\). While the permutation group \(S_n\) acts on \(\mathcal S\) by right multiplication, all the other groups act on \(\mathcal S\) by left multiplication. Here, lexicographic sort refers to the sorting result of the columns of \(\mP \in \mathcal S\) by the row lexicographic order. The canonical graph is derived from the canonical labeling method~\citep{mckay2014practical}. The canonical form \(\hat{\mR}\) is derived from our generalized QR decomposition, where \(\phi(\mP) = \hat{\mQ} \hat{\mR}\) and \(\hat{\mQ}\) belongs to the corresponding group. The canonical form \(\mH\) is derived from the polar decomposition \(\phi(\mP) = \mU\mH\). For canonical form \(\mLambda\) and \(\mJ\), \(\mLambda_{ii}\) and \(\mJ_i\) denote the eigenvalue and the Jordan block, respectively. $\sim$ denotes the equivalent relation, and $\mLambda_{11}\sim \cdots \sim \mLambda_{dd}$ and $\mJ_1\sim\cdots\sim \mJ_m$ mean eigenvalues and Jordan blocks are unique up to permutation.}
\label{tb:quotient}

\resizebox{\columnwidth}{!}{
\begin{tabular}{cccccc}
\toprule
Group $G$ & Domain $\mathcal{S}$ & $\phi(\mP)$ & Induced $G$-set $\mathcal{S}_\phi$ & Canonical Form & Matrix Decomposition \\ 
\hline
\midrule
$\mathrm{S}_n$ & $\mathbb{R}^{d\times n}$ & $\mP$ & $\mathbb{R}^{d\times n}$ & Lexicographical sort & -\\
$\mathrm{S}_n$ & $\mathbb{R}^{d \times n}$ & $\mP^T\mP$ & $\mathrm{Sym}(n, \mathbb{R})$ & Canonical graph & -\\
$(\mathbb{R}^d,+)$ & $\mathbb{R}^{d\times n}$ & $\frac{1}{n}\mP\mathbf{1}$ & $\mathbb{R}^{d}$ & $\mathbf{0}$ & -\\
$\mathrm{O}(d)$ & $\mathbb{R}^{d\times n}$ & $\mP\mM$ & $\mathbb{R}^{d\times m}_*, d\ge m$ & $\hat{\mR}$ & QR decomposition \\
$\mathrm{SO}(d)$ & $\mathbb{R}^{d\times n}$ & $\mP\mM$ & $\mathbb{R}^{d\times m}_*, d\ge m$ & $\hat{\mR}$ & QR decomposition \\
$\mathrm{O}(1, d-1)$ & $\mathbb{R}^{d\times n}$ & $\mP\mM$ & $\mathbb{R}^{d\times m}_*, d\ge m$ & $\hat{\mR}$ & QR decomposition \\
$\mathrm{SO}(1, d-1)$ & $\mathbb{R}^{d\times n}$ & $\mP\mM$ & $\mathbb{R}^{d\times m}_*, d\ge m$ & $\hat{\mR}$ & QR decomposition \\
$\mathrm{U}(d)$ & $\mathbb{C}^{d\times n}_*$ & $\mP\mM$ & $\mathbb{C}^{d\times d}_*$ & $\hat{\mR}$ & QR decomposition \\
$\mathrm{SU}(d)$ & $\mathbb{C}^{d\times n}_*$ & $\mP\mM$ & $\mathbb{C}^{d\times d}_*$ & $\hat{\mR}$ & QR decomposition \\
$\mathrm{SL}(d, \mathbb{R})$ & $\mathbb{R}^{d\times n}_*$ & $\mP\mM$ & $\mathbb{R}^{d\times d}_*$ & $\text{sign}(\text{det}\phi(\mP))\sqrt[d]{|\text{det}\phi(\mP)|}\mI_d$ & - \\
$\mathrm{GL}(d, \mathbb{R})$ & $\mathbb{R}^{d\times n}_*$ & $\mP\mM$ & $\mathbb{R}^{d\times d}_*$ & $\mI_d$ & - \\
$\mathrm{O}(d)/\mathrm{SO}(d)$ & $\mathbb{R}^{d\times n}$ & $\mP\mP^T$ & $\mathrm{Sym}(d, \mathbb{R})$ & $\mLambda\mid \mLambda_{11}\sim \cdots \sim \mLambda_{dd}$ & Eigendecomposition \\
$\mathrm{U}(d)/\mathrm{SU}(d)$ & $\mathbb{C}^{d\times n}$ & $\mP\mP^T$ & $\mathrm{Sym}(d, \mathbb{C})$ & $\mLambda \mid \mLambda_{11}\sim \cdots \sim \mLambda_{dd}$ & Eigendecomposition \\
$\mathrm{O}(d)$ & $\mathbb{R}^{d\times n}_*$ & $\mP\mP^T$ & $\mathrm{SPos}(d, \mathbb{R})$ & $\mH$ & Polar decomposition \\
$\mathrm{U}(d)$ & $\mathbb{C}^{d\times n}_*$ & $\mP\mP^T$ & $\mathrm{SPos}(d, \mathbb{C})$ & $\mH$ & Polar decomposition \\
$\mathrm{GL}(d, \mathbb{C})/\mathrm{SL}(d, \mathbb{C})$ & $\mathbb{C}^{d\times n}$ & $\mP\mP^T$ & $\mathrm{Sym}(d, \mathbb{C})$ & $\text{diag}(\mJ_1, \cdots, \mJ_m)\mid \mJ_1\sim\cdots\sim \mJ_m$ & Jordan decomposition \\
\bottomrule
\end{tabular}
}
\end{center}
\end{table}

\section{Experimental Details}
\label{sec:experimental_details}

\subsection{Equivariance Errors on Random Synthetic Data}
\label{sec:relative_error}
In this section, we first present the equivariance error defined in~\cref{eqn:relative_error} for several common groups, including \(\mathrm{O}(d), \mathrm{SO}(d),\mathrm{U}(d),\mathrm{SU}(d), \mathrm{O}(1, d-1), \mathrm{SO}(1, d-1), \mathrm{E}(d),\mathrm{SE}(d), \mathrm{GL}(d, \mathbb{R})\), and \(\mathrm{SL}(d,\mathbb{R})\). We denote the model without any frames as \textbf{Plain}, our minimal frame averaging as \textbf{MFA},~\citet{puny2021frame}'s frame averaging method as \textbf{FA} and~\citet{duval2023faenet}'s stochastic frame averaging as \textbf{SFA}. We adopt six models to test the equivariance error. Unless otherwise specified, they are GIN~\citep{xu2018powerful}, GCN~\citep{kipf2016semi}, an MLP model, an MLP model with batch normalization, a nonlinear function with ReLU given by
\begin{equation}
    \Phi\textsubscript{ReLU}(x) = x + \text{ReLU}(x)
\end{equation}
and a nonlinear function with the $\sin$ function given by
\begin{equation}
    \Phi\textsubscript{Sine}(x) = \sin(x) - \frac{x}{\|x \|}.
\end{equation}
We provide further details regarding the synthetic data and the corresponding groups in the captions of the below figures. Unless otherwise specified, each group experiment is conducted on 100 point cloud samples with coordinates independently drawn from a Gaussian distribution with a mean of $0$ and a standard deviation of $1$ for each value. To better visualize the figure, we apply log scaling. To prevent zero input, \textit{e.g.}, in the equivariance test of the permutation group, we add \(\epsilon=10^{-12}\) to each resulting value when encountering zero input.

\begin{figure}[ht]
\centering
\begin{tabular}{c@{\hspace{50pt}}c}
    \includegraphics[height=0.2\textheight]{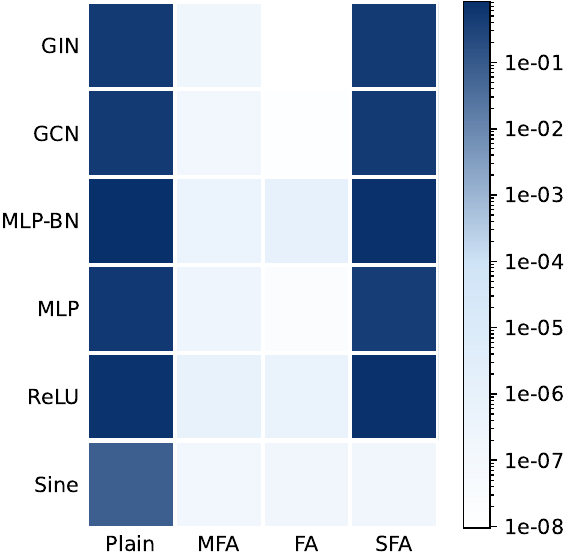} &
    \includegraphics[height=0.2\textheight]{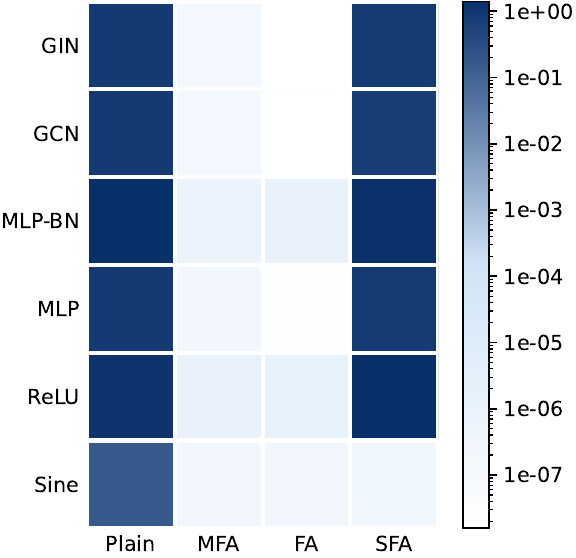} \\
    (a) $\mathrm{O}(3)$ & (b) $\mathrm{O}(5)$ \\
    \includegraphics[height=0.2\textheight]{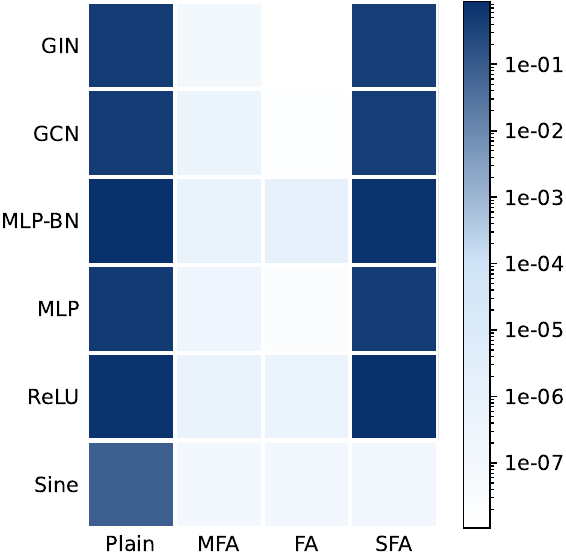} &
    \includegraphics[height=0.2\textheight]{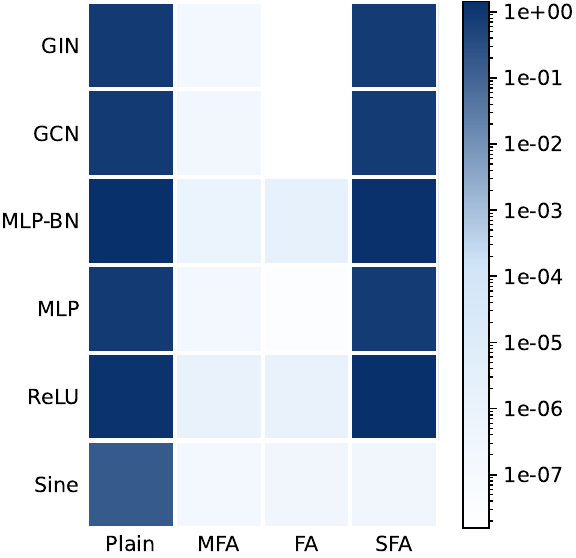} \\
    (c) $\mathrm{SO}(3)$ & (d) $\mathrm{SO}(5)$ \\
    \includegraphics[height=0.2\textheight]{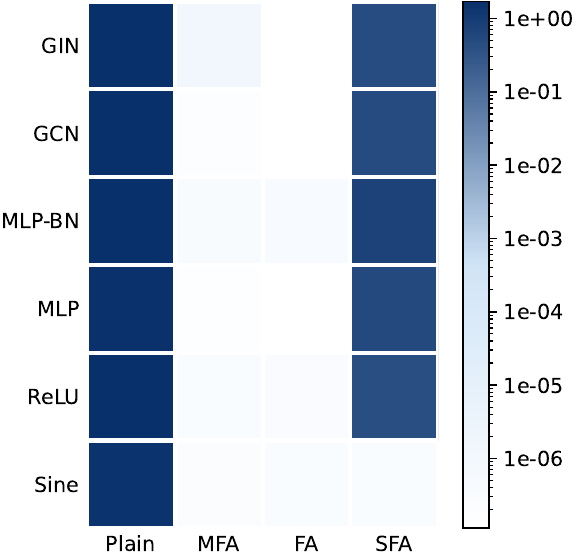} &
    \includegraphics[height=0.2\textheight]{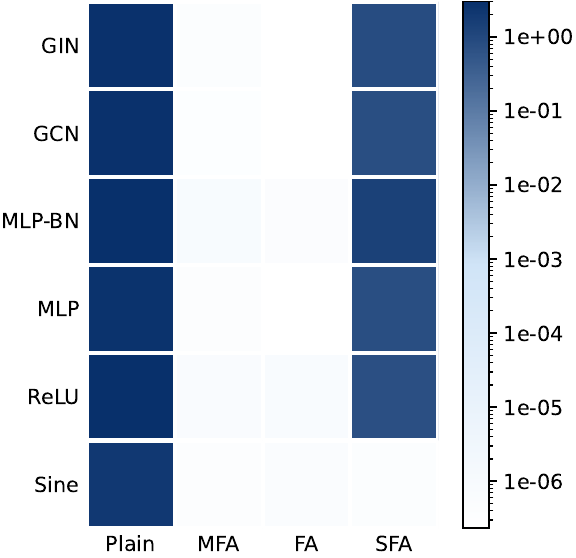} \\
    (e) $\mathrm{E}(3)$ & (f) $\mathrm{E}(5)$\\
    \includegraphics[height=0.2\textheight]{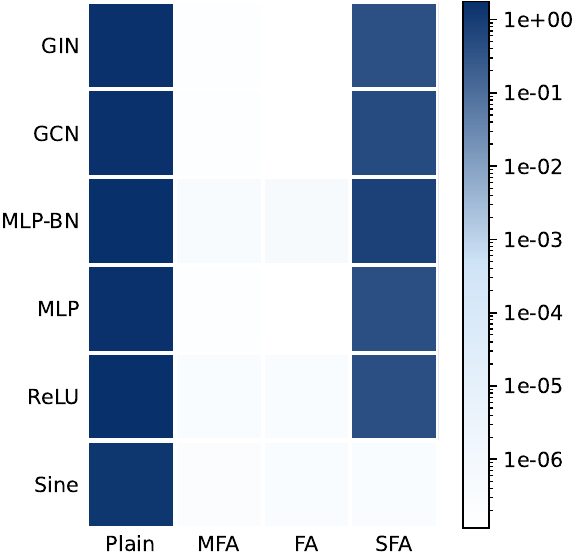} &
    \includegraphics[height=0.2\textheight]{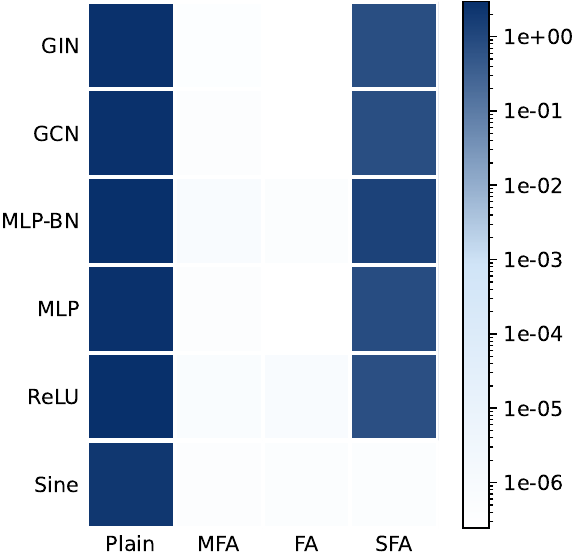} \\
    (g) $\mathrm{SE}(3)$ & (h) $\mathrm{SE}(5)$ \\
\end{tabular}
\caption{Comparative illustrations of equivariance test on the groups \(\mathrm O(3), \mathrm O(5), \mathrm{SO}(3), \mathrm{SO}(5), \mathrm E(3), \mathrm E(5), \mathrm{SE}(3)\), and \(\mathrm{SE}(5)\). Figures (a) and (b) depict the outcomes for $\mathrm O(3)$ and $\mathrm O(5)$, figures (c) and (d) display the results for $\mathrm{SO}(3)$ and $\mathrm{SO}(5)$, figures (e) and (f) show the results for $\mathrm{E}(3)$ and $\mathrm{E}(5)$, figures (g) and (h) demonstrate the results for $\mathrm{SE}(3)$ and $\mathrm{SE}(5)$, respectively. Data for left column is randomly sampled with a shape of \(100\times 3\) and data for the right column is randomly sampled with a shape of \(20\times 5\). All figures show the effectiveness of both our frame averaging method and original frame averaging method~\citep{puny2021frame}.}
\end{figure}

\begin{figure}[ht]
\centering
\begin{tabular}{cccc}
    \includegraphics[width=0.2\textwidth]{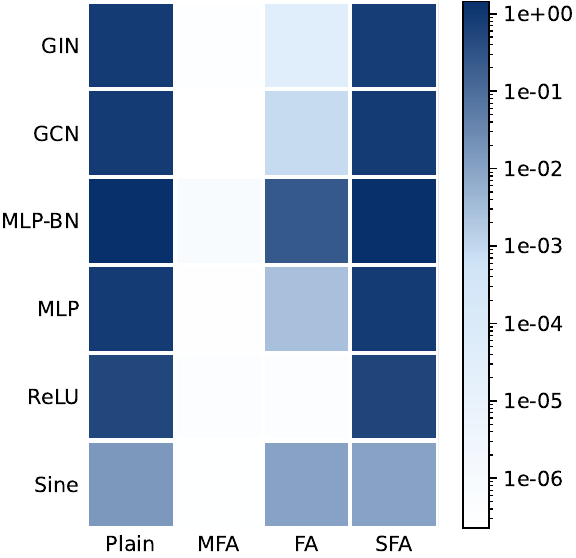} &
    \includegraphics[width=0.2\textwidth]{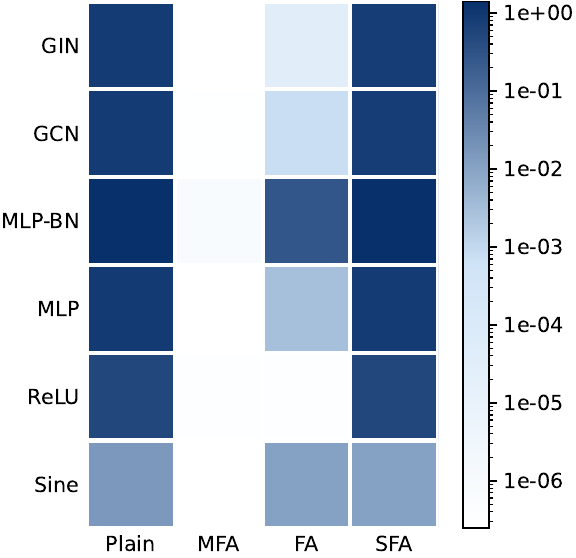} &
    \includegraphics[width=0.2\textwidth]{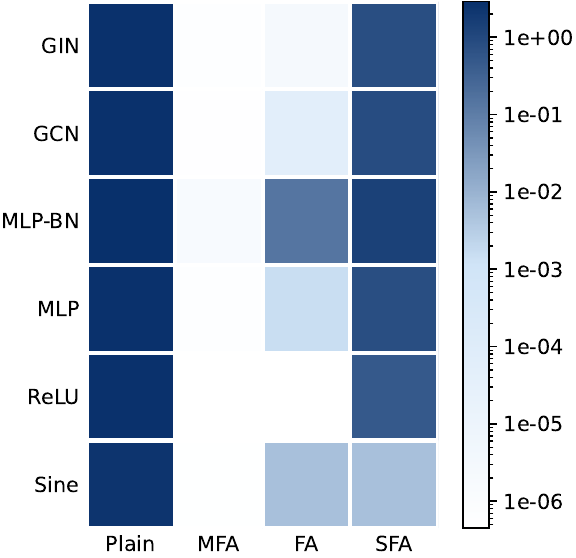} &
    \includegraphics[width=0.2\textwidth]{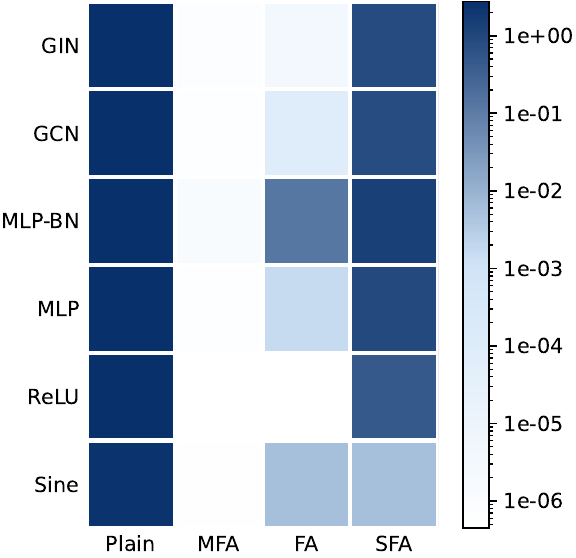} \\(a) $\mathrm{O}(5)$ & (b) $\mathrm{SO}(5)$ &
    (c) $\mathrm{E}(5)$ & (d) $\mathrm{SE}(5)$
\end{tabular}
\caption{Illustrations of equivariance test on the groups \(\mathrm O(5), \mathrm{SO}(5), \mathrm{E}(5), \mathrm{SE}(5)\) with degenerate singular values. All data is randomly sampled with a shape of \(20\times 5\) with $3$ repeated singular values. As shown above, the original frame averaging method~\citep{puny2021frame} fails the degenerate cases due to the repeated eigenvalues causing the frame size into infinity, while our method is not affected by the repeated eigenvalues and our minimal frame averaging is still equivariant to these groups.}
\end{figure}

\begin{figure}[ht]
\centering
\begin{tabular}{cccc}
    \includegraphics[height=0.2\textheight]{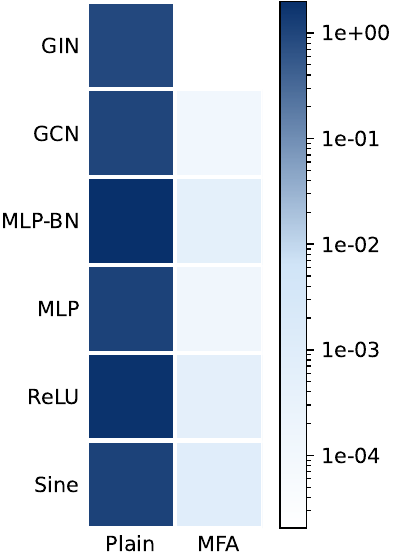} &
    \includegraphics[height=0.2\textheight]{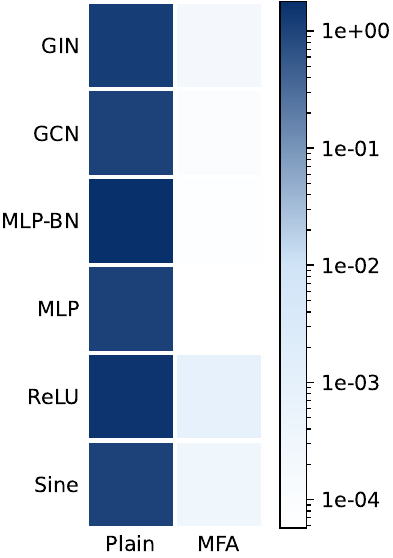} &
    \includegraphics[height=0.2\textheight]{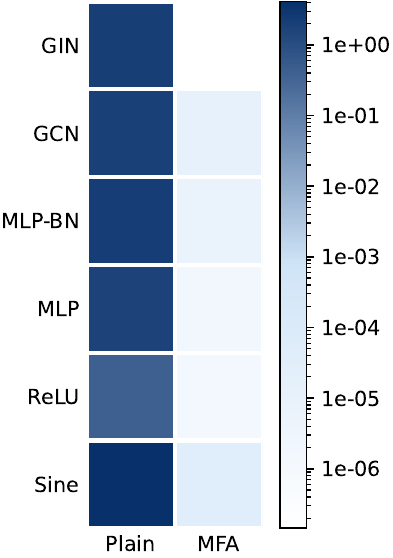} &
    \includegraphics[height=0.2\textheight]{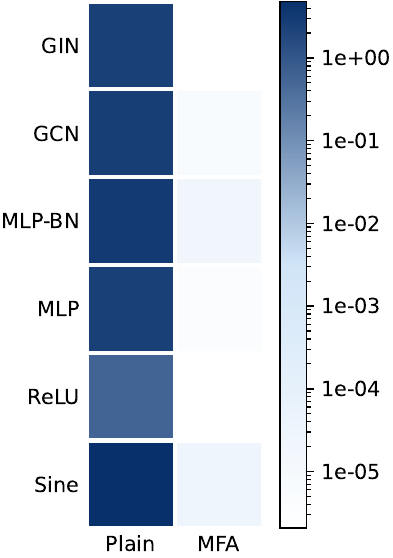} \\(a) $\mathrm{O}(1, 3)$ & (b) $\mathrm{SO}(1,3)$ &
    (c) $\mathrm{GL}(3,\mathbb{R})$ & (d) $\mathrm{SL}(3,\mathbb{R})$
\end{tabular}
\caption{Illustrations of equivariance test on the groups \(\mathrm O(1,3), \mathrm{SO}(1,3), \mathrm{GL}(3,\mathbb{R}), \mathrm{SL}(3,\mathbb{R})\). As the original frame averaging method~\citep{puny2021frame} does not give the frame construction for these groups, we only compare with the model without our method. The data for \(\mathrm O(1,3)\) and \(\mathrm{SO}(1,3)\) is randomly sampled with a shape of \(100\times 4\), and data for \(\mathrm{GL}(3,\mathbb{R})\) and \(\mathrm{SL}(3,\mathbb{R})\) is randomly sampled with a shape of \(100\times 3\). All the error scaling our method in the figures of is below \(1e^{1e-3}\), showing that our method is indeed equivariant with respect to these group.}
\end{figure}

\begin{figure}[ht]
\centering
\begin{tabular}{cccc}
    \includegraphics[width=0.2\textwidth]{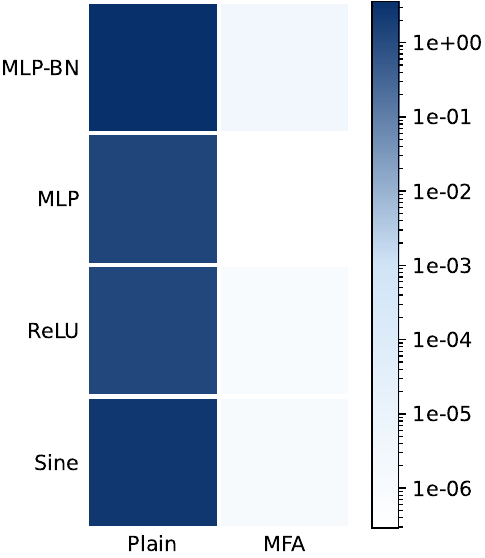} &
    \includegraphics[width=0.2\textwidth]{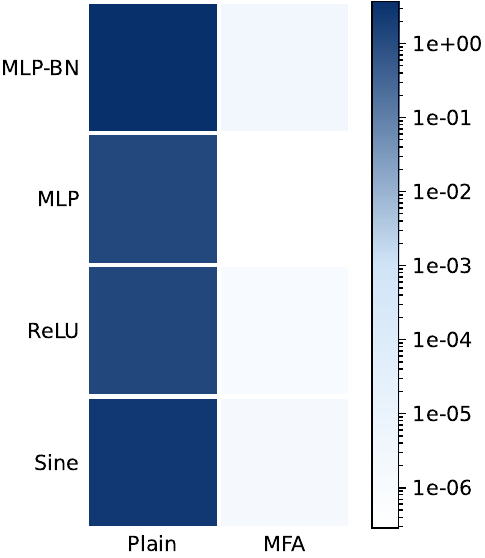} &
    \includegraphics[width=0.2\textwidth]{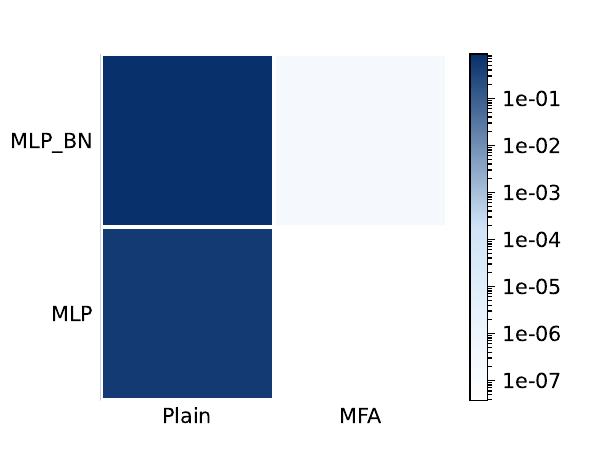} &
    \includegraphics[width=0.2\textwidth]{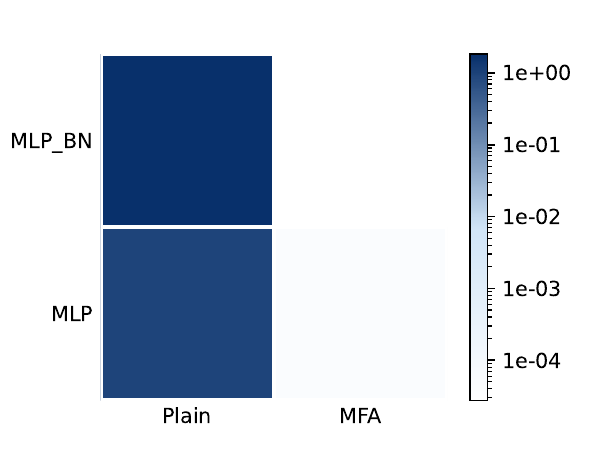} \\(a) $\mathrm{U}(3)$ & (b) $\mathrm{SU}(3)$ &
    (c) $\mathrm{S}_n\times \mathrm{O}(3)$ & (d) $\mathrm{S}_n\times\mathrm{O}(1,3)$
\end{tabular}
\caption{Illustrations of equivariance test on the groups \(\mathrm U(3), \mathrm{SU}(3), \mathrm{S}_n, \mathrm{S}_n \times \mathrm{O}(3), \mathrm{S}_n \times \mathrm{O}(1,3)\). In (d), the first two rows correspond to the error test of $\mathrm{S}_n \times \mathrm{O}(3)$ and the last two rows correspond to that of $\mathrm{S}_n \times \mathrm{O}(1,3)$. As the original frame averaging method~\citep{puny2021frame} does not give the frame construction for these groups, we only compare with the model without our method. The data for \(\mathrm U(3)\) and \(\mathrm{SU}(3)\) is randomly sampled with a shape of \(100\times 3\), data for \(\mathrm{S}_n\times \mathrm{O}(3)\) is randomly sampled with a shape of \(32\times 3\), and data for \(\mathrm{S}_n\times \mathrm{O}(1, 3)\) is randomly sampled with a shape of \(32\times 4\). Note that the networks used for these groups are different from the previous groups to accommodate the properties of these groups. The MLP models used for both \(\mathrm{U}(3)\) and \(\mathrm{SU}(3)\) are complex valued networks, and the MLP models used for \(\mathrm{S}_n\times\mathrm{O}(3) \) and \(\mathrm{S}_n\times\mathrm{O}(1,3)\) are transforming both the node dimension and the feature dimension, and are neither \(\mathrm{S}_n\times \mathrm{O}(3)\) nor \(\mathrm{S}_n\times\mathrm{O}(1,3)\)-equivariant. Specially, the \(\mathrm{S}_n\times \mathrm{O}(3)\) or \(\Sn\times\mathrm{O}(1,3)\)-equivariant frame is created by applying \(\mathrm{S}_n\)-equivariant and \(\mathrm{O}(3)\) or \(\mathrm{O}(1,3)\)-invariant frame to the MLP with \(\mathrm{O}(3)\) or \(\mathrm{O}(1,3)\)-equivariant frame, corresponding to~\cref{sec:pcloud}.}
\end{figure}

\subsection{\(n\)-Body Problem}
\label{sec:training_n_body}

\paragraph{Model.} Similar to~\citet{puny2021frame, kaba2023equivariance, ruhe2023clifford}, our model follows the architecture proposed by~\citet{satorras2021n}. Inspired by~\citet{ruhe2023clifford}, it integrates layer normalization~\cite{ba2016layer} and dropout~\cite{hinton2012improving} in message passing layers. The architecture consists of 4 message passing layers, each with a hidden dimension of 60, and employs a dropout rate of 0.1. Our generalized QR decomposition is used to compute the \(\mathrm O(3)\)-equivariant frame, which maintains a constant size of 1 across all data points, ensuring model equivariance with a single forward pass.

\paragraph{Training.} The model is optimized using a Mean Absolute Error (MAE) loss criterion, employing the Adam optimizer~\cite{kingma2014adam}. Training parameters include a batch size of 100, a learning rate of \(1 \times 10^{-3}\), and a weight decay of \(5 \times 10^{-6}\) over 10,000 epochs. During training, to alleviate the discontinuity of canonicalization, input augmentation through random rotations is applied, and FA-GNN is retrained under identical settings for fair comparison. Training and evaluation are conducted on a single NVIDIA GeForce RTX 2080 Ti GPU.

\subsection{Open Catalyst Dataset}
\label{sec:training_ocp}

\paragraph{Model.} We perform a direct comparison to \textsc{FAENet} by training it using the MFA paradigm with generalized QR decomposition in place of stochastic frame averaging as used by~\citet{duval2023faenet}. Network configurations include a radius cutoff of 6.0, 5 interaction layers, and a hidden dimension of 384. 
Frames are calculated based on an identical point projection to 2D as employed by~\citet{duval2023faenet}. 
 
\paragraph{Training.} The Adam optimizer is employed with a batch size of 256 and an initial learning rate of \(2 \times 10^{-3}\), adjusted by a cosine annealing scheduler over 12 epochs. Random reflection augmentation is used during training. The model is trained and evaluated on an NVIDIA A100 GPU. Comprehensive results for the OC20 dataset are presented in~\cref{tab:OC}.

\begin{table*}[t]
    \caption{Results on the IS2RE (initial structure to relaxed energy) task from OC20~\cite{ocp_dataset}. Models are trained and evaluated on the IS2RE direct task, where the relaxed energy is predicted directly from the initial structure, as opposed to via relaxation or Noisy Nodes data augmentation~\cite{godwin2022simple,liao2022equiformer}. Baselines are \textsc{SchNet}~\cite{schutt2018schnet}, \textsc{DimeNet++}~\cite{gasteiger2020fast}, \textsc{GemNet-T}~\cite{gasteiger2019directional}, \textsc{SphereNet}~\cite{liu2022spherical}, \textsc{ComENet}~\cite{wang2022comenet}, SEGNN~\cite{brandstetter2021geometric}, \textsc{Equiformer}~\cite{liao2022equiformer} and \textsc{FAENet}~\cite{duval2023faenet}.}
    \label{tab:OC}
    \begin{center}
    \begin{sc}
    \resizebox{\textwidth}{!}{
\begin{tabular}{lccccc|ccccc}
\toprule
{} & \multicolumn{5}{c}{Energy MAE [eV]} & \multicolumn{5}{c}{EwT}\\
Model & ID &   OOD-Cat & OOD-Ads &  OOD-Both & Average & ID &   OOD-Cat & OOD-Ads &  OOD-Both & Average\\
\midrule
SchNet & 0.6372 & 0.6611 & 0.7342 & 0.7035 & 0.6840 & 2.96\% & 3.03\% & 2.22\% & 2.38\% & 2.65\% \\
DimeNet++ & 0.5716 & 0.5612 & 0.7224 & 0.6615 & 0.6283 & 4.26\% & 4.10\% & 2.06\% & \textbf{3.21}\% & 3.40\%\\
GemNet-T & 0.5561 & 0.5659 & 0.7342 & 0.6964 & 0.6382 & 4.51\% & 4.37\% & 2.24\% & 2.38\% & 3.38\%\\
SphereNet & 0.5632 & 0.5590 & 0.6682 & 0.6190 & 0.6024 & 4.56 \% & 4.59 \% & 2.70 \% & 2.70 \% & 3.64\%\\
ComENet & 0.5558 & 0.5491 & 0.6602 & 0.5901 & 0.5888 & 4.17\% & 4.53\% & 2.71\%  & 2.83\% & 3.56\% \\
SEGNN & 0.5310 & 0.5341 & 0.6432 & 0.5777 & 0.5715 & \textbf{5.32}\% & 4.89 \%  & 2.80 \% & 3.09 \% & \textbf{4.03} \%\\
Equiformer & \textbf{0.5088} & \textbf{0.5051} & 0.6271 & 0.5545 & \textbf{0.5489} & 4.88 \% & \textbf{4.92} \% & 2.93 \% & 2.98 \% & 3.93 \%\\
\midrule
FAENet & 0.5446 & 0.5707 & \textbf{0.6115} & \textbf{0.5449} & \textbf{0.5679} & \textbf{4.46}\% & \textbf{4.67}\% & 2.95 \% & \textbf{3.01} \% & \textbf{3.78}\% \\
MFAENet & \textbf{0.5437} & \textbf{0.5415} & 0.6203 & 0.5708 & 0.5691 & 4.33\% & 4.54\% & \textbf{2.96} \% & 2.97 \% & 3.70\%\\

\bottomrule
\end{tabular}
    }
    \end{sc}
    \end{center}
    \vspace{-0.05in}
\end{table*}

\subsection{Top Tagging Dataset}
\label{sec:training_top_tagging}

\paragraph{Model.} 
\textsc{LorentzNet}~\citep{gong2022efficient} is an very strong $\mathrm O(1,3)$-invariant baseline developed specifically for the top tagging task. Each layer of \textsc{LorentzNet} is designed to be $\mathrm O(1,3)$-equivariant and the final output is $\mathrm O(1,3)$-invariant. We replace the invariant layer with the outer product between the invariant features and equivariant vectors, following batch normalization and nonlinearities, thereby breaking the invariance of the final output. The model, termed \textsc{MinkGNN}, incorporates 6 message passing layers with a hidden dimension of 72. Further refinement is achieved through our generalized QR decomposition to establish an $\mathrm{O}(1,3)$-invariant frame, resulting in an $\mathrm{O}(1,3)$-invariant model referred to as \textsc{MFA-MinkGNN}.

\paragraph{Training.} Both \textsc{MinkGNN} and \textsc{MFA-MinkGNN} are optimized using the Adam optimizer, with a batch size of 64, a learning rate of \(5 \times 10^{-4}\), a weight decay of $1\times 10^{-2}$, and a cosine annealing scheduler over 100 epochs. We then evaluate them with the Exponential Moving Average (EMA) with a decay of 0.995. Training and evaluation are conducted on 1 NVIDIA A100 GPU. Comprehensive results for the top tagging dataset are presented in~\cref{tb:top-tagging}.

\begin{table}[t]
\caption{Accuracy and AUC for the top tagging experiment. Baselines are \textsc{ResNeXt}~\cite{xie2017aggregated}, \textsc{ParticleNet}~\cite{qu2020jet}, EGNN~\cite{satorras2021n}, LGN~\cite{bogatskiy2020lorentz}, \textsc{LorentzNet}~\cite{gong2022efficient}, and CGENN~\cite{ruhe2023clifford}.}
\label{tb:top-tagging}
\begin{center}
\begin{small}
\begin{sc}
\begin{tabular}{lcc}
\toprule
Method & Accuracy & AUC \\
\midrule
ResNeXt & 0.936 & 0.9837 \\
ParticleNet & 0.940 & 0.9858\\
EGNN & 0.922 & 0.9760\\
LGN & 0.929 & 0.9640\\
LorentzNet & \textbf{0.942} & 0.9868 \\
CGENN & \textbf{0.942} & \textbf{0.9869} \\
MFA-MinkGNN & \textbf{0.942} & \textbf{0.9869} \\
\bottomrule
\end{tabular}
\end{sc}
\end{small}
\end{center}
\end{table}

\subsection{Graph Separation}
\label{sec:training_graph_separation}

\paragraph{Model.} For both the \textsc{Graph8c} and EXP datasets, we adopt the MLP and GIN models as used by~\citet{puny2021frame} for the graph separation task. Using our custom implementation of the nauty canonical labeling algorithm~\cite{mckay2007nauty}, we generate canonical graphs for each sample and concatenate these to their respective inputs. Given the directed graph input and the permutation invariance of the task, frame sampling is unnecessary, allowing direct computation as per~\cref{eq:simplified_frame_averaging}.

\paragraph{Training.} Optimization is performed using the Adam optimizer, with settings including a learning rate of 0.001 and a batch size of 100 across 200 epochs.

\subsection{Convex Hull}
\label{sec:training_convex_hull}

\paragraph{Model.} We employ a MLP model with a hidden dimension of 32 for projections of both the feature and node dimensions, rendering the model neither \(\mathrm S_n\)-invariant nor \(\mathrm O(5)\)-invariant. After being centered, vectors from the origin of points are used as input. The frame is made \(\mathrm{S}_n\)-equivariant and \(\mathrm{O}(5)\)-invariant using the inner product matrix of vectors and is subsequently applied to a \(\mathrm{O}(5)\)-invariant frame to achieve an \(\mathrm S_n \times \mathrm O(5)\)-invariant frame averaging. The canonicalization of the inner product matrix follows the methodology described in~\cref{sec:mckay}.

\paragraph{Training.} The convex hull dataset, comprising 16,384 samples for each of the training, validation, and test sets, is generated following~\citet{ruhe2023clifford}. Training utilizes the Adam optimizer with a learning rate of 0.001, a batch size of 128, and a cosine annealing scheduler.

\end{document}